\documentclass{article}

\usepackage[numbers]{natbib}

\usepackage[preprint]{neurips_2021}

\usepackage[utf8]{inputenc} 
\usepackage[T1]{fontenc}    
\usepackage{hyperref}       
\usepackage{url}            
\usepackage{amsfonts}       
\usepackage{nicefrac}       
\usepackage{microtype}      
\usepackage{xcolor}         
\usepackage{subcaption,booktabs}
\usepackage{physics}
\usepackage[inline]{enumitem}

\usepackage{bm}
\usepackage{algorithm, algorithmic}
\usepackage{mathtools}
\usepackage{amsthm,amssymb,amsmath}
\usepackage{braket}
\usepackage{multirow}
\usepackage{breqn}
\usepackage{mathrsfs}

\usepackage{subcaption}
\DeclareCaptionSubType * [alph]{table}
\newtheorem{theorem}{Theorem}
\newtheorem*{theorem*}{Theorem}

\newtheorem{lemma}{Lemma}
\newtheorem{remark}{Remark}

\newtheorem{corollary}{Corollary}
\newtheorem{assumption}{Assumption}

\DeclareMathOperator*{\argmin}{argmin}

\def\*#1{\mathbf{#1}}
\def\$#1{\mathcal{#1}}
\def\^#1{\mathbb{#1}}

\def\innerprod#1{\left\langle#1\right\rangle}
\def\qtext#1{\quad\text{#1}\quad}

\title{Optimization Induced  Deep Equilibrium Networks}

\author{%
   Xingyu Xie\thanks{Equal Contribution}~~\thanks{Key Laboratory of Machine Perception (MOE), School of EECS, Peking University}\\
   \And
   Qiuhao Wang\footnotemark[1]~~\footnotemark[2]\\
  \And
  Zenan Ling \footnotemark[2]~~\thanks{Pazhou Lab, Guangzhou, China}
  \And
  Xia Li\thanks{Swiss Federal Institute of Technology}\\
  \And
  Yisen Wang\footnotemark[2]\\
  \And
  Guangcan Liu\thanks{Nanjing University of Information Science \& Technology}\\
  \And
  Zhouchen Lin\footnotemark[2]\\
  \texttt{zlin@pku.edu.cn}\\
}

\begin{document}

\maketitle

\begin{abstract}
Implicit equilibrium models, i.e., deep neural networks (DNNs) defined by implicit equations, have been becoming more and more attractive recently. In this paper, we investigate an emerging question: can an implicit equilibrium model's equilibrium point be regarded as the solution of an optimization problem? To this end, we first decompose DNNs into a new class of unit layer that is the proximal operator of an implicit convex function while keeping its output unchanged. Then, the equilibrium model of the unit layer can be derived, named Optimization Induced Equilibrium Networks (OptEq), which can be easily extended to deep layers. The equilibrium point of OptEq can be theoretically connected to the solution of its corresponding convex optimization problem with explicit objectives. Based on this, we can flexibly introduce prior properties to the equilibrium points: 1) modifying the underlying convex problems explicitly so as to change the architectures of OptEq; and 2) merging the information into the fixed point iteration, which guarantees to choose the desired equilibrium point when the fixed point set is non-singleton. We show that deep OptEq outperforms previous implicit models even with fewer parameters.
This work establishes the first step towards optimization guided design of deep models. 

\end{abstract}

\section{Introduction}

Recently, the implicit models have gained significant attention, since they have been demonstrated to match or exceed the performance of traditional deep neural networks (DNNs) while consuming much less memory. Instances of such models include Neural ODE \cite{chen2018neural,massaroli2020dissecting}, whose implicit equation is the ODE of a continuous-time dynamical system, and Deep Equilibrium (DEQ) Model \cite{bai2019deep,bai2020multiscale}, which tries to solve a nonlinear fixed point equation. Instead of specifying the explicit computation procedure, the implicit equilibrium model specifies some conditions that should be held at its solution. Thus the implicit equilibrium model’s output is created by finding a solution to some implicitly defined equation. This is quite suitable to study from the optimization perspective, since most convex problems have their equivalent monotone inclusion form, i.e., finding the zero set of the sum of maximal monotone operators. 

Some previous work aims at explaining the forward propagation of DNNs as optimization iterations. \citet{li2020training} show that DNN’s activation function is indeed a proximal operator if it is non-decreasing. Other works study unrolled networks for optimization \cite{amos2017optnet,agrawal2019differentiable}, \textit{i.e.}, designing neural networks following the optimization algorithms. However, given a general DNN architecture, whether there exists an underlying optimization problem remains open. If we can find such an underlying optimization problem explicitly, it may help us further understand the mechanism of DNN and the reasons for its empirical success. 
We can even introduce the customized property into the DNN by adding some regularization terms to the underlying optimization objective, which may help to inspire more effective DNN architectures. 
Moreover, when finding the model's output is equivalent to minimizing a convex objective, we can utilize any optimization algorithm, especially the accelerated ones, to obtain the equilibrium point rather than limited root find methods.


In this paper, we investigate the emerging problem of discoverying the underlying optimization problem of DNN. We first reformulate the general multi-layer feedforward DNN and decompose it into the composition of a new class of unit layers while keeping the output unchanged. The unit layer is the proximal operator of an underlying implicit convex function under a mild condition. Thus the underlying convex function determines its behavior completely. We propose the corresponding equilibrium model of the new unit layer, named Optimization Induced Equilibrium Networks (OptEq). By replacing the convex objective with its Moreau envelope, OptEq naturally produces the commonly used skip connection architecture while keeping the equilibrium point unchanged. To strengthen the representation ability, we further propose a deep version of OptEq, which includes the general feedforward DNN as a special case. We also provide the underlying convex optimization problem of deep OptEq, whose objective is the sum of Moreau envelopes of several proper convex functions. Moreover, 
we can introduce any customized property to the equilibrium of OptEq, e.g., feature disentanglement. We propose two methods for such feature regularization. The first way is to modify the underlying optimization problem directly, which will lead to changes in the architecture of OptEq. Another method is using a modified SAM \cite{sabach2017first} iteration to select the fixed point with minimum regularization when the fixed point set is non-singleton. In summary, our main contributions include:
\begin{itemize}
    \item By decomposing the general DNN, we propose a new class of unit layer, which is the proximal operator of a convex function. Then we extract the unit layer to make it an equilibrium model called OptEq, and further propose its deep version. Without further reparameterization, the equilibrium point of OptEq is a solution to an underlying convex problem. 
    \item We propose two methods to introduce customized properties to the model's equilibrium points. One is inspired by the underlying optimization, and the other is induced by a modified SAM~\cite{sabach2017first} iteration.
	This is the first time that we can customize deep implicit models in a principled way.
	\item 
	We conduct experiments on CIFAR-10 for image classification and Cityscapes for semantic segmentation. Deep OptEq significantly outperforms implicit baseline models. Moreover, we also provide several feature regularizations that can significantly improve the generalization.
\end{itemize}

\section{The Proposed Optimization Induced Equilibrium Networks}
\subsection{Preliminaries and Notations}
We provide some definitions that are frequently used throughout the paper.
A function $f: \$H \rightarrow \^R \cup\{+\infty\}$ is proper if the set $\{x: f(x)<+\infty\}$ is non-empty, where $\${H}$ is the Euclidean space. 
We write l.s.c for short of lower semi-continuous.
The subdifferential, proximal operator and Moreau envelope of a proper convex function $f$ are defined as:
\[
\left\{
\begin{aligned}
&\partial f(\*x)\coloneqq\!\qty{\*g \!\in \!\$H : f(\*y) \geq f(\*x)+\innerprod{\*y \!-\!\*x, \*g}, \forall \*y \in \$H}, \\
&\operatorname{prox}_{\mu \cdot f}(\*x) \!\coloneqq\!\!\qty{\*z\!\in \! \$H : \*z\!=\!\argmin_{\*u} \frac{1}{2\mu}\norm{\*u\!-\!\*x}^{2} \!+\! f(\*u)},\\
&M_f^{\mu}(\*x)\coloneqq\mathop{\min}\limits_{\*u} \frac{1}{2\mu}\norm{\*u-\*x}^{2} \!+\! f(\*u),
\end{aligned}
\right.
\]
respectively, where $\innerprod{\cdot,\cdot}$ is the inner product and $\norm{\cdot}$ is the induced norm.
The conjugate $f^{*}$ of a proper convex function $f$ is defined as: $f^*(\*y)\coloneqq \sup\qty{\innerprod{\*y,\*x} - f(\*x): \*x\in\$H}$.
For the matrix $\*W$, $\norm{\*W}_2$ is the operator norm. 
While for the vector $\ell_2$-norm, we write $\norm{\cdot}$ for simplicity.
Given the map $\$T:\^R^m \to \^R^m$, the set of all fixed points of $\$T(\cdot)$ is denoted by $\operatorname{Fix}(\$T) = \qty{ \*z \in \^R^m : \$T(\*z) = \*z},$ whose cardinality is $|\operatorname{Fix}(\$T)|.$

DEQ is a recently proposed implicit model inspired by the observation on the feedforward DNN (with an input-skip connection): for $k=1,\cdots,L-1,$
\begin{equation}\label{eq:basicDNN}
\*z_{k+1}=\sigma (\*W_k \*z_k +\*U_k \*x+\*b_k),\quad \*y=\*W_{L+1}\*z_{L},
\end{equation}
where $\sigma(\cdot)$ is a non-linear activation function, $\*W_k \in \^R^{n_k\times n_{k-1}}$ and $\*U_k \in \^R^{n_k\times d}$ are learnable weights and $\*b_k \in \^R^{n_k}$ is a bias term.
A direct way to obtain the equilibrium point of this system is to consider the fixed point equation: 
$\*z^\star=\sigma (\*W \*z^\star +\*U \*x+\*b)$.
And we can utilize any root-finding algorithm to solve this equation.
Although DEQ may achieve good performance with a smaller number of parameters than DNNs, its superiority heavily relies on the careful initialization and regularization due to the instability issue of the fixed point problems. 
Some recent work \cite{winston2020monotone} is devoted to solving the instability issue by using a tricky re-parametrization of the weight matrix $\*W$. 
However, this may greatly weaken the expressive power of DEQ, see Prop. 8 in \cite{revay2020lipschitz}.
Moreover, in the present equation form, we have difficulty in getting further properties of the equilibrium point.

\subsection{One Layer OptEq}\label{sec:oneopteq}

Here, we consider an alternative form of the system in Eq.(\ref{eq:basicDNN}).
\begin{lemma}[Universal Hidden Unit]\label{lem:reformulate}
Given the parameters $\{\qty(\*W_{k} ,\*U_k, \*b_k)\}_{k=1}^L$ of a general DNN in Eq.(\ref{eq:basicDNN}), there exists a set of weights $\{\overline{\*W}_k \in \^R^{n_k\times m}\}_{k=0}^L$ with  $m \geq 2 \max\{n_k, k=0,\cdots,L\}$, such that the system in Eq.(\ref{eq:basicDNN}) can be re-written as the following network: for $k=1,\cdots,L-1,$
\begin{equation}\label{eq:reformulate}
    \overline{\*z}_{k+1}=\overline{\*W}_k^\top\sigma (\overline{\*W}_k \overline{\*z}_{k} +\*U_k \*x+\*b_k),\quad
    \*y=\overline{\*W}_{L+1}\overline{\*z}_{L}.    
\end{equation}
\end{lemma}
Notably, \emph{without changing the output} $\*y$, any feedforward DNN has the reformulation in Eq.(\ref{eq:reformulate}). The formal proof can be found in appendix, we present the main idea here:
\[
\begin{aligned}
\*y= & \*W_{L} \sigma \qty\Big( \*W_{L-1} \sigma \qty( \cdots\*W_2 \sigma(\*W_1\*z_0 + \*U_1\*x + \*b_1)\cdots )   )\\
= & \underbrace{\overline{\*W}_{L}~\overline{\*W}_{L-1}^\top}_{{\*W}_{L}} \sigma \qty( \underbrace{\overline{\*W}_{L-1}~\overline{\*W}_{L-2}^\top}_{{\*W}_{L-1}} \sigma \qty( \cdots\underbrace{\overline{\*W}_{2}~\overline{\*W}_{1}^\top}_{{\*W}_{2}} \sigma(\underbrace{\overline{\*W}_{1}~\overline{\*W}_{0}^\top}_{{\*W}_{0}}\*z_0 + \*U_1\*x + \*b_1)\cdots )   ).
\end{aligned}
\]
In the sense of weight-tying (i.e., all the layers share the same weights), the DNN's output $\*y$ is a linear transformation of $\overline{\*z}_{L}$, where $\overline{\*z}_{L}$ is a good approximation of the following fixed point equation under some mild assumptions:
\begin{equation}\label{eq:Opteq}
\*z^*=\*W^\top \sigma (\*W \*z^* +\*U \*x+\*b).
\end{equation}
Hence, the feedforward DNN also inspires an interesting and different equilibrium model Eq.(\ref{eq:Opteq}).
We call it Optimization Induced Equilibrium Networks (OptEq) since it is tightly associated with an underlying optimization problem.
As shown in Theorem \ref{thm:core} that follows, the equilibrium point $\*z^*$ is a solution of a convex problem that has an explicit formulation.
From the perspective of optimization, we can easily solve the  existence and the uniqueness problems of the fixed point equation, rather than resorting to the cumbersome reparameterization trick in~\cite{winston2020monotone}.
Most importantly, by studying the underlying optimization problem, we can investigate the properties of the equilibrium point of OptEq.
The following theorem formally shows the relation between OptEq and optimization.
\begin{assumption}\label{asm:sigma}
The activation function $\sigma:\^R \rightarrow \^R$ is monotone and $\Tilde{L}_\sigma$-Lipschitz, i.e.,
\[
0\leq \frac{\sigma(a)-\sigma(b)}{a-b} \leq \Tilde{L}_\sigma, \quad \forall a,b \in \^R, \quad a\neq b.
\]
\end{assumption}
\begin{theorem}\label{thm:core}
If Assumption \ref{asm:sigma} holds, for one NN layer $f:\^R^m \rightarrow \^R^m$ given by:
\[
f(\*z) \coloneqq \frac{1}{\mu} \*W^\top\sigma\qty(\*W\*z+\*U \*x +\*b),
\]
we have $f(\*z) = \operatorname{prox}_{\varphi}(\*z)$ when $\mu \geq \Tilde{L}_\sigma \norm{\*W}_2^2$, where
\[
\varphi(\*z) = \psi^{*}(\*z) - \frac{1}{2}\norm{\*z}^2,~\psi(\*z) = \frac{1}{\mu}  \mathbf{1}^\top  \tilde{\sigma}\qty(\*W\*z +\*U \*x + \*b),
\]
in which
$\forall a \in \^R,~ {\tilde{\sigma}}(a) = \int_0^a \sigma(t)\,dt$, applied element-wisely to vectors, and $\*1$ is the all one vector.
Furthermore, the solution to the fixed point equation $\*z = f(\*z)$ is the minimizer of the convex function $\varphi(\cdot)$.
\end{theorem} 
Theorem \ref{thm:core} shows that unit layer given in Eq.(\ref{eq:Opteq}) is a proximal operator of an underlying convex function given by a conjugate function in most cases, and the equilibrium point of OptEq happens to be the minimizer of this function\footnote{Besides the forms we show, one may expect a common rule to determine whether a general mapping is a proximal operator or not, which can help to create the implicit equilibrium model from the optimization perspective.
We provide the sufficient and necessary conditions in Lemma \ref{lem:proxandsubg} (see appendix).}.
In the rest part of this paper, for ease of discussion, we focus on the case that $\mu=1$ for OptEq, which may correspond to the assumption $\Tilde{L}_\sigma = 1$ and $\norm{\*W}_2 \leq 1$.
In some cases, we can write down the closed form of the optimization objection $\varphi(\cdot)$. For example, when the weight matrix $\*W$ is invertible, and the activation $\sigma(\cdot)$ is ReLU, i.e., $\sigma(x) = \max\qty{x, 0}, \forall x \in \^R$, we have:
\[\varphi(\*z) = \*1^\top \tilde{\sigma}^{*}\qty(\*W^{-\top} \*z) - \innerprod{\*U\*x+\*b,\*W^{-\top}\*z}-\frac{1}{2}\|\*z\|^2,
\]
where $\tilde{\sigma}^{*}(x) = 
\begin{cases}
\frac{1}{2}x^2, & x>0,\\
\infty, & x \le 0
\end{cases}$, applied element-wisely to vectors.

By Theorem \ref{thm:core}, we can quickly obtain the well-posedness of OptEq.
In general, any $\*W$ that makes the underlying objective $\varphi$ to be strictly convex will ensure the existence and uniqueness of OptEq's equilibrium. 
For example, when $\|\*W\|_2<1$, the operator: $\*z \mapsto \*W^\top\sigma\qty(\*W\*z +\*U \*x + \*b)$ is contractive, therefore the fixed point equation Eq.(\ref{eq:Opteq}) has a unique solution, i.e., it exists and is unique. 
What's more, we show a training strategy that can actually deal with a much more general case --- $|\operatorname{Fix}(\$T)|\geq 1$. 
Please see Sec. \ref{sec:SAM}  for more details.
\par
OptEq's most attractive aspect is that we can introduce customized properties or NN architecture to the model simply by modifying the underlying convex problem.
For example, we can naturally introduce the commonly used skip connection structure only by replacing $\varphi(\*z)$ with its Moreau envelope $\alpha M_{\varphi}^{1-\alpha}(\*z)$, then OptEq becomes:
\[
\*z = \alpha\*W^\top\sigma\qty(\*W\*z +\*U \*x + \*b)+(1-\alpha)\*z.
\]
Note that we do not change the equilibrium point of OptEq but make the operator on the right hand side strongly monotone and invertible, which can stablize the iteration. 

\subsection{Deep OptEq}
Some work claims that one layer implicit equilibrium is enough \cite{bai2019deep} and improves the model expressive ability by stacking small DEQs to obtain a wide one-layer DEQ, i.e.,  considering a fixed point problem in a higher dimension.
However, in practice, it is difficult to solve a large scale fixed point problem.
Hence, multi-layer DEQ, having a one-layer equivalent form in most cases, is a good substitute for the one-layer wide one since it improves the model ability without changing the problem scale.
Furthermore, multi-layer DEQ has a more powerful expressive ability when the output dimension is fixed.
Indeed, as we will show in the next section, the wide one-layer DEQ may be a special case of the multi-layer DEQ in the asymptotic sense after fixing the output dimension.
In this subsection, we propose a multi-layer version of OptEq, which is also associated to an underlying optimization problem.
\par
We consider a multi-layer OptEq, where $\*x_0 \in \^R^{d_x}$ denotes the input, $\*z \in \^R^{m}$ denotes the hidden unit, and $\*y\in \^R^{d_y}$ denotes the output. Namely, deep OptEq follows the implicit equation:
\begin{equation}\label{eq:ComDeq}
\left\{
\begin{aligned}
& \*x = g(\*x_0,\*W_0),\\
&\*z = \$T(\*z,\*x, \bm{\theta})\coloneqq f_{L}\circ f_{L-1}\cdots \circ f_1(\*z,\*x, \bm{\theta})\ ,\\
&\*y = \*W_{L+1} \*z, 
\end{aligned}
\right.
\end{equation}
where for all $l\in[1,L]$ and $\alpha \in (0,1]$,
\begin{equation}
f_l(\*z,\*x) = \alpha \*W_l^\top \sigma\qty(\*W_l\*z + \*U_l\*x+ \*b_l) + (1-\alpha)\*z.
\end{equation}
Here,
given the set of learnable parameters $\*W_0, g(\cdot, \*W_0): \^R^{d_x} \rightarrow \^R^{d}$ is any continuous function which we usually choose as a feature extractor, e.g., shallow NNs.
$\bm{\theta} = \{\qty(\*W_{l},\*U_l, \*b_l)\}_{l=1}^L$ is the set of all learnable parameters for our equilibrium network, where $\*W_{l}\in \^R^{n_l\times m}$, $\*U\in \^R^{n_l\times d}$, $\*b_l \in \^R^{n_l}$ are the learnable weight matrices and bias term, respectively. 
Note that $\*W_{L+1} \in \^R^{d_y\times m}$ is also learnable. 
$\sigma: \^R \rightarrow \^R $ is the activation function, when the input is multi-dimensional, we apply the function $\sigma(\cdot)$ element-wise.
The hidden unit $\*z$ is the equilibrium point of the fixed point equation $\*z = \$T(\*z,\*x, \bm{\theta})$ when $\*x$ and $\bm{\theta}$ are given.
Without loss of generality, we assume that the feature extractor satisfies a Lipschitz continuity assumption w.r.t. the learnable weight, i.e., $\norm{g(\*x_0,\*W_1)-g(\*x_0,\*W_2)} \leq L_g \norm{\*W_1-\*W_2}_2$.
\par
At the first glance, deep OptEq seems very different from the traditional DNNs.
Some negative results on DEQ are shown in previous work \cite{revay2020lipschitz}: with improper weight re-parameterization, DEQ does not contain any feedforward networks.
By contrast, without complicated weight re-parameterization, deep OptEq can include the general feedforward DNN as its special case. 
\begin{lemma}\label{lem:allDNNs}
The deep OptEqs contain all feedforward DNNs.
\end{lemma}

\section{Recover Optimization Problem from Implicit Equilibrium  Models}
This section provides our main results on the connection between convex optimization problem and our deep OptEq.
In the previous section, we have shown that one layer of DNN is a proximal operator under a mild assumption. 
However, the composition of multiple proximal operators is not a proximal operator in most cases. 
Fortunately, we can still recover the underlying optimization problem of deep OptEq, and find that its equilibrium point is a zero point of a convex function's subdifferential with a weak additional permutation constraint. 
In addition, we can explicitly provide the optimization objectives corresponding to deep OptEq in some cases. 
Before providing the main results, we first show the connection between our deep OptEq given in Eq.(\ref{eq:ComDeq}) and a multi-block one-layer OptEq.
\begin{lemma}[Deep OptEq and Multi-block OptEq]\label{lem:cyclFixedPoint}
If $\*z^*\coloneqq\*z^*_0$ is an equilibrium point of the equation $\*z = f_{L}\circ f_{L-1}\cdots \circ f_1(\*z,\*x, \bm{\theta})$, set $\*z^*_1\coloneqq f_1(\*z^*_0,\*x,\bm{\theta}),~\*z^*_2\coloneqq f_2 \circ f_1(\*z^*_0,\*x,\bm{\theta}), \cdots,~\*z^*_{L-1}\coloneqq f_{L-1} \cdots \circ f_2 \circ f_1(\*z^*_0,\*x, \bm{\theta})$, then $\widetilde{\*z}^*\coloneqq [\*z^*_1,\cdots,\*z^*_{L-1},\*z^*_0]^\top \in \^R^{mL}$ is an equilibrium point of the equation:
\begin{equation}\label{eq:MatDeqMain}
   \widetilde{\*z} = \alpha \widetilde{\*W}^\top\sigma\qty(\widetilde{\*W}\*P \widetilde{\*z} +\widetilde{\*U}\*x + \widetilde{\*b}) + (1-\alpha)\*P \widetilde{\*z}, 
\end{equation}
where $\widetilde{\*W}$ is block diagonal and $\*P$ is a permutation matrix,
\[
\widetilde{\*W}\coloneqq \mqty[\dmat{\*W_1,\ddots,\*W_L}], \quad\*P\coloneqq \mqty[\mqty{0&&&&& \*I} \\ \mqty{\*I & 0 &&&&} \\ \mqty{ & & \ddots& \ddots&}\\ \mqty{ & & & & \*I &0 }],
\]
$\widetilde{\*U} \coloneqq [\*U_1,\cdots,\*U_L]^\top$ and $\widetilde{\*b} \coloneqq [\*b_1,\cdots,\*b_L]^\top$ are the concatenated matrix and vector, respectively (see Eq.(\ref{eq:MatDeq}) in appendix for more details).
\end{lemma} 
By Eq.(\ref{eq:MatDeqMain}), and the necessary and sufficient condition for one DNN layer to be a proximal-like operator (Lemma \ref{lem:D_x_yandsubg} in appendix), we can further reveal the connection between deep OptEq and optimization under a mild assumption.
\begin{assumption}\label{asm:weightBound}
Assumption \ref{asm:sigma} with $\Tilde{L}_\sigma = 1$ and $\norm{\*W_i}_2 \leq 1, \forall i \in [1,L]$ hold.
\end{assumption}
Note that we make this assumption just for the ease of discussion. The assumption is actually unnecessary since we can introduce an additional constant to re-scale the whole operator as did in Theorem \ref{thm:core}. 

\begin{theorem}[Recovering Optimization Problem from Deep OptEq]\label{thm:zeroset}
If Assumption \ref{asm:weightBound} holds, then any equilibrium point $\widetilde{\*z}^*$ of Eq.(\ref{eq:MatDeqMain}) satisfies:
\begin{equation}\label{eq:optCondition}
  0\in\partial \Phi \qty(\widetilde{\*z}^*)+\qty(\*{I}-\*P)\widetilde{\*z}^*,  
\end{equation}
where $\*I$ is the identity matrix
and $\Phi(\widetilde{\*z}^*)$ is given by a sequence Moreau envelopes of convex functions $\qty{\varphi_i}_{i=1}^L$ such that $\operatorname{prox}_{\varphi_i}(\*z) = \*W_i^\top \sigma\qty(\*W_i\*z + \*U_i\*x+ \*b_i)$, namely:
\[
\Phi(\widetilde{\*z}) = \sum_{i=1}^{L}\alpha M_{\varphi_i}^{1-\alpha}(\*z_i), 
\]
where $\*z_i$ is the $i$-th block of $\widetilde{\*z}^*$ and $M_{\varphi_i}^{1-\alpha}(\*z)$ is the $\varphi_i$'s Moreau envelope.
\end{theorem}
When the block size is $1$, i.e., $L=1$,  we can immediately obtain that $ 0\in\partial \Phi (\widetilde{\*z}^*)$, namely, the equilibrium point is a solution of a convex optimization problem. 
So the results provided in Theorem \ref{thm:core} is a special case here.
Note that one block does not mean that $\*z$ is one-dimensional.
Moreover, for two blocks, the deep OptEq is also an optimization solver.
\begin{corollary}
\label{coro:twoblock}
If the block size $L=2$ and Assumption \ref{asm:weightBound} holds, then the equilibrium point $\widetilde{\*z}^*=[\*z_1^*,\*z_0^*]^\top$ of Eq.(\ref{eq:MatDeqMain}) is also a solution to a convex problem:
\[
  \min_{\*z_1,\*z_0} \qty{\alpha M_{\varphi_1}^{1-\alpha}(\*z_1)+ \alpha M_{\varphi_2}^{1-\alpha}(\*z_0) + \frac{1}{2} \norm{\*z_1 - \*z_0}^2}.
\]
\end{corollary}
For general $L>2$, we can also write down the monotone inclusion equation  Eq.(\ref{eq:optCondition})'s underlying optimization problem when $\alpha \to 0$.
Interestingly, this result implies the equivalence between the composited deep models and the wide shallow ones in the asymptotic sense.
\begin{theorem}[Connection between Wide and Deep OptEq]\label{thm:asyalpha}
If Assumption \ref{asm:weightBound} holds, and there is at least one $\|\*W_i\|_2<1$. Assume $\widetilde{\*z}^*(\alpha)\coloneqq[\*z^*_1(\alpha),\cdots,\*z^*_{L-1}(\alpha),\*z^*_0(\alpha)]^\top \in \^R^{mL}$ is the equilibrium point of Eq.(\ref{eq:MatDeqMain}). When $\alpha \rightarrow 0$, all $\*z^*_l(\alpha)$s tend to be equal, and the limiting point is the last entry $\*y$ of the minimizer $(\*x_1,\cdots,\*x_L, \*y )$ of the following optimization problem:
\[
  \min_{\*x_1,\cdots,\*x_L, \*y} \qty{\sum_{l=1}^{L}\left( \varphi_l(\*x_l)+\frac{1}{2}\|\*x_l-\*y\|^2 \right)},
\]
where $\varphi_l(\cdot)$ is the same as that in Theorem \ref{thm:zeroset}.
\end{theorem}
Theorem \ref{thm:asyalpha} implies that, when $\alpha \to 0$, the equilibrium point of the deep OptEq is the same as the solution of $L\*z = \sum_{i=1}^L \*W_i^\top \sigma\qty(\*W_i\*z + \*U_i\*x+ \*b_i)$, which is a wide one-layer OptEq with multiple blocks. 
Given the output dimension and the same amount of learnable parameters, the wide multi-block OptEq is actually a special case of deep OptEq in the asymptotic sense.
Hence, deep OptEq is more expressive than wide one-layer OptEq.
\par
In general, we can still loosely treat the equilibrium point as a minimizer of an implicit optimization problem, since the only difference between the monotone inclusion equation $0\in\partial \Phi (\widetilde{\*z}^*)$ and Eq.(\ref{eq:optCondition}) is a weak constraint dominated by the operator $(\*I-\*P)(\cdot)$,  which  aims to reduce the divergence between the multi-blocks.

\section{Introducing Customized Properties to Equilibrium Points}\label{sec:introProp}
By employing the underlying optimization problem, we can investigate the potential property of the equilibrium points. A more advanced way to use the connection between (deep) OptEq and optimization is to introduce some customized properties to equilibrium points, i.e., the feature learned by OptEq. Note that none of previous
DEQs take into account the regularization of features, which has been proved to be effective both theoretically \cite{amjad2019learning,alemi2016deep} and empirically \cite{zhang2021deep}.
\subsection{Underlying Optimization Inspired Feature Regularization}\label{sec:Reg1}
As we mentioned in Sec.\ref{sec:oneopteq}, if we replace the underlying optimization objective with its Moreau envelope, OptEq will naturally have a skip-connection structure, which has been adopted in the construction of deep OptEq. 
Following this idea, when we modify the underlying optimization problem of deep OptEq, it should inspire more network architectures.
\par
An exciting application of Theorem \ref{thm:zeroset} is introducing  customized properties to the equilibrium points of deep OptEq, 
because appending one layer after deep OptEq is equivalent to adding one term to the objective $\Phi(\cdot)$.
Specifically, if we modify $\Phi(\widetilde{\*z})$ to $\Phi(\widetilde{\*z})  + \$R_z(\*z_L)$, then deep OptEq becomes:
\[
\*z = \$T_{\$R_z}\circ\$T(\*z,\*x, \bm{\theta}),
\]
where $\$T_{\$R_z} = \operatorname{prox}_{R_z}$ or $\$T_{\$R_z} = \$I -\gamma \pdv{\$R_z}{\*z}$ when the proximal is hard to calculate. 
For example, ${\$R_z}$ will re-scale $\$T(\*z,\*x, \bm{\theta})$'s output when ${\$R_z}(\cdot) = \norm{\cdot}^2$, and becomes a shrinkage operator after changing ${\$R_z}$ to $\norm{\cdot}_1$.
In general, ${\$R_z}(\cdot)$ can be any convex function that contains the prior information of the feature.
In summary, we introduces feature regularization by modifying the underlying optimization problem, which leads to a change of network structure.
Once again, studying the implicit equilibrium models from the perspective of optimization shows great superiority.

\subsection{SAM Iteration Induced Feature Regularization}\label{sec:SAM}
In this subsection, we provide another strategy for feature regularization. Note that most previous DEQs are devoted to ensuring a singleton fixed point set, relying on the tricky weight matrix re-parameterization. Considering the general case --- $|\operatorname{Fix}(\$T)|\geq 1$, we can choose the equilibrium with desired property by solving the following constrained optimization problem: 
\begin{equation}\label{eq:inclusion}
\*z^{*}(\*x, \bm{\theta}) \coloneqq \argmin_{\*z \in \operatorname{Fix}\qty(\$T(\cdot,\*x, \bm{\theta}))} \$R_z(\*z),
\end{equation}
where $\$R_z(\cdot)$ is the feature regularization that contains the prior information of the feature.
Given the training data $ (\*x_0,\*y_0) \in \^R^{d_x}\times \^R^{d_y}$, the whole training procedure becomes\footnote{For the sake of clarity, we utilize one training pair for discussion. 
In general, all the discussed results hold when we replace the single data point with the whole data set.}:
\begin{equation}\label{eq:loss}
    \min_{\widetilde{\bm{\theta}} } \ell\qty(\*W_{L+1}\cdot\*z^*(\*x,~\bm{\theta}), \*y_0) + \$R_w(\widetilde{\bm{\theta}}),
\end{equation}
where $\widetilde{\bm{\theta}} \coloneqq \qty{\*W_0,\bm{\theta},\*W_{L+1}}$, $\*x$   is given by Eq.(\ref{eq:ComDeq}), 
$\$R_w(\cdot)$ is the regularizer on the parameters , e.g., weight decay, and $\ell: \^R^{d_y} \times \^R^{d_y} \to \^R^{+}$ is the loss function.
\par
We adopt the SAM fixed point iteration \cite{sabach2017first} to solve the problem Eq.(\ref{eq:inclusion}). 
Starting from any $\*z_0 \in \^R^m$,  we consider the following sequence $\qty{\*z^k}_{k\in \^N}$: 
\begin{equation}\label{eq:InclusinApprox}
    \*z^{k} = \beta_k \$S_{\lambda_k}(\*z^{k-1}) + (1-\beta_k)\$T(\*z^{k-1},\*x, \bm{\theta}),
\end{equation}
where $\{\beta_k\}_{k\in\^N}$ and $ \{\lambda_k\}_{k\in\^N} $ are sequences of real numbers in $(0,1]$, $\$S_{\lambda}(\*z) = (1-\gamma \lambda)\*z - \gamma \pdv{\$R_z(\*z)}{\*z}$. 
Then we choose the $K$-th iteration $\*z^K(\*x, \bm{\theta})$ as an approximate of $\*z^{*}(\*x, \bm{\theta})$ and put it in the final loss term:  
\begin{equation}\label{eq:lossapprox}
    \min_{\widetilde{\bm{\theta}}} \ell\qty(\*W_{L+1}\cdot\*z^K(\*x,~\bm{\theta}), \*y_0)+ \$R_w(\widetilde{\bm{\theta}}).
\end{equation}
The unrolling term $\*z^K(\*x,~\bm{\theta})$ aggregate information from both $\$R_z(\*z) $ and $\$T$, making the prior information of feature being an inductive bias during training. And our model can be easily trained by any first-order optimization algorithms, e.g., GD, SGD, Adam, etc. 
\begin{remark}
SAM iteration \cite{sabach2017first} needs $\$R_z(\cdot)$ to be strongly convex, here we use a modified version of SAM which only assumes convexity. 
Given well-chosen $\{\beta_k\}_{k\in\^N}$ and $ \{\lambda_k\}_{k\in\^N} $, we can prove that the sequence generated by Eq.(\ref{eq:InclusinApprox}) converges to the point $\*z^{*}(\*x, \bm{\theta})$. Furthermore, we prove that the whole training dynamic, using the unrolling SAM strategy with backpropagation (BP), converges with a linear convergence rate.
We leave all the convergence results in appendix. 
\end{remark}
\begin{remark}
If we use the method in Sec.\ref{sec:Reg1} to introduce the prior information, there is no need to use SAM iteration again. Therefore, we can let beta $\beta_k = 0,~\$S_{\lambda_k}(\cdot) = 0$, and use the unrolling fixed point iteration strategy during training. Note that we can also use the implicit function theorem (IFT) based training way.
\end{remark}
Previous implicit models utilize IFT in training to avoid the storage consumption of forward-propagation.  The cost for that is it needs to solve two large-scale linear equations (or perform the matrix inversion directly) during training. Deep OptEqs can be trained both in the unrolling based and IFT based ways. In the sense of BP, the two training ways have different merits and limitations. We provide comparative experiments in Sec.\ref{sec:exp} and detailed discussion in appendix.
\vspace{-3mm}

\section{Experiments}\label{sec:exp}
\vspace{-3mm}
In this section, we investigate the empirical performance of deep OptEq from three aspects. 
First, on the image classification problem, we evaluate the performance of deep OptEqs along with our feature regularization strategies. 
The results trained with different $\alpha$s are also reported. Second, we compare deep OptEqs with previous implicit models and traditional DNNs. Finally, we compare our unrolling-based method with the IFT-based method and investigate the influence of unrolling iteration number $K$. 
Furthermore, we present the results on Cityscapes for semantic segmentation. 
\paragraph{Training Strategy of Deep OptEqs} 
In order to compare the effect of feature regularization on performance in detail, we compare three unrolling training ways based on Eq.(\ref{eq:lossapprox}): 1) $\$R^\dagger_z$: strategy in Sec.\ref{sec:SAM}, using the proposed SAM given in Eq.(\ref{eq:InclusinApprox}); 
2) $\$R^*_z$: strategy in Sec.\ref{sec:Reg1}, and set $\beta_k = 0$ in Eq.(\ref{eq:InclusinApprox}); and 3) no Reg: without feature regularization. 
During backward propagation, we utilize the commonly used first order optimization algorithm --- SGD. 
We set the learning rate as $0.1$ at the beginning and halve it after every $30$ epochs. And the total training epoch is $200$.
Here we set the iterative number $K=20$ and 
discuss the influence of different $K$ in Sec \ref{sec:diffK}.
\begin{table}
\centering
\caption{(a) The testing accuracy (Acc.) of deep OptEq with different settings. 
We set different $\lambda$ for $\$R^*_z$ and the mean values are taken on the whole feature tensor. 
(b) Comparisons with previous implicit models.
}
\begin{subtable}[h]{.58\linewidth}
\setlength{\tabcolsep}{3pt}
\raggedright
\caption{Feature Regularization (\# params 199k)}
	\begin{tabular}{l|ccccc}
	
		\toprule
		$\alpha$	& 0.01 & 0.1 & 0.4 & 0.8 & 1.0 \\
		\midrule
		Acc-(no Reg) & 58.4\% & 56.8\% & 86.9\%  & \textbf{87.4}\% & {87.2\%}  \\
		Acc-($\$R^\dagger_z$)  & 72.7\% & 61.5\% & 86.5\%  & 87.0\% & \textbf{87.7\%}\\
		Acc-($\$R^*_z$)  & 72.6\% & 60.0\% & \textbf{88.0}\%  & 87.6\% & {87.5\%}\\
		\midrule
		Acc-DNN. && & 82.7\% & &  \\
		\bottomrule
	$\$R^*_z = \lambda\norm{\cdot}_1$	& 0.01 & & 0.15 & & 0.5 \\
	\midrule
	mean $\norm{\cdot}_1$ &  $>5$ & & 0.81 & &\textbf{0.34}   \\
	Acc. & 86.4\%  & & \textbf{87.7}\% & &  86.9\% \\
	\bottomrule
	$\$R^*_z = \lambda \norm{\cdot}^2$	& 0.01 & & 1.0 & & 10.0 \\
	\midrule
	mean $\norm{\cdot}^2$ &  $>10$  & & 1.56 & &\textbf{0.82}   \\
	Acc. & 87.3\%  & & \textbf{87.6}\% & & 86.7\%  \\
	\midrule
\end{tabular} 
\label{tab:appending}
\end{subtable}
\begin{subtable}[h]{.38\linewidth}
\setlength{\tabcolsep}{3pt}
	\raggedright
	\caption{Performance Comparisons}
\label{tab:single}
\begin{tabular}{ccc}
	\toprule
Method	&  \# params &Acc. \\ 
\midrule
	ResNet-18 &  10M & \textbf{92.9}\%\\
\midrule
	Neural ODE&  172K&53.7\%\\
	Aug. Neural ODE& 172k &60.6\%\\
	\midrule 
	MON	&  &\\
	Single conv& 172K &74.1\%\\
	Single conv lg & 854K &82.5\%\\
	\midrule 
	deep OptEq	 & &\\
    decorrelation ($\$R^*_z$) & 1.4M & {91.0\%}\\
	 decorrelation ($\$R^\dagger_z$) & 162k & {86.0\%}\\
	 HSIC ($\$R^*_z$) & 162k & {87.4\%}\\
	no Reg  & 162k &85.7\%\\
	\bottomrule
\end{tabular}
\end{subtable}
\vspace{-6mm}
\end{table}
\vspace{-3mm}
\subsection{Performance of Different Feature Regularizations}
\vspace{-2mm}
We construct the deep OptEqs with $5$ convolutional layers, using five $3 \times 3$ convolution kernels with channels of $16, 32, 64, 128, 128$. In this experiment, we compare the performance of two ways to introduce the feature regularization on CIFAR-10 dataset.
We adopt the feature decorrelation as the $\$R_z$ here.
The specific formulation can be found in appendix.
We also use two norm regularizations to show whether there is a corresponding effect on the learned feature of deep OptEq.
Moreover, we show how the hyperparameter $\alpha$ affects the model performance.
\par
The results are shown in Table~\ref{tab:appending}. 
With the same size of parameters, deep OptEqs beats the general DNN (given in Eq.(\ref{eq:basicDNN})) easily.
It turns out that there is no linear relationship between the performance and the hyperparameter $\alpha$.
The hyperparameter $\alpha$ serves as a trade-off between the effect of fixed point equation and  the regulation induced operator $\$S$. In our setting, with a small initialization for $\{\*W_l\}$s, all weights will stay in a small compact set during training (see proof of Theorem \ref{thm:global}).
Therefore when $\alpha$ approaches $1$, deep OptEq is an intense contraction (i.e., with a small contractive coefficient), and the SAM iterations will quickly converge to the fixed point, in which case regularization induced operator $\$S$ has a limited impact. When $\alpha$ approaches $0$, deep OptEq is to be more like the identity operator, so $\$S$ dominates the whole iterations. 
\par
The optimization inspired implicit regularization ($\$R^*_z$) is also an efficient feature regularization way since it modifies deep OptEq structure directly.
Here we present two $\$R^*_z(\cdot)$ candidates: $\lambda\norm{\cdot}_1$ and $\lambda\norm{\cdot}^2$. 
The outputs of the feature show decreases in the corresponding norm, and a suitable regularization coefficient can lead to  better generalization performances.
\vspace{-2mm}
\subsection{Comparison with Previous Implicit Models}
\vspace{-2mm}
In this experiment, we compare deep OptEq with other implicit models, NODEs~\cite{chen2018neural}, Augmented NODEs~\cite{dupont2019augmented}, single convolutional Monotone DEQs~\cite{winston2020monotone} (short as MON), and classical ResNet-18~\cite{he2016deep}. 
We leave the definition of the HSIC regularization in appendix.
Note that deep OptEqs do not require the additional re-parameterization like MONs~\cite{winston2020monotone}. 
For fair comparisons, we construct the deep OptEqs with $5$ convolutional layers. 
In order to construct deep OptEqs with a similar number of parameters as baseline methods, we use five $3 \times 3$ convolution kernels with channels of $16, 32, 64, 64, 128$.   Moreover, we only use a single convolutional layer as the feature extractor $g(\cdot)$ for the model with 162k parameters, which is the same as single convolution MONs~\cite{winston2020monotone}.
\par
The results are shown in Table~\ref{tab:single}. 
Notably, even without feature regularization trick, our deep OptEqs significantly outperform baseline methods. 
We highlight the performance of deep OptEqs on CIFAR-10 which outperforms Augmented Neural ODE by $25.1\%$ and MON by $11.6\%$ with \emph{fewer parameters}.
Without adding the number of parameters, feature regularization helps deep OptEq to achieve better performance easily. 
Notably, HSIC, a feature disentanglement regularization, provides a significant gain for the generalization.
\vspace{-2mm}
\subsection{Efficiency and Approximation Error}\label{sec:diffK}
We train deep OptEqs by the IFT based way given in \citet{bai2019deep} and compare the results with the unrolling way.
The time for inference and BP is provided, and it is the total time for $80$ iteration steps with the batch size being $125$ on GPU NVIDIA GTX 1070.
The relative residual is averaged over all \emph{test} batches: $\norm{\*z^K - \$T(\*z^K,\*x, \bm{\theta})}_2/\norm{\*z^K}_2$.
For fair comparison, we do not utilize any feature regularization in this experiment. We set $\alpha = 0.8$ and let ``thd'' represent the residual threshold.
\begin{table}[ht]
\vspace{-4mm}
	\centering
\caption{Comparison between SAM and IFT based Training (\# params 199k)}
\begin{tabular}{c|cccc}
	\hline 
Method & Acc. & Inference Time &  Back-Prop Time & Relative Residual \\ 
	\hline  
Unrolling(K=5)	& 83.52\%  &\textbf{1.3}s   &\textbf{1.8}s & 1.22e-02  \\ 
	\hline 
Unrolling(K=10)	& 87.28\%  &2.2s   &3.3s & 4.88e-03  \\ 
	\hline 
Unrolling(K=20)	& 87.71\%  &3.9s   &6.4s & 4.81e-04  \\ 
	\hline 
Unrolling(K=40)	& \textbf{87.83}\%  &7.5s   &12.7s & \textbf{1.20e-05}  \\ 
	\hline 
IFT (thd = 1e-03)     	& 87.63\%  &16.3s   &6.7s & 7.33e-04  \\ 
	\hline 
IFT  (thd = 1e-02)   	& 85.95\%  &15.6s   &1.3s & 9.52e-03  \\ 
	\hline 
\end{tabular} \label{tab:time}
\vspace{-4mm}
\end{table}
\par
Although IFT based methods consume much less memory, given the comparable relative residual, the unrolling methods achieve better performance with much less inference and BP time. 
Note that a loose residual threshold may destroy the IFT based method significantly. We should choose the appropriate training method according to practice.
For IFT, the inference time is longer than the back-prop one since the fixed point equation needed to solve during inference is non-linear, which is more challenging than the linear one during BP.
\vspace{-2mm}
\subsection{Cityscapes Semantic Segmentation}
\vspace{-2mm}
In this experiment, we evaluate the empirical performance of our deep OptEq on a large-scale computer vision task: semantic segmentation on the Cityscapes dataset. We construct a deep OptEq with only three weighted layers and channels of 256, 512 and 512. The deep OptEq is used as the ``backbone'' of the segmentation network.  
We compare our method with FCN~\cite{shelhamer2017fully} on
the Cityscapes test set. We employ the poly learning rate policy to adjust the learning rate, where the initial learning rate is multiplied by $(1 - iter / total\_iter)^{0.9}$ after each iteration. The initial learning rate is set to be 0.01 for both networks. Moreover, momentum and weight decay are set to 0.9 and 0.001, respectively. Note that we only train on finely annotated data.
We train the model for 40K iterations, with mini-batch size set as 8. The results on the validation set are shown in Table~\ref{Table-Cityscapes}. Notably, our deep OptEq significantly outperforms FCN with a similar number of parameters. Note that in this experiment, we have not introduced any customized property of the feature, so the performance improvement is entirely due to the superiority of the implicit structure of deep OptEq. 
\begin{table}[h]
\vspace{-3mm}
	\centering
\caption{Evaluation on the validation set of Cityscapes semantic
	segmentation.}\label{Table-Cityscapes}
\begin{tabular}{c|ccc}
	\hline 
Method	& mIoU &  mAcc & aAcc \\ 
	\hline 
FCN	& 71.47  &79.23   &95.56  \\ 
	\hline 
deep OptEq   &74.47   &81.91   &95.93\\
	\hline 
\end{tabular} 
\vspace{-6mm}
\end{table}
\label{sec:5.1}

\section{Conclusions}
\vspace{-2mm}
In this paper, we decompose the feedforward DNN and find a more reasonable basic unit layer, which shows a close relationship with the proximal operator. Based on it, we propose new implicit models, OptEqs, and explore their underlying optimization problems thoroughly. 
We provide two strategies to introduce customized regularizations to the equilibrium points, and achieve significant performance improvement in experiments. We highlight that by modifying the underlying optimization problems, we can create more effective network architectures.
Our work may inspire more interpretable implicit models from the optimization perspective.

\bibliography{nips2021}

\begin{thebibliography}{35}
\providecommand{\natexlab}[1]{#1}
\providecommand{\url}[1]{\texttt{#1}}
\expandafter\ifx\csname urlstyle\endcsname\relax
  \providecommand{\doi}[1]{doi: #1}\else
  \providecommand{\doi}{doi: \begingroup \urlstyle{rm}\Url}\fi

\bibitem[Agrawal et~al.(2019)Agrawal, Amos, Barratt, Boyd, Diamond, and
  Kolter]{agrawal2019differentiable}
Akshay Agrawal, Brandon Amos, Shane Barratt, Stephen Boyd, Steven Diamond, and
  J~Zico Kolter.
\newblock Differentiable convex optimization layers.
\newblock In \emph{Advances in Neural Information Processing Systems}, pages
  9562--9574, 2019.

\bibitem[Alemi et~al.(2016)Alemi, Fischer, Dillon, and Murphy]{alemi2016deep}
Alexander~A Alemi, Ian Fischer, Joshua~V Dillon, and Kevin Murphy.
\newblock Deep variational information bottleneck.
\newblock \emph{arXiv preprint arXiv:1612.00410}, 2016.

\bibitem[Amjad and Geiger(2019)]{amjad2019learning}
Rana~Ali Amjad and Bernhard~C Geiger.
\newblock Learning representations for neural network-based classification
  using the information bottleneck principle.
\newblock \emph{IEEE transactions on pattern analysis and machine
  intelligence}, 42\penalty0 (9):\penalty0 2225--2239, 2019.

\bibitem[Amos and Kolter(2017)]{amos2017optnet}
Brandon Amos and J~Zico Kolter.
\newblock Optnet: Differentiable optimization as a layer in neural networks.
\newblock In \emph{International Conference on Machine Learning}, pages
  136--145, 2017.

\bibitem[Ayinde et~al.(2019)Ayinde, Inanc, and Zurada]{ayinde2019regularizing}
Babajide~O Ayinde, Tamer Inanc, and Jacek~M Zurada.
\newblock Regularizing deep neural networks by enhancing diversity in feature
  extraction.
\newblock \emph{IEEE Transactions on Neural Networks and Learning Systems},
  30\penalty0 (9):\penalty0 2650--2661, 2019.

\bibitem[Bai et~al.(2019)Bai, Kolter, and Koltun]{bai2019deep}
Shaojie Bai, J~Zico Kolter, and Vladlen Koltun.
\newblock Deep equilibrium models.
\newblock In \emph{Advances in Neural Information Processing Systems}, pages
  690--701, 2019.

\bibitem[Bai et~al.(2020)Bai, Koltun, and Kolter]{bai2020multiscale}
Shaojie Bai, Vladlen Koltun, and J~Zico Kolter.
\newblock Multiscale deep equilibrium models.
\newblock \emph{Advances in Neural Information Processing Systems}, 33, 2020.

\bibitem[Bauschke et~al.(2011)Bauschke, Combettes, et~al.]{bauschke2011convex}
Heinz~H Bauschke, Patrick~L Combettes, et~al.
\newblock \emph{Convex analysis and monotone operator theory in Hilbert
  spaces}, volume 408.
\newblock Springer, 2011.

\bibitem[Beck(2017)]{beck2017first}
Amir Beck.
\newblock \emph{First-order methods in optimization}.
\newblock SIAM, 2017.

\bibitem[Bengio et~al.(2013)Bengio, Courville, and
  Vincent]{bengio2013representation}
Yoshua Bengio, Aaron Courville, and Pascal Vincent.
\newblock Representation learning: A review and new perspectives.
\newblock \emph{IEEE transactions on pattern analysis and machine
  intelligence}, 35\penalty0 (8):\penalty0 1798--1828, 2013.

\bibitem[Chen et~al.(2018)Chen, Rubanova, Bettencourt, and
  Duvenaud]{chen2018neural}
Ricky~TQ Chen, Yulia Rubanova, Jesse Bettencourt, and David Duvenaud.
\newblock Neural ordinary differential equations.
\newblock In \emph{Proceedings of the 32nd International Conference on Neural
  Information Processing Systems}, pages 6572--6583, 2018.

\bibitem[Dupont et~al.(2019)Dupont, Doucet, and Teh]{dupont2019augmented}
Emilien Dupont, Arnaud Doucet, and Yee~Whye Teh.
\newblock Augmented neural {ODEs}.
\newblock \emph{arXiv preprint arXiv:1904.01681}, 2019.

\bibitem[Finn et~al.(2017)Finn, Abbeel, and Levine]{finn2017model}
Chelsea Finn, Pieter Abbeel, and Sergey Levine.
\newblock Model-agnostic meta-learning for fast adaptation of deep networks.
\newblock In \emph{International Conference on Machine Learning}, pages
  1126--1135. PMLR, 2017.

\bibitem[Frigon(2007)]{frigon2007fixed}
Marlene Frigon.
\newblock Fixed point and continuation results for contractions in metric and
  gauge spaces.
\newblock \emph{Banach Center Publications}, 77:\penalty0 89, 2007.

\bibitem[Gribonval and Nikolova(2020)]{gribonval2020chara}
R{\'e}mi Gribonval and Mila Nikolova.
\newblock A characterization of proximity operators.
\newblock \emph{Journal of Mathematical Imaging and Vision}, 62:\penalty0
  773–789, 2020.

\bibitem[He et~al.(2016)He, Zhang, Ren, and Sun]{he2016deep}
Kaiming He, Xiangyu Zhang, Shaoqing Ren, and Jian Sun.
\newblock Deep residual learning for image recognition.
\newblock In \emph{Proceedings of the IEEE Conference on Computer Vision and
  Pattern Recognition}, pages 770--778, 2016.

\bibitem[Hong-Kun(2004)]{xu2004viscosity}
Xu~Hong-Kun.
\newblock Viscosity approximation methods for nonexpansive mappings.
\newblock \emph{Journal of Mathematical Analysis and Applications},
  298\penalty0 (1):\penalty0 279--291, 2004.

\bibitem[Li et~al.(2020)Li, Xiao, Fang, Dai, Xu, and Lin]{li2020training}
Jia Li, Mingqing Xiao, Cong Fang, Yue Dai, Chao Xu, and Zhouchen Lin.
\newblock Training neural networks by lifted proximal operator machines.
\newblock \emph{IEEE Transactions on Pattern Analysis and Machine
  Intelligence}, 2020.

\bibitem[Lorraine et~al.(2020)Lorraine, Vicol, and
  Duvenaud]{lorraine2020optimizing}
Jonathan Lorraine, Paul Vicol, and David Duvenaud.
\newblock Optimizing millions of hyperparameters by implicit differentiation.
\newblock In \emph{International Conference on Artificial Intelligence and
  Statistics}, pages 1540--1552. PMLR, 2020.

\bibitem[Massaroli et~al.(2020)Massaroli, Poli, Park, Yamashita, and
  Asma]{massaroli2020dissecting}
Stefano Massaroli, Michael Poli, Jinkyoo Park, Atsushi Yamashita, and Hajime
  Asma.
\newblock Dissecting neural odes.
\newblock In \emph{34th Conference on Neural Information Processing Systems,
  NeurIPS 2020}. The Neural Information Processing Systems, 2020.

\bibitem[Pedregosa(2016)]{pedregosa2016hyperparameter}
Fabian Pedregosa.
\newblock Hyperparameter optimization with approximate gradient.
\newblock In \emph{International conference on machine learning}, pages
  737--746. PMLR, 2016.

\bibitem[Rajeswaran et~al.(2019)Rajeswaran, Finn, Kakade, and
  Levine]{rajeswaran2019meta}
Aravind Rajeswaran, Chelsea Finn, Sham Kakade, and Sergey Levine.
\newblock Meta-learning with implicit gradients.
\newblock \emph{Advances in neural information processing systems}, 2019.

\bibitem[Razin and Cohen(2020)]{razin2020implicit}
Noam Razin and Nadav Cohen.
\newblock Implicit regularization in deep learning may not be explainable by
  norms.
\newblock \emph{Advances in Neural Information Processing Systems}, 33, 2020.

\bibitem[Revay et~al.(2020)Revay, Wang, and Manchester]{revay2020lipschitz}
Max Revay, Ruigang Wang, and Ian~R Manchester.
\newblock Lipschitz bounded equilibrium networks.
\newblock \emph{arXiv preprint arXiv:2010.01732}, 2020.

\bibitem[Rodr{\'\i}guez et~al.(2017)Rodr{\'\i}guez, Gonzalez, Cucurull,
  Gonfaus, and Roca]{rodriguez2016regularizing}
Pau Rodr{\'\i}guez, Jordi Gonzalez, Guillem Cucurull, Josep~M Gonfaus, and
  Xavier Roca.
\newblock Regularizing cnns with locally constrained decorrelations.
\newblock In \emph{International Conference on Learning Representations}, 2017.

\bibitem[Sabach and Shtern(2017)]{sabach2017first}
Shoham Sabach and Shimrit Shtern.
\newblock A first order method for solving convex bilevel optimization
  problems.
\newblock \emph{SIAM Journal on Optimization}, 27\penalty0 (2):\penalty0
  640--660, 2017.

\bibitem[Shelhamer et~al.(2017)Shelhamer, Long, and
  Darrell]{shelhamer2017fully}
Evan Shelhamer, Jonathan Long, and Trevor Darrell.
\newblock Fully convolutional networks for semantic segmentation.
\newblock \emph{IEEE Transactions on Pattern Analysis and Machine
  Intelligence}, 39\penalty0 (4):\penalty0 640--651, 2017.

\bibitem[Song et~al.(2012)Song, Smola, Gretton, Bedo, and
  Borgwardt]{song2012feature}
Le~Song, Alex Smola, Arthur Gretton, Justin Bedo, and Karsten Borgwardt.
\newblock Feature selection via dependence maximization.
\newblock \emph{Journal of Machine Learning Research}, 13\penalty0 (5), 2012.

\bibitem[Soudry et~al.(2018)Soudry, Hoffer, Nacson, Gunasekar, and
  Srebro]{soudry2018implicit}
Daniel Soudry, Elad Hoffer, Mor~Shpigel Nacson, Suriya Gunasekar, and Nathan
  Srebro.
\newblock The implicit bias of gradient descent on separable data.
\newblock \emph{The Journal of Machine Learning Research}, 19\penalty0
  (1):\penalty0 2822--2878, 2018.

\bibitem[Takada et~al.(2018)Takada, Suzuki, and
  Fujisawa]{takada2018independently}
Masaaki Takada, Taiji Suzuki, and Hironori Fujisawa.
\newblock Independently interpretable lasso: A new regularizer for sparse
  regression with uncorrelated variables.
\newblock In \emph{International Conference on Artificial Intelligence and
  Statistics}, pages 454--463. PMLR, 2018.

\bibitem[Vershynin(2018)]{vershynin2018high}
Roman Vershynin.
\newblock \emph{High-dimensional probability: An introduction with applications
  in data science}, volume~47.
\newblock Cambridge University Press, 2018.

\bibitem[Winston and Kolter(2020)]{winston2020monotone}
Ezra Winston and J~Zico Kolter.
\newblock Monotone operator equilibrium networks.
\newblock \emph{arXiv preprint arXiv:2006.08591}, 2020.

\bibitem[Xu(2002)]{xu2002iterative}
Hong-Kun Xu.
\newblock Iterative algorithms for nonlinear operators.
\newblock \emph{Journal of the London Mathematical Society}, 66\penalty0
  (1):\penalty0 240--256, 2002.

\bibitem[Yuan et~al.(2017)Yuan, Yang, and Zhang]{yuan2017feature}
Yuhui Yuan, Kuiyuan Yang, and Chao Zhang.
\newblock Feature incay for representation regularization.
\newblock \emph{arXiv preprint arXiv:1705.10284}, 2017.

\bibitem[Zhang et~al.(2021)Zhang, Cui, Xu, Zhou, He, and Shen]{zhang2021deep}
Xingxuan Zhang, Peng Cui, Renzhe Xu, Linjun Zhou, Yue He, and Zheyan Shen.
\newblock Deep stable learning for out-of-distribution generalization.
\newblock \emph{arXiv preprint arXiv:2104.07876}, 2021.

\end{thebibliography}
\bibliographystyle{plainnat}


\newpage
\onecolumn
\appendix

\section*{Appendix}
This Supplementary material section contains 
more detailed experimental results with different regularizer settings, the technical proofs of main theoretical results, and some auxiliary lemmas. 
\section{Additional Experimental Results}
\subsection{HSIC}
In this section, we introduce HSIC regularizer, which is a feature disentanglement method.
\par
HSIC is a statistical method to test independence. Compared with the decorrelation method we will present in the following section, HSIC can better capture the nonlinear dependency between random variables. We apply HSIC to the feature space. Many works~\cite{takada2018independently,bengio2013representation} show that when the features learned by the network are uncorrelated, the model usually obtains a good generalization performance. For any pair of random variables $\*X=(\*x_1,\cdots,\*x_B),\*Y=(\*y_1,\cdots,\*y_B)$, where $B$ is the batch size, we utilize the biased finite-sample estimator of HSIC~\cite{song2012feature}:
\[
\operatorname{HSIC}(\*X,\*Y)\coloneqq (B-1)^{-2}   \tr (\*K_X\*H\*K_Y\*H ),
\]
where $\*K_X$ and $\*K_Y$ are the kernel matrices w.r.t. Gaussian RBF kernel of $\*X$ and $\*Y$, and $\*H$ is the centering matrix $\*H=\*I-B^{-1} \*1_B \*1_B \in \*R^{B \times B}$. Following~\cite{zhang2021deep}, we aim to eliminate all the correlations between feature maps. To this end, our HSIC regularization is:
\[
\$R_z(\*Z)=\sum_{1\leq i<j \leq m} \operatorname{HSIC}(\*Z_{i,:},\*Z_{j,:}).
\]
Note that HSIC is a nonparametric regularization term, so it does not increase the parameter size of deep OptEq. The computing cost of HSIC grows as the batch size and feature dimension increase. Some tricks, such as Random Fourier Features approximation~\cite{zhang2021deep}, can be applied to speed up the calculation. In addition, Theorem \ref{thm:bilevelConvergence} is only guaranteed for convex regularization, while HSIC regularization is non-convex. In this paper, we report the great empirical superiority of HSIC and leave the above issues to future work.
\subsection{Different Settings for Regularizers}
On the dataset CIFAR-10, this experiment detailedly shows the effect of different settings on 
regularizer
$\$R^\dagger_z$ (for Sec.\ref{sec:SAM}), regularizer $\$R^*_z$ (for Sec.\ref{sec:Reg1}), and $\alpha$. 
\par
\textbf{Regularizers}. Here is the function we utilize to introduce the customized property to the equilibrium point, see Sec. \ref{sec:introProp} for more details. 
For the regularizer $\$R_z(\cdot)$ (both for $\$R^\dagger_z$ and  $\$R^*_z$), we set four different settings: (1)  $\$R_z(\*Z)=\sum_{1\leq i<j \leq m} \operatorname{HSIC}(\*Z_{i,:},\*Z_{j,:})$;
(2)$\$R_z(\*z) = \frac{1}{2}\norm{\*z}^2$; (3) $\$R_z(\*z) = 1/\qty(\norm{\*z}^2 + \epsilon)$ which is explored in \cite{yuan2017feature};
(4) Decorrelation: for the $B$-batch equilibrium points matrix $\*Z \in \^R^{m\times B}$:
\[
\$R_z(\*Z)= \frac{1}{2}\norm{\*D\*Z\*Z^\top\*D-\*I}_F^2  \coloneqq \$F_{Dz}(\*Z),
\]
where $\*D$ is a diagonal matrix whose non-zero entries are $\frac{1}{\norm{\*z_i}}$ and $\*z^i$ is the $i$-th row of the matrix $\*Z$. Note that $\$R_z(\*Z)$ here aims at reducing redundant information between feature dimensions, which has been discussed in \cite{rodriguez2016regularizing,ayinde2019regularizing}.
\par
\textbf{Settings}. In this experiment, we set $K=20$ and utilize weight decay to regularize the learnable parameters, i.e., $\$R_w(\cdot) = \xi\norm{\cdot}^2$, where we choose $\xi = 3e-4$.
We utilize the commonly used SGD to train the model. 
We set the learning rate as $0.1$ at the beginning and decay it by $0.7$ after every $20$ epochs. And the total training epoch is $200$.
The batch size is $125$ in this experiment.
We construct the deep OptEqs with $5$ convolutional layers, using five $3 \times 3$ convolution kernels with channels of $16, 32, 64, 128, 128$. The total number of learnable parameters is $199$k. 
\begin{table}[ht]
\centering
\caption{The testing accuracy of deep OptEq with different settings. We denote by ``no Reg'' the SAM with $\beta_k=0$. And the scripts $\dagger$ and $*$ means the regularizer given in Sec.\ref{sec:SAM} and Sec.\ref{sec:Reg1}, respectively.}
\resizebox{\textwidth}{!}{
	\begin{tabular}{c|cccccccc}
		\toprule
		$\alpha$	& no Reg & $ \qty(\frac{\norm{\cdot}^2}{2})^\dagger$ &$ \qty(\norm{\cdot}^2/2 + \epsilon)^{-1 }_\dagger $ & $\qty(\norm{\cdot}^2/2 + \epsilon)^{-1}_*$  & 
		 $\$F^\dagger_{Dz}$ & $\$F^*_{Dz}$ & HSIC$^\dagger$ & HSIC$^*$ \\
		\midrule
		0.01& 58.4\% & 69.4\% & \textbf{75.0\%}  & 64.9\%  & 72.7\%  &{72.6\%} & 70.6\% & 72.2\%  \\
		\midrule
		0.1 & 56.8\% & 61.9\% & {63.2\%}  &{64.7\%}  &  61.5\% &{60.0\%} & \textbf{66.1}\%  &64.5\%  \\
		\midrule
		0.4 & 86.9\% & 87.3\% & 78.4\%  &87.7\%  &86.5\%  &\textbf{88.0\%} &85.1\%& 85.5\%  \\
		\midrule
		0.8 & 87.4\% & 87.3\% & 87.3\%  &87.5\%  &87.0\%   &\textbf{87.6\%} &\textbf{87.6\%} &{87.5\%}  \\
		\midrule
		1.0 & 87.2\% & 87.4\% & 87.0\%  & 87.6\% &{87.7\%}  &{87.5\%}  & \textbf{88.1\%} &{87.9\%}  \\
		\bottomrule
\end{tabular}  \label{tab:regsRes}
}
\end{table}
\par
\textbf{Results}. The results with different regularizers are presented in Table~\ref{tab:regsRes}.
We can see that either adopting SAM iteration or changing the underlying convex optimization problem both improves classification performance.
Note that, given the same type of regularization, modifying the underlying optimization problem, i.e., using $\$R^*_z$, usually make more improvements.
Indeed, to modify the underlying optimization problem, we need to change the architecture of deep OptEq, which has a more direct impact on the model than turning the training loss by $\$R^\dagger_z$.
When $\alpha=0.01$, deep OptEq is almost equivalent to the one-layer wide OptEq (see Theorem \ref{thm:asyalpha}), which is far outperformed by deep OptEq for $\alpha>0.1$.
Compared with other results, $\alpha = 0.1$ gives a poor result, which implies that the performance is not monotonic to parameter $\alpha$.
Fortunately, from the table, setting $\alpha>0.4$ is a safe choice.
We notice that the overall performance of feature disentanglement methods (decorrelation and HSIC) are better than the other types of regularization terms whether we utilize it as $\$R^\dagger_z$ or $\$R^*_z$.
\section{Discussion about the Training Methods}
The IFT based implicit way utilizes the limited memory to train the model and is insensitive to the equilibrium point finding algorithms. However, it consumes much computation budget to solve the equation during the inference and BP.
On the other hand, the way that unrolls the fixed point finding method may induce implicit bias \cite{soudry2018implicit,razin2020implicit} and consumes much memory during training,  but it is faster to infer and train. Note that implicit bias is a two-edged sword; 
The proposed SAM method can aggregate the information from the prior regularization and the fixed point equation. Hence, the implicit bias becomes a controllable inductive bias.
\par
The tradeoff between memory and computing efficiency for the implicit and unrolling training methods is quite common in the other learning community, such as meta-learning~\cite{finn2017model,rajeswaran2019meta} and hyper-parameter optimization~\cite{pedregosa2016hyperparameter,lorraine2020optimizing}. Similarly, for DEQ, the two training ways are neither good nor bad. We should choose them in proper circumstances.
\section{Convergence Analysis}\label{sec:train}
This section offers the convergence results: (i) the sequence generated by Eq.(\ref{eq:InclusinApprox}) converges to some point $\*z^* \in \operatorname{Fix}(\$T)$ such that $\$R_z(\*z^*)\leq \$R_z(\*z),~\forall \*z \in \operatorname{Fix}(\$T)$; 
(ii) gradient descent can find a global minimum for the approximately implicit model in Eq.(\ref{eq:lossapprox}).
\subsection{Approximation of Equilibrium Point}
Note that we approximate the points $ \*z^{*}(\*x, \bm{\theta})$ by the iterative steps in Eq.(\ref{eq:InclusinApprox}). 
In fact, we take the iterative step by extending an existing algorithm, called Sequential Averaging Method (SAM), which was developed in \citet{xu2004viscosity} for solving a certain class of fixed-point problems, and then was applied to the bi-level optimization problems \cite{sabach2017first}. 
However, the existing SAM method can only deal with strongly convex $\$R_z(\*z)$.
Our method is the first SAM type algorithm that can solve the general convex problem restricted to a nonexpansive operator's fixed point set.
The following theorem provides the formal statement and the required conditions.
Since during the forward-propagation, $(\bm{\theta},\*x)$ is fixed, for the sake of convenience, we simplify $\$T(\*z,\*x, \bm{\theta})$ as $\$T(\*z)$.
\begin{theorem}[Convergence of Modified SAM Iterates]\label{thm:bilevelConvergence}
Suppose that $\nabla \$R_z(\*z)$ is $L_z$-Lipschitz, and that for any $\beta \in [0,\frac{1}{2}], \lambda \in [0,\frac{L_z}{2}]$, the fixed point set of equation: $\*z=\beta\qty(\*z-\gamma \qty( \nabla \$R_z(\*z)+\lambda \*z ))+(1-\beta)\$T(\*z)$ is uniformly bounded by $B_1^*$ (in norm $\norm{\cdot}$) w.r.t. $\beta$ and $\lambda$. Suppose that convex function $\$R_z(\*z)$ has a unique minimizer $\Bar{\*z}$ on $\operatorname{Fix}(\$T)$. Let $\beta_k=\frac{\eta}{k^\rho},\lambda_k=\frac{\eta}{k^c}, \gamma = \frac{1}{2L_z}$, where $\rho,c>0 $, $ \rho +2c<1$ and $ \eta=\min\qty{\sqrt{2L_z}, \frac{L_z}{2} , \frac{1}{2}}$,
then the sequence $\{\*z^{k}\}_{k \in\^N^+}$ generated by Eq.(\ref{eq:InclusinApprox}) converge to $\Bar{\*z}$.
\end{theorem}
The formal assumptions of Theorem \ref{thm:bilevelConvergence} seem complicated, however, they can be easily fulfilled when $\$T(\*z)$ is contractive. 
A sufficient condition that makes $\$T(\*z)$ contractive is to let one $\|\*W_i\|_2\leq \zeta <1$.
More specifically, if some $\|\*W_i\|_2\leq \zeta < 1$, then $\*z\mapsto \beta\qty(\*z-\gamma ( \nabla \$R_z(\*z)+\lambda \*z ))+(1-\beta)\$T(\*z)$ is contractive and has a unique fixed point, which depends continuously on $\beta$ and $\lambda$ \cite{frigon2007fixed}, and thereby have a uniform bound.
\vspace{-2mm}
\subsection{Global Convergence of Implicit Model}
Most previous works on DEQs lack the convergence guarantees for their training. 
However, analyzing the learnable parameters' dynamics is crucial since it may weaken many model constraints and greatly broader the function class that the implicit model can represent.
For example, the one-layer DEQ, given in \citet{winston2020monotone}, maintains the positive definiteness of $(\*I-\*W)$ for all weight $\*W$ in $\^R^{m\times m}$ through a complicated parameterization technique.
However, after analysis, we find that the learnable weight will stay in a small compact set during training, thus, we may only need the positive definiteness within a local region instead of global space for the DEQ \cite{winston2020monotone}.
\par
\begin{theorem}[Global Convergence (informal)]\label{thm:global}
Suppose that the initialized weight $\*W_l$'s singular values are lower bounded away from zero for all $l\in [1,L+1]$,  and the fixed point set $\operatorname{Fix}(\$T(\cdot,\*X,\bm{\theta}))$ is non-empty and uniformly bounded for any $\widetilde{\bm{\theta}}$ in a pre-defined compact set.
Assume that the activation function is Lipschitz smooth, strongly monotone and 1-Lipschitz. 
Define constants $Q_0$, $Q_1$ and $Q_3$, which depend on the bounds for initialization parameters, initial loss value, and the datasize. 
Let the learning rate be $\eta < \min\{\frac{1}{Q_0},\frac{1}{Q_1}\}$.
If the training data size $N$ is large enough, then the training loss vanishes at a linear rate as:
$
\ell(\widetilde{\bm{\theta}}^t) \leq \ell(\widetilde{\bm{\theta}}^0)\qty(1-\eta Q_0)^t$,
where $t$ is the number of iteration. 
Furthermore, the network parameters also converge to a global minimizer $\widetilde{\bm{\theta}}^*$ at a linear speed:$
\|\widetilde{\bm{\theta}}^t - \widetilde{\bm{\theta}}^*\|\leq Q_3 \qty(1-\eta Q_0)^{t/2}
$.
\end{theorem}
Theorem \ref{thm:global} shows that GD converges to a global optimum for any initialization satisfying the boundedness assumption.
In general, the lower bounded assumption on singular values is easy to fulfill. 
With high probability, the weight matrix's singular values are lower bounded away from zero when it is a rectangle and has independent, sub-Gaussian rows or has independent Gaussian entries, see Thm.4.6.1 and Ex.7.3.4 in \cite{vershynin2018high}.
Similar to the remark after Theorem \ref{thm:bilevelConvergence}, the existence and boundedness assumption on the set $\operatorname{Fix}(\$T)$ is mild.
Suppose that we have one weight that satisfies $\*W_l\leq \zeta <1$, then the fixed point of $\$T$ exists and is unique, and is continuous w.r.t. the parameters $\widetilde{\bm{\theta}}$. 
On the other hand, the parameters $\widetilde{\bm{\theta}}$ will stay in a compact set during training.
 Therefore, the fixed points are uniformly upper bounded.

\section{Proofs for the DNN Reformulation}
\subsection{Proof of  Lemma \ref{lem:reformulate}}
The formal proof for Lemma \ref{lem:reformulate} relies on the following auxiliary lemma.
\begin{lemma}\label{lem:auxRef}
If $k\geq 2 \max\{m,n\}$, given any $\*W\in \^R^{m\times n}$, and a full rank matrix $\*A \in \^R^{m\times k}$, there exists a full rank matrix $\*B \in \^R^{n\times k}$, such  that $\*W = \*A\*B^\top$.
\end{lemma}
\begin{proof}
Considering the full SVD of $\*A = \*U \bm{\Sigma} \*V^\top$, where $\*U \in \^R^{m\times m}$, $\bm{\Sigma}\in \^R^{m\times k}$, and $\*V \in \^R^{k\times k}$.
Let 
\[
\*U^\top \*W \coloneqq \bm{\Omega} = 
\left[\begin{array}{c}
\bm{\Omega}_{1} \\
\bm{\Omega}_{2} \\
\vdots \\
\bm{\Omega}_{m}
\end{array}\right].
\]
Considering the equation:
\[
\bm{\Omega} = \bm{\Sigma} \*C.
\]
We can easily find that:
\[\*C = 
\left[\begin{array}{c}
\bm{\Omega}_{1}/\sigma_1\\
\bm{\Omega}_{2}/\sigma_2 \\
\vdots \\
\bm{\Omega}_{m}/\sigma_m\\
\bm{*}
\end{array}\right],
\]
is a solution, and we let $\rank(\*C) = n$ by adjusting $\bm{*}$. 
We let $\*B = \*C^\top\*V^\top$, hence, we can conclude that $\rank(\*B) = n$.
It is easy to verify that $\*W = \*A\*B^\top$.
\end{proof}
Recall that, we have:
\[
\begin{aligned}
\*y= & \*W_{L} \sigma \qty\Big( \*W_{L-1} \sigma \qty( \cdots\*W_2 \sigma(\*W_1\*z_0 + \*U_1\*x + \*b_1)\cdots )   )\\
= & \underbrace{\overline{\*W}_{L}~\overline{\*W}_{L-1}^\top}_{{\*W}_{L}} \sigma \qty( \underbrace{\overline{\*W}_{L-1}~\overline{\*W}_{L-2}^\top}_{{\*W}_{L-1}} \sigma \qty( \cdots\underbrace{\overline{\*W}_{2}~\overline{\*W}_{1}^\top}_{{\*W}_{2}} \sigma(\underbrace{\overline{\*W}_{1}~\overline{\*W}_{0}^\top}_{{\*W}_{0}}\*z_0 + \*U_1\*x + \*b_1)\cdots )   ).
\end{aligned}
\]
Based on the above Lemma \ref{lem:auxRef}, the existence of $\overline{\*W}_k \in \^R^{n_k\times m}$ for all $k$ can be easily guaranteed by giving a rank $m$ matrix $\overline{\*W}_{L}$. Note that $\forall k, \rank(\overline{\*W}_k) = m$ and $m \geq \max\{n_k,n_{k-1}\}$.

\subsection{Proof of Lemma \ref{lem:allDNNs}}
\begin{proof}
Set $\alpha=1$ and let $\*W_1 = \*0, \*U_1 = \*U_3 = \cdots = \*U_L = 0$ and $\*b_i = 0$ for all $i$, we have:
\[
\*A_L \sigma\qty\Big(\*A_{L-1}\sigma \cdots\sigma(\*A_2\sigma(\*A_1\*x))) = 
 \underbrace{\*W_{L+1}\*W_L^\top}_{\*A_L}\sigma\qty\bigg(\underbrace{\*W_{L}\*W^\top_{L-1}}_{\*A_{L-1}} \sigma \cdots \sigma \qty\Big( \underbrace{\*W_3\*W^\top_2}_{\*A_2}\sigma\qty\Big(\underbrace{\*U_2}_{\*A_1}\*x))).
\]
The existence of $\qty{\*W_{l}}_{l=2}^{L+1}$ can be easily obtained by  Lemma \ref{lem:auxRef}.
The bias term $\*b_i$ can be easily included by changing each layer's output $\*x_i$ to $\qty[\*x_i;1]$ and set $\*A_i$ to $\mqty[\*A_i & \*b_i \\ 0 & 1] $.
\end{proof}

\section{Proofs for the Connection between Optimization and OptEq}
\subsection{Conditions to be a Proximal Operator}
\begin{lemma}[modified version of Prop. 2 in \citet{gribonval2020chara}]\label{lem:proxandsubg}
Consider $f: \$H\rightarrow \$H $ defined everywhere. The following properties are equivalent:
\begin{enumerate}[label=(\roman*)]
    \item there is a proper convex l.s.c function $\varphi:\$H \rightarrow \^R\cup\{+\infty\}$ such that $f(\*z) \in \operatorname{prox}_{\varphi}(\*z)$ for each $\*z\in \$H$;
    \item \label{lem:proxandsubgii} the following conditions hold jointly:
    \begin{enumerate}[label=(\alph*)]
        \item there exists a convex l.s.c function $\psi: \$H \rightarrow \^R$ such that  $\forall \*y \in \$H, f(\*y) = \nabla \psi(\*y)$;
        \item $\norm{f(\*y)-f(\*y')}\leq \norm{\*y-\*y^{\prime}},~ \forall \*y,~\*y^{\prime} \in \$H$.
    \end{enumerate}
\end{enumerate}
There exists a choice of $\varphi(\cdot)$ and $\psi(\cdot)$, satisfying (i) and (ii), such that
$\varphi(\*x)=\psi^{*}(\*x)-\frac{1}{2}\|\*x\|^2$.
\end{lemma}
\begin{proof}
(i)$\Rightarrow$(ii):
Since $\varphi (\*x)+\frac{1}{2}\|\*x\|^2$ is a proper l.s.c 1-strongly convex function, then by Thm. 5.26 in \citet{beck2017first}, i.e, the conjugate function $f^*$ is $\frac{1}{\sigma}$-smooth when $f$ is proper, closed and $\sigma$ strongly convex and vice versa.
Thus, we have:
\[
\psi(\*x):=\qty[\varphi (\*x)+\frac{1}{2}\norm{\*x}^2]^*,
\]
is 1-smooth with $\operatorname{dom}(\psi)=\$H$. Then we get:
\[
\begin{aligned}
 f(\*x) &\in \argmin\limits_{\*u} \frac{1}{2}\norm{\*u-\*x}^{2}+ \varphi(\*u)
 = \{\*u \mid \*x \in \partial \varphi(\*u)+ \*u   \}\\
&= \qty{\*u \mid \*x \in \partial \qty(\varphi(\*u)+\frac{1}{2}\norm{\*u}^2  ) }\\
&= \qty{\*u \mid \*u=\nabla \psi (\*x) }=  \qty{\nabla \psi (\*x) }
\end{aligned}
\]
where the third equality comes from Thm. 4.20 in in \citet{beck2017first}, which notes that $\*y \in \partial f(\*x)$ is equivalent to $\*x \in \partial f^*(\*y)$. 
Hence $f(\*x)=\nabla \psi (\*x)$, and 1-smoothness of $\psi$ implies $f$ is nonexpansive.
\par
(ii)$\Rightarrow$(i): Let $\varphi(\*x)=\psi^*(\*x)-\frac{1}{2}\|\*x\|^2$. Since $\psi$ is 1-smooth, similarly we can conclude: $\psi^*$ is 1-strongly convex. Hence, $\varphi$ is convex, and:
\[
\begin{aligned}
\operatorname{prox}_{\varphi}(\*x) &=  \argmin\limits_{\*u} \qty{ \frac{1}{2}\norm{\*u-\*x}^{2}+ \varphi(\*u)}\\
& = \qty{\*u \mid \*x \in \partial \varphi(\*u)+\*u   }\\
& = \qty{\nabla \psi (\*x) }=\qty{f(\*x)},
\end{aligned}
\]
which means $f(\*x)=\operatorname{prox}_{\varphi}(\*x)$.
\end{proof}

\subsection{Proof for Theorem \ref{thm:core}}
\begin{proof}
In the proof, w.l.o.g, we let $\mu=1$ for the ease of presentation, and hence let $\Tilde{L}_\sigma=1$ and $\norm{\*W}\leq 1$.
Since $ \mathbf{1}^\top  \tilde{\sigma}(\*y) = \sum_{i=1}^n \tilde{\sigma}(y_i)$, we have $\nabla\qty(\tilde{\sigma}(\*y))  =[\sigma(y_1),\cdots,\sigma(y_n)]^{\top}=\sigma(\*y)$, by the chain rule, $\nabla{\psi}(\*z)=\*W^\top\sigma\qty(\*W\*z + \*b)=f(\*z)$.
\par
Since $\sigma(a)$ is a single-valued function with slope in $[0,1]$, the element-wise defined operator $\sigma(\*a)$ is nonexpansive (see the definition in Lemma \ref{lem:proxandsubg}). Combining with $\norm{\*W}_2 \leq 1$, operator $f(\*z)=\*W^\top\sigma\qty(\*W\*z + \*b)$ is also nonexpansive.
\par
Due to Lemma \ref{lem:proxandsubg}, we have $f(\*z) = \operatorname{prox}_{\varphi}(\*z)$, and $\varphi(\*z)$ can be chosen as $\varphi(\*z) = \psi^{*}(\*z) - \frac{1}{2}\norm{\*z}^2 $
\end{proof}

\subsection{Proof of Lemma \ref{lem:cyclFixedPoint}}
\begin{proof}
We can rewrite deep OptEq $\*z = f_{L}\circ f_{L-1}\cdots \circ f_1(\*z,\*x, \bm{\theta})$ in a separated form:  let $\*z=\*z_0$, 
\begin{equation}\label{eq:SepDeq}
\left\{
\begin{array}{l}
\*z_{1}=\alpha \*W_{1}^{\top} \sigma\left(\*W_{1} \*z_{0}+\*U_{1} \*x+\*b_{1}\right)+(1-\alpha) \*z_{0} \\
\*z_{2}=\alpha \*W_{2}^{\top} \sigma\left(\*W_{2} \*z_{1}+\*U_{2} x+\*b_{2}\right)+(1-\alpha) \*z_{1} \\
\vdots \\
\*z_{L-1}=\alpha \*W_{L-1}^{\top} \sigma\left(\*W_{L-1} \*z_{L-2}+\*U_{L-1} x+\*b_{L-1}\right)+(1-\alpha) \*z_{L-2} \\
\*z_0=\alpha \*W_{L}^{\top} \sigma\left(\*W_{L} \*z_{L-1}+\*U_{L} x+\*b_{L}\right)+(1-\alpha) \*z_{L-1}
\end{array},
\right.
\end{equation}
and it also has a compact matrix form:
\begin{equation}\label{eq:MatDeq}
\resizebox{\hsize}{!}{$
\begin{aligned}
\left[\begin{array}{c}
\*z_{1} \\
\*z_{2} \\
\vdots \\
\*z_{L-1} \\
\*z_{0}
\end{array}\right]=
\alpha&\left[\begin{array}{ccccc}
\*W_{1}^\top & & & &  \\
 & \*W_{2}^\top& & & \\
&  & \*W_{3}^\top & & \\
& &  & \ddots & \\
& & &  & \*W_{L}^\top
\end{array}\right]
\sigma\left(\left[\begin{array}{ccccc}
0 & & & & \*W_1 \\
\*W_2 & 0 & & & \\
 & \*W_3 & \ddots & & \\
 &  & \ddots & \ddots & \\
 &  &  & \*W_L & 0
\end{array}\right]\left[\begin{array}{c}
\*z_{1} \\
\*z_{2} \\
\vdots \\
\*z_{L-1 } \\
\*z_{0}
\end{array}\right] +\right.\\
&\left. \left[\begin{array}{c}
\*U_{1} \\
\*U_{2} \\
\vdots \\
\*U_{L-1 } \\
\*U_{L}
\end{array}\right]x+
\left[\begin{array}{c}
\*b_{1} \\
\*b_{2} \\
\vdots \\
\*b_{L-1 } \\
\*b_{L}
\end{array}\right]\right)+(1-\alpha)\*P \left[\begin{array}{c}
\*z_{1} \\
\*z_{2} \\
\vdots \\
\*z_{L-1} \\
\*z_{0}
\end{array}\right],
\end{aligned}$}
\end{equation}
where $\*P =\left[\begin{array}{ccccc}
0 & & & & \*I \\
\*I & 0 & & & \\
& \*I & 0 & & \\
& & \ddots & \ddots & \\
& & & \*I & 0
\end{array}\right]$ is a permutation matrix.
\par
Hence a multi-layer deep OptEq is actually a single-layer OptEq with multi-blocks.
\end{proof}

\subsection{Proof of Theorem \ref{thm:zeroset}}
Before proving the main results, we first present an auxiliary lemma. 
\begin{lemma}[an extension of Lemma \ref{lem:proxandsubg}]\label{lem:D_x_yandsubg}
Consider $f: \$H\rightarrow \$H $ defined everywhere, $\$A: \$H\rightarrow \$H $ is any invertible $1$-Lipschitz operator. The following properties are equivalent:
\begin{enumerate}[label=(\roman*)]
    \item there is a proper convex l.s.c function $\varphi:\$H \rightarrow \^R\cup\{+\infty\}$ such that $f(\*z) \in \mathop{\argmin}\limits_{\*u} \qty{\frac{1}{2}\norm{\*u}^{2} -\innerprod{\*u,\$A(\*z)}+ \varphi(\*u)}$ for each $\*z\in \$H$;
    \item  the following conditions hold jointly:
    \begin{enumerate}[label=(\alph*)]
        \item there exists a convex l.s.c function $\psi: \$H \rightarrow \^R$ such that for each $\*y \in \$H, f\qty(\$A^{-1}(\*y)) = \nabla \psi(\*y)$;
        \item f is nonexpansive, i.e. $\norm{f(\*y)-f(\*y')}\leq \norm{\*y-\*y^{\prime}}$,\quad $\forall \*y,~\*y^{\prime} \in \$H$.
    \end{enumerate}
\end{enumerate}
Moreover, there exists a choice of $\varphi(\cdot)$, $\psi(\cdot)$, satisfying (i) (ii), such that
$\varphi(\*x)=\psi^{*}(\*x)-\frac{1}{2}\|\*x\|^2$.
\end{lemma}
\begin{proof}
(i)$\Rightarrow$(ii):
Since $\varphi (\*x)+\frac{1}{2}\|\*x\|^2$ is a proper l.s.c 1-strongly convex function, then
$\psi(\*x):=\qty[\varphi (\cdot)+\frac{1}{2}\|\cdot\|^2]^*(\*x)$ is 1-smooth with $\operatorname{dom}(\psi)=\$H$.
Note that $\psi\circ \$A(\cdot)$ is $1$-smooth due to $\$A(\cdot)$ is $1$-Lipschitz.
Moreover, we have:
\begin{equation}
\begin{aligned}
 f(\*x) &\in \argmin_{\*u} \qty{\frac{1}{2}\norm{\*u}^{2} -\innerprod{\*u,\$A(\*x)}+ \varphi(\*u)}\\
& = \qty{\*u \mid \$A(\*x) \in \qty(\partial \varphi(\*u)+ \*u)   }\\
&= \qty{\*u \mid \$A(\*x) \in \partial \qty(\varphi(\*u)+\frac{1}{2}\norm{\*u}^2  ) }\\
&= \qty{\*u \mid \*u=\nabla \psi\qty(\$A(\*x)) }=  \qty{\nabla \psi \qty(\$A(\*x)) },
\end{aligned}
\end{equation}
Hence $f(\*x)=\nabla \psi \qty(\$A(\*x))$, and $1$-smoothness of $\psi\circ \$A(\cdot)$ implies $f$ is nonexpansive.
\par
(ii)$\Rightarrow$(i): Let $\varphi(\*x)=\psi^*(\*x)-\frac{1}{2}\|\*x\|^2$. Since $\psi$ is 1-smooth, then $\psi^*$ is 1-strongly convex, and $\varphi$ is convex. 
Note that:
\[
\resizebox{\hsize}{!}{$
\mathop{\argmin}\limits_{\*u}
\qty{\frac{1}{2}\norm{\*u}^{2} -\innerprod{\*u,\$A(\*x)}+ \varphi(\*u)}
 = \qty{\*u \mid \$A(\*x)\in \qty(\partial \varphi(\*u)+\*u)}= \qty{\nabla \psi \qty(\$A(\*x)) }=\qty{f(\*x)},
 $}
\]
which means $f(\*x)=\mathop{\argmin}\limits_{\*u} \qty{\frac{1}{2}\norm{\*u}^{2} -\innerprod{\*u,\$A(\*x)}+ \varphi(\*u)} $.
\end{proof}
Now, we are ready to prove the main theorem.
\begin{proof}
Denote the right hand side of equation \ref{eq:MatDeq} as $F(\widetilde{\*z})$ (For convenience of notation, we omit $\*x, \bm{\theta}$ here), then
\[
\resizebox{\hsize}{!}{$
\begin{aligned}
F(\*P^\top \widetilde{\*z}) 
= \alpha & \left[\begin{array}{ccccc}
\*W_{1}^\top & & & &  \\
 & \*W_{2}^\top& & & \\
&  & \*W_{3}^\top & & \\
& &  & \ddots & \\
& & &  & \*W_{L}^\top
\end{array}\right]
\sigma\left(\left[\begin{array}{ccccc}
\*W_{1} & & & &  \\
 & \*W_{2}& & & \\
&  & \*W_{3} & & \\
& &  & \ddots & \\
& & &  & \*W_{L}
\end{array}\right]
\left[\begin{array}{c}
\*z_{1} \\
\*z_{2} \\
\vdots \\
\*z_{L-1 } \\
\*z_{0}
\end{array}\right] + 
\right.\\ 
&\left.   \left[\begin{array}{c}
\*U_{1} \\
\*U_{2} \\
\vdots \\
\*U_{L-1 } \\
\*U_{L}
\end{array}\right]x+
\left[\begin{array}{c}
\*b_{1} \\
\*b_{2} \\
\vdots \\
\*b_{L-1 } \\
\*b_{L}
\end{array}\right]\right)+(1-\alpha) \left[\begin{array}{c}
\*z_{1} \\
\*z_{2} \\
\vdots \\
\*z_{L-1} \\
\*z_{0}
\end{array}\right]\\
= &\left[\begin{array}{c}
\alpha \nabla \psi_{1}(\*z_{1})+(1-\alpha)\*z_{1} \\
\alpha \nabla \psi_{2}(\*z_{2})+(1-\alpha)\*z_{2} \\
\vdots \\
\alpha \nabla \psi_{L-1}(\*z_{L-1})+(1-\alpha)\*z_{L-1} \\
\alpha \nabla \psi_{L}(\*z_{0})+(1-\alpha)\*z_{0}
\end{array}\right]\\
=  \nabla & \left[\alpha \Psi (\widetilde{\*z})+(1-\alpha)\frac{1}{2}\norm{\widetilde{\*z}}^2\right],
\end{aligned}$}
\]
where $\nabla \psi_i (\*z_i)\coloneqq \*W_i^\top \sigma (\*W_i \*z_i +\*U_i \*x +\*b_i )$ and $\Psi (\widetilde{\*z})\coloneqq\psi_{1}(\*z_{1})+\psi_{2}(\*z_{2})+\cdots+\psi_{L-1}(\*z_{L-1})+\psi_{L}(\*z_{0})$.
\par
Given $\norm{\*W_i}_2 \leq 1, \forall i \in [1,L]$, $\nabla \psi_i (\*z_i)$ is a nonexpansive operator. 
Then for $\forall \, \widetilde{\*z}_1,\widetilde{\*z}_2$, we have:
\[\begin{aligned}
\|F(\*P^\top \widetilde{\*z}_1)-F(\*P^\top \widetilde{\*z}_2)\|^2
=&\sum_{i=1}^{L}\norm{\alpha \qty(\nabla\psi_{i}(\*z_{1,i})-\nabla\psi_{i}(\*z_{2,i}))+(1-\alpha)(\*z_{1,i}-\*z_{2,i})}^2\\
\leq& \sum_{i=1}^{L}(\alpha \norm{\qty(\nabla\psi_{i}(\*z_{1,i})-\nabla\psi_{i}(\*z_{2,i}))\|+(1-\alpha)\|(\*z_{1,i}-\*z_{2,i})})^2\\
\leq & \sum_{i=1}^{L}(\alpha \norm{\qty(\*z_{1,i}-\*z_{2,i})\|+(1-\alpha)\|(\*z_{1,i}-\*z_{2,i})})^2\\
=& \|\widetilde{\*z}_1-\widetilde{\*z}_2 \|^2.
\end{aligned}
\]
By the results, we have\[
\norm{F( \widetilde{\*z}_1)-F( \widetilde{\*z}_2)}=\norm{F(\*P^\top \*P \widetilde{\*z}_1)-F(\*P^\top \*P \widetilde{\*z}_2)}\leq \|\*P \widetilde{\*z}_1-\*P \widetilde{\*z}_2\|=\|\widetilde{\*z}_1-\widetilde{\*z}_1\|,
\]
which means $F$ is nonexpansive. 
By Lemma \ref{lem:D_x_yandsubg}, let $\Phi(\widetilde{\*z})=\qty[\alpha \Psi (\cdot)+(1-\alpha)\frac{1}{2}\|\cdot\|^2]^*(\widetilde{\*z})-\frac{1}{2}\norm{\widetilde{\*z}}^2$, we have:
\[
F(\widetilde{\*z})\in  \mathop{\argmin}\limits_{\*u} \qty { \frac{1}{2}\norm{\*u}^{2} -\langle\*u,\*P\widetilde{\*z}\rangle+ \Phi(\*u)}.
\]
Hence any fixed point $\widetilde{\*z}$ of Eq.(\ref{eq:MatDeq}) satisfies:
\[
0\in\partial \Phi (\widetilde{\*z}^*)+(\*{I}-\*P)\widetilde{\*z}^*,  
\]
By using Thm. 4.14:
\[\text{if}\, h(\*x)=\alpha f(\frac{\*x}{\alpha}) \,\text{then} \,h^*(\*y)=\alpha f^*(\*y),
\]
Thm. 4.19:
\[
 (h_1^*+h_2^*)^*(\*x)=\inf_{\*u}\{h_1(\*u)+h_2(\*x-\*u)\},
\] 
in~\citet{beck2017first} and the definition of Moreau envelope:
\[ M_f^\mu (\*x)=\inf_{\*u}\qty{f(\*u)+\frac{1}{2\mu}\|\*x-\*u\|^2},\]
$\Phi(\widetilde{\*z})$ can be formed in terms of $\varphi_i(\*z_i)$:
\[
\begin{aligned}
\Phi(\widetilde{\*z})&=\left[\alpha \Psi (\cdot)+(1-\alpha)\frac{1}{2}\|\cdot\|^2\right]^*(\widetilde{\*z})-\frac{1}{2}\|\widetilde{\*z}\|^2\\
&=\sum_{i=1}^{L}\left\{\left[\alpha \psi_i (\cdot)+(1-\alpha)\frac{1}{2}\|\cdot\|^2\right]^*(\*z_i)-\frac{1}{2}\|\*z_i\|^2\right\}\\
&=\sum_{i=1}^{L}\left\{\left[\left(\alpha \psi_i^* (\frac{\cdot}{\alpha})\right)^*+\left((1-\alpha)\frac{1}{2}\norm{\frac{\cdot}{1-\alpha}}^2\right)^*\right]^*(\*z_i)-\frac{1}{2}\|\*z_i\|^2\right\}\\
&=\sum_{i=1}^{L}\inf_{\*y_i}\left\{ \alpha \psi_i^* (\frac{\*y_i}{\alpha}) +   (1-\alpha)\frac{1}{2}\norm{\frac{\*y_i-\*z_i}{1-\alpha}}^2  -\frac{1}{2}\|\*z_i\|^2\right\}\\
&=\sum_{i=1}^{L}\inf_{\*y_i}\left\{  \alpha \psi_i^*(\*y_i)+\frac{1}{2(1-\alpha)}\norm{\alpha \*y_i-\*z_i}^2 -\frac{1}{2}\|\*z_i\|^2      \right\}\\
&=\sum_{i=1}^{L}\inf_{\*y_i}\left\{  \alpha \left( \varphi_i(\*y_i)+\frac{1}{2}\|\*y_i\|^2 \right)+\frac{1}{2(1-\alpha)}\|\alpha \*y_i-\*z_i\|^2 -\frac{1}{2}\|\*z_i\|^2      \right\}\\
&=\sum_{i=1}^{L}\inf_{\*y_i} \left\{ \alpha \varphi_i(\*y_i) +\frac{\alpha}{2(1-\alpha)}\|\*y_i-\*z_i\|^2      \right\}\\
&=\sum_{i=1}^{L}\alpha M_{\varphi_i}^{1-\alpha}(\*z_i).
\end{aligned}
\]
\end{proof}
\subsection{Proof for Corollary \ref{coro:twoblock}}
\begin{proof}
When $L=2$, $\*{I}-\*P$ is a symmetric matrix 
\[
\left[\begin{array}{ll}
\ \ \*{I} & -\*{I} \\
-\*{I} & \ \ \*{I}
\end{array}\right],\] 
and the operator
\[\left[\begin{array}{ll}
\ \ \*{I} & -\*{I} \\
-\*{I} & \ \ \*{I}
\end{array}\right] \left[\begin{array}{l}
\*z_1 \\
\*z_0
\end{array}\right],
\]
is the gradient of convex function $\frac{1}{2}\|\*z_1-\*z_0\|^2$. 
Hence the monotone operator splitting equation Eq.(\ref{eq:optCondition}) is now an first optimality condition of convex optimization problem:
\[
  \min_{\*z_1,\*z_0} \qty{\alpha M_{\varphi_1}^{1-\alpha}(\*z_1)+ \alpha M_{\varphi_2}^{1-\alpha}(\*z_0) + \frac{1}{2} \norm{\*z_1 - \*z_0}^2}.
\]
\end{proof}

\subsection{Proof for Theorem \ref{thm:asyalpha}}
For proving the results, we need a lemma from \citet{frigon2007fixed}.
\begin{lemma}\label{lem:contifixed}
$\$H$ is a Hilbert space, and $k<1$. If $\*T_n: \$H\rightarrow \$H $ is a k-contraction, for all $n \in \mathbb{N}^+ \cup\{0\}$, and  $\*T_{n} \rightarrow \*T_{0}$  point-wisely. Then the fixed point of  $\*T_n$  tends to the fixed point of  $\*T_0$  when  $n\rightarrow \infty$.
\end {lemma}
\begin{proof} 
Given $\operatorname{prox}_{\varphi_i}(\*z) = \*W_i^\top \sigma\qty(\*W_i\*z + \*U_i\*x+ \*b_i)$, $\*z_0^*(\alpha)$ is the fixed point of composition equation:
\[
\*z=\left[\alpha\operatorname{prox}_{\varphi_L}+(1-\alpha)\${I} \right]\circ \cdots \circ \left[\alpha\operatorname{prox}_{\varphi_1}+(1-\alpha)\${I} \right](\*z)
\]
For $\alpha \in (0,1)$, we get:
\begin{equation}\label{eq:trasformfixed}
    \begin{aligned}
    \*z=&\frac{\left[\alpha\operatorname{prox}_{\varphi_L}+(1-\alpha)\${I} \right]\circ \cdots \circ \left[\alpha\operatorname{prox}_{\varphi_1}+(1-\alpha)\${I} \right]-\${I}}{L\alpha}(\*z)+\*z\\
    =&\left(\frac{(1-\alpha)^L-1}{L\alpha}+1 \right)\*z+\frac{\alpha(1-\alpha)^{L-1}}{L\alpha}\left(\operatorname{prox}_{\varphi_1} +\cdots +\operatorname{prox}_{\varphi_L} \right) (\*z)\\
    &+\frac{\alpha^2(1-\alpha)^{L-2}}{L\alpha}\left( \sum_{p>q}\operatorname{prox}_{\varphi_p}\circ \operatorname{prox}_{\varphi_q}  \right)(\*z)+\cdots+\frac{\alpha^L}{L\alpha}\left(\prod_{p=L}^1 \operatorname{prox}_{\varphi_p} \right)(\*z)\\
    =& \left[ \frac{1}{L} \sum_{p=2}^L (-1)^p \binom{L}{p} \alpha^{p-1} \right]\*z+\frac{(1-\alpha)^{L-1}}{L}\left(\operatorname{prox}_{\varphi_1} +\cdots +\operatorname{prox}_{\varphi_L} \right) (\*z)\\
    &+\frac{\alpha(1-\alpha)^{L-2}}{L}\left( \sum_{p>q}\operatorname{prox}_{\varphi_p}\circ \operatorname{prox}_{\varphi_q}  \right)(\*z)+\cdots+\frac{\alpha^{L-1}}{L}\left(\prod_{p=L}^1 \operatorname{prox}_{\varphi_p} \right)(\*z),
    \end{aligned}
\end{equation}
Note that $\*z_0^{*}(\alpha)$ is also the fixed point of the above equation.
\par
Denote the right hand side of Eq.(\ref{eq:trasformfixed}) as $\*T_\alpha(\*z)$, note that $\*T_0(\*z)$ is also well-defined now.
Estimate the Lipschitz constant of $\*T_\alpha(\*z),\alpha \in [0,1)$ :
\[
\begin{aligned}
\operatorname{Lip}(\*T_\alpha)\leq & \left|  \frac{1}{L} \sum_{p=2}^L (-1)^p \binom{L}{p} \alpha^{p-1}  \right| +\left|  \frac{(1-\alpha)^{L-1}}{L} \right| \qty( \|\*W_1\|_2^2+\cdots+\|\*W_L\|_2^2)\\
&+\left| \frac{\alpha(1-\alpha)^{L-2}}{L} \right| \left|  \sum_{p>q}  \|\*W_p\|_2^2\|\*W_q\|_2^2  \right| + \cdots + \left| \frac{\alpha^{L-1}}{L}  \right| \left(\prod_{p=L}^1 \|\*W_p\|_2^2 \right).
\end{aligned}
\]
Each terms in the right hand side of the above inequality (except the second term) is a polynomial of $\alpha $ with non-zero order for $\alpha$, hence they tend to zero when $\alpha \rightarrow 0$. 
Note that $(1-\alpha)^L \to (1 - L\alpha)$, when $\alpha \to 0$.
Thus, the second term tend to $\frac{1}{L}  \qty( \|\*W_1\|_2^2+\cdots+\|\*W_L\|_2^2)$, which is less than $1 $ by assumption. Hence there is a $\kappa \in (0,1) $, when $\alpha \in [0,\kappa]$, $\operatorname{Lip}(\*T_\alpha)<\kappa$.

By using Lemma \ref{lem:contifixed}, the fixed point of $\*T_\alpha$ ( i.e. $\*z_0^*(\alpha)$ ) tends to the fixed point of $\*T_0$, i.e.
\[\*y^*=\frac{\operatorname{prox}_{\varphi_1}(\*y^*)+\cdots+\operatorname{prox}_{\varphi_L}(\*y^*)}{L}.
\]
Using first order optimality condition, $\left(\operatorname{prox}_{\varphi_1}(\*y^*),\cdots,\operatorname{prox}_{\varphi_L}(\*y^*),\*y^*\right)$ is the minimizer of the following strongly convex problem:
\[\sum_{l=1}^{L}\left( \varphi_l(\*x_l)+\frac{1}{2}\|\*x_l-\*y\|^2 \right).
\]
Finally, let $\*z_1^*(\alpha)$ be the fixed point of the composition equation:
\[
\*z=\left[\alpha\operatorname{prox}_{\varphi_1}+(1-\alpha)\${I} \right]\circ\left[\alpha\operatorname{prox}_{\varphi_L}+(1-\alpha)\${I} \right]\circ \cdots \circ \left[\alpha\operatorname{prox}_{\varphi_2}+(1-\alpha)\${I} \right](\*z).
\]
A similar argument can be applied to $\*z_1^*(\alpha)$ to show that, when $\alpha \rightarrow 0$, $\*z_1^*(\alpha)$ tends to the same $\*y^*$ defined above, and so do $\*z_2^*(\alpha),\cdots,\*z_{L-1}^*(\alpha) $.

\end{proof}

\section{Proofs for the Convergence of SAM}
\subsection{Auxiliary Lemmas}
Next lemma follows Lem. 2.5 in \citet{xu2002iterative}.
\begin{lemma}\label{lem:recurrlem}
If $a_1 \geq 0, 0<t_1<1, t_2>0, 0<r_1<1, r_2>0 $, $\{a_n\}$ is a sequence of non-negative numbers satisfying
\[ a_{k+1}=\left(1-\frac{r_1}{(k+1)^{t_1}}\right)a_k+  \frac{r_1}{(k+1)^{t_1}}\frac{r_2}{k^{t_2}}.    \]
Then $\lim_{n\to \infty} a_n=0$, and there exists $B=\order{r_1,r_2,t_1,t_2,a_1}$ such that $\|a_k\|\leq B$.
\end{lemma}

\begin{proof}
Denote $b_k=\frac{r_1}{(k+1)^{t_1}}, c_k=\frac{r_2}{k^{t_2}}$. Since $0<t_1<1, 0<r_1<1$, we have $0<b_k<1, \sum_{k=1}^{\infty} b_k=\infty$ and:
\[
\begin{aligned}
\prod_{k=1}^\infty (1-b_k)=&\exp\left( \sum_{k=1}^\infty  \ln (1-b_k)  \right) =\lim_{K\rightarrow \infty} \exp\left( \sum_{k=1}^K  \ln (1-b_k)  \right)\\
\leq& \limsup_{K\rightarrow \infty} \exp\left( \sum_{k=1}^K  -b_k  \right)=0.
\end{aligned}
\]
For any $\epsilon >0$, choose $N$ big enough such that $c_k\leq \epsilon, \forall k\geq N$, then for all $k>N$, by induction, we have,
\[\begin{aligned}
a_{k+1}&=(1-b_k)a_k+b_kc_k\\
&=(1-b_k)(1-b_{k-1})a_{k-1}+(1-b_k)b_{k-1}c_{k-1}+b_kc_k\\
&=\cdots\\
&=\prod_{j=N}^{k}(1-b_j)a_N+\sum_{i=N}^{k}\left( \prod_{j=i+1}^{k} (1-b_j) \right) b_ic_i\\
&\leq \prod_{j=N}^{k}(1-b_j)a_N+\sum_{i=N}^{k}\left( \prod_{j=i+1}^{k} (1-b_j) \right) b_i\epsilon\\
&=\prod_{j=N}^{k}(1-b_j)a_N+\left( 1-\prod_{j=N}^{k}(1-b_j) \right)\epsilon\\
&\leq \prod_{j=N}^{k}(1-b_j)a_N+\epsilon.
\end{aligned}\]
Hence $\limsup_{k\rightarrow \infty} a_k \leq \epsilon$, let $\epsilon \rightarrow 0^+$, we get $\lim_{n\to \infty} a_n=0$. Since every convergence sequence is bounded, there exists $B=\order{r_1,r_2,t_1,t_2,a_1}$ such that: $\|a_k\|\leq B$.
\end{proof}
Next lemma follows from Prop. 3 in \citet{sabach2017first}.
\begin{lemma}\label{lem:sc2contra}
$l(\*x)$ is a $L_z$-smooth convex function (i.e. $\nabla l(\*x)$ is $L_z$-Lipschitz), and $\gamma=\frac{1}{2L_z},0< \lambda \leq \frac{L_z}{2}$. Then the operator $\$S_{\lambda}(\*x)= \*x-\gamma \qty(\nabla l(\*x)+\lambda \*x )$ satisfies $\frac{1}{4}\|\*x-\*y\|\leq \| \$S_{\lambda}(\*x)-\$S_{\lambda}(\*y)\| \leq (1-\frac{\gamma \lambda}{2})\|\*x-\*y\| $.
\end{lemma}
\begin{proof}
Note that  $l(\*x) +\frac{\lambda}{2}\|\*x\|^2$ is $\lambda$-strongly convex and $(L_z+\lambda)$-smooth. 
By Prop. 3 in \citet{sabach2017first}, it follows that:
\[\norm{\$S_{\lambda}(\*x)-\$S_{\lambda}(\*y)}\leq \sqrt{1-\frac{2\gamma \lambda (L_z+\lambda)}{L_z+2\lambda}} \|\*x-\*y\|.  \]
Note that $\sqrt{1-\frac{2\gamma \lambda (L_W+\lambda)}{L_W+2\lambda}}<\sqrt{1-\gamma\lambda}<1-\frac{\gamma \lambda}{2}$, so the operator $\$S_{\lambda}(\*x)$ is $(1-\frac{\gamma \lambda}{2})$-contractive.

On the other hand, by Cauchy–Schwarz inequality, $\innerprod{ (\nabla l(\*x)+\lambda \*x)-(\nabla l(\*y)+\lambda \*y),\*x-\*y} \leq \frac{2\gamma}{3} \|(\nabla l(\*x)+\lambda \*x)-(\nabla l(\*y)+\lambda \*y)\|^2+\frac{3}{8\gamma}\|\*x-\*y\|^2$, and note that $\|\$S_{\lambda}(\*x)-\$S_{\lambda}(\*y)\|^2=\|\*x-\*y\|^2-2\gamma\innerprod{ (\nabla l(\*x)+\lambda \*x)-(\nabla l(\*y)+\lambda \*y),\*x-\*y}+\gamma^2\|(\nabla l(\*x)+\lambda \*x)-(\nabla l(\*y)+\lambda \*y)\|^2$, we have:
\[
\begin{aligned}
\|\$S_{\lambda}(\*x)-\$S_{\lambda}(\*y)\|^2&\geq \frac{1}{4}\|\*x-\*y\|^2-\frac{\gamma^2}{3}\|(\nabla l(\*x)+\lambda \*x)-(\nabla l(\*y)+\lambda \*y)\|^2\\
&\geq \frac{1}{4}\|\*x-\*y\|^2-\frac{\gamma^2}{3}(L_z+\lambda)^2\|\*x-\*y\|^2\\
&\geq \frac{1}{4}\|\*x-\*y\|^2-\frac{\gamma^2}{3}(\frac{3L_z}{2})^2\|\*x-\*y\|^2\\
&\geq  (\frac{1}{4}-\frac{3}{16})\|\*x-\*y\|^2\\
&= \frac{1}{16}\|\*x-\*y\|^2,
\end{aligned}
\]
which is equivalent to $\|\$S_{\lambda}(\*x)-\$S_{\lambda}(\*y)\|\geq \frac{1}{4}\|\*x-\*y\|$.
\end{proof}

\begin{lemma}\label{lem:FixofnonE}
If $\$T(\*z)$ is a nonexpansive map with $\operatorname{dom}(\$T)=\^R^n$, then the fixed point set of $\$T(\*z)$ is closed and convex. And for any $\*x,\*y$, $\innerprod{ (\*x-\$T(\*x))-(\*y-\$T(\*y)),\*x-\*y} \geq 0$.
\end{lemma}
\begin{proof}
If $\text{Fix}(\$T)$ is empty, then it is closed and convex.
If $\text{Fix}(\$T)$ is non-empty, for any $\*x,\*y \in \text{Fix}(\$T)$, let $\*z =\theta \*x+(1-\theta) \*y$, where $\theta \in (0,1)$. Since  $\$T$ is nonexpansive, we have:
\[
\left\{
\begin{aligned}
&\|\$T(\*z)-\*x\|=\|\$T(\*z)-\$T(\*x)\|\leq \|\*z-\*x\|=(1-\theta) \|\*x-\*y\|\\
& \|\$T(\*z)-\*y\|=\|\$T(\*z)-\$T(\*y)\|\leq \|\*z-\*y\|=\theta \|\*x-\*y\|.
\end{aligned}
\right.
 \]
So the triangle inequality,
\[  \|\*x-\*y\| \leq \|\$T(\*z)-\*x\| + \|\$T(\*z)-\*y\| \leq (1-\theta) \|\*x-\*y\|+\theta \|\*x-\*y\|=\|\*x-\*y\|, \]
holds with equality. So $\$T(\*z)$ lies on the line segment between $\*x,\*y$, and $ \|\$T(\*z)-\*y\|=\theta \|\*x-\*y\|, \|\$T(\*z)-\*x\|=(1-\theta) \|\*x-\*y\|$ hold, which means $\$T(\*z)=\theta \*x+(1-\theta) \*y=\*z$, i.e., $\*z \in \text{Fix}(\$T)$.
\par
The inequality $\innerprod{ (\*x-\$T(\*x))-(\*y-\$T(\*y)),\*x-\*y} \geq 0$ follows directly form Cauchy–Schwarz inequality.
\end{proof}

\subsection{Proof for Theorem \ref{thm:bilevelConvergence}}
\begin{proof}
Since $\lambda_k\leq \frac{L_z}{2}$, by Lemma \ref{lem:sc2contra}, operator $\*z \mapsto \*z-\gamma \qty( \nabla \$R_z(\*z)+\lambda_k \*z )$ is contractive, so is the operator $\*z \mapsto \beta_k\qty(\*z-\gamma ( \nabla \$R_z(\*z)+\lambda_k \*z ))+(1-\beta_k)\$T(\*z)$, hence the latter operator has a unique fixed point $\Bar{\*z}^k$.
\par
By assumption, $\|\Bar{\*z}^{k+1}\|\leq B_1^*$. Note that in our setting, $\nabla \$R_z(\cdot)$ and $\$T(\cdot)$ are continuous, so we have $\|\$T(\Bar{\*z}^{k+1})\|+  \|\Bar{\*z}^{k+1}-\gamma  \nabla \$R_z(\Bar{\*z}^{k+1})\| \leq B_2^*$. First we estimate the difference of two successive fixed point:
\[
\begin{aligned}
&\|\Bar{\*z}^k-\Bar{\*z}^{k+1}\|\\
=&  \left\|\beta_k\qty(\Bar{\*z}^k-\gamma ( \nabla \$R_z(\Bar{\*z}^k)+\lambda_k \Bar{\*z}^k ))+(1-\beta_k)\$T(\Bar{\*z}^k) \right. \\  
&\left.-\beta_{k+1}\qty(\Bar{\*z}^{k+1}-\gamma ( \nabla \$R_z(\Bar{\*z}^{k+1})+\lambda_{k+1} \Bar{\*z}^{k+1} ))-(1-\beta_{k+1})\$T(\Bar{\*z}^{k+1})  \right\| \\
\leq &  \left\|\beta_k\qty(\Bar{\*z}^k-\gamma ( \nabla \$R_z(\Bar{\*z}^k)+\lambda_k \Bar{\*z}^k ))    -\beta_{k}\qty(\Bar{\*z}^{k+1}-\gamma ( \nabla \$R_z(\Bar{\*z}^{k+1})+\lambda_{k} \Bar{\*z}^{k+1} ))  \right\| \\ 
+&  \left\|\beta_{k}\qty(\Bar{\*z}^{k+1}-\gamma ( \nabla \$R_z(\Bar{\*z}^{k+1})+\lambda_{k} \Bar{\*z}^{k+1} ))   -\beta_{k+1}\qty(\Bar{\*z}^{k+1}-\gamma ( \nabla \$R_z(\Bar{\*z}^{k+1})+\lambda_{k+1} \Bar{\*z}^{k+1} ))  \right\|  \\
+&  \left\|  (1-\beta_k)(\$T(\Bar{\*z}^k)-\$T(\Bar{\*z}^{k+1}) ) \right\| +\left\|  (\beta_{k+1}-\beta_k)\$T(\Bar{\*z}^{k+1} ) \right\|\\
\leq &  \beta_k (1-\frac{\gamma \lambda_k}{2})\norm{\Bar{\*z}^k-\Bar{\*z}^{k+1}} +(1-\beta_k)\norm{\Bar{\*z}^k-\Bar{\*z}^{k+1}} +   B_2^*\qty|\beta_k -\beta_{k+1}|+ B_1^*| \beta_k\lambda_k -\beta_{k+1}\lambda_{k+1} | \\
=& (1- \frac{\gamma \beta_k\lambda_k}{2})\|\Bar{\*z}^k-\Bar{\*z}^{k+1}\|+   B_2^*|\beta_k -\beta_{k+1}|+ B_1^*| \beta_k\lambda_k -\beta_{k+1}\lambda_{k+1} |.
\end{aligned}
\]
We get:
\[
\begin{aligned}
&\frac{\|\Bar{\*z}^k-\Bar{\*z}^{k+1}\|}{\frac{\gamma}{2} \beta_{k+1}\lambda_{k+1}}\\
\leq& \frac{4B_2^*}{\gamma^2}\frac{|\beta_k -\beta_{k+1}|}{(\beta_k\lambda_k)(\beta_{k+1}\lambda_{k+1})}+
\frac{4B_1^*}{\gamma^2}\frac{| \beta_k\lambda_k -\beta_{k+1}\lambda_{k+1} |}{(\beta_k\lambda_k)(\beta_{k+1}\lambda_{k+1})}\\
=&  \frac{4B_2^*}{\gamma^2 \eta^3}\frac{|k^{-\rho}-(k+1)^{-\rho}|}{k^{-\rho-c}(k+1)^{-\rho-c}}+
\frac{4B_1^*}{\gamma^2 \eta^2}\frac{| k^{-\rho-c}-(k+1)^{-\rho-c} |}{k^{-\rho-c}(k+1)^{-\rho-c}} \\
\leq &  \frac{4B_2^*\rho}{\gamma^2 \eta^3}k^{-\rho-1}k^{\rho+c}(k+1)^{\rho+c}+
\frac{4B_1^*(\rho+c)}{\gamma^2 \eta^2}k^{-\rho-c-1}k^{\rho+c}(k+1)^{\rho+c}   \\
\leq &  \frac{4B_2^*\rho 2^{\rho+c}}{\gamma^2 \eta^3}k^{\rho+2c-1}+
\frac{4B_1^*(\rho+c)2^{\rho+c}}{\gamma^2 \eta^2}k^{\rho+c-1} \\
\leq & \left[ \frac{4B_2^*\rho 2^{\rho+c}}{\gamma^2 \eta^3}+\frac{4B_1^*(\rho+c)2^{\rho+c}}{\gamma^2 \eta^2} \right]k^{\rho+2c-1}\\
:= & B_3^*k^{\rho+2c-1}.
\end{aligned}
\]
By the iteration equation:
\[
\*z^{k+1}=\beta_{k+1}\qty(\*z^k-\gamma ( \nabla \$R_z(\*z^k)+\lambda_{k+1} \*z^k ))+(1-\beta_{k+1})\$T(\*z^k),
\]
we have:
\[
\begin{aligned}
&\|\*z^{k+1}-\Bar{\*z}^{k+1}\|\\
= & \left\| \qty[\beta_{k+1}(\*z^k-\gamma ( \nabla \$R_z(\*z^k)+\lambda_{k+1} \*z^k ))+(1-\beta_{k+1})\$T(\*z^k)]\right.\\
&-\left. \left[\beta_{k+1}(\Bar{\*z}^{k+1}-\gamma ( \nabla \$R_z(\Bar{\*z}^{k+1})+\lambda_{k+1} \Bar{\*z}^{k+1} ))+(1-\beta_{k+1})\$T(\Bar{\*z}^{k+1})\right]  \right\|\\
\leq &  \beta_{k+1} \left\| (\*z^k-\gamma ( \nabla \$R_z(\*z^k)+\lambda_{k+1} \*z^k ))- (\Bar{\*z}^{k+1}-\gamma ( \nabla \$R_z(\Bar{\*z}^{k+1})+\lambda_{k+1} \Bar{\*z}^{k+1} ))  \right\| \\
&+(1-\beta_{k+1})\| \$T(\*z^k)-\$T(\Bar{\*z}^{k+1})  \|  \\
\leq & \beta_{k+1}(1-\frac{\gamma \lambda_{k+1}}{2})\|\*z^k-\Bar{\*z}^{k+1}\|+(1-\beta_{k+1})\|\*z^k-\Bar{\*z}^{k+1}\|\\
\leq &(1-\frac{\gamma}{2}\beta_{k+1}\lambda_{k+1})\qty(\|\*z^k-\Bar{\*z}^k\|+\|\Bar{\*z}^k-\Bar{\*z}^{k+1}\|)\\
\leq & (1-\frac{\gamma}{2}\beta_{k+1}\lambda_{k+1})\|\*z^k-\Bar{\*z}^k\|+\frac{\gamma}{2}\beta_{k+1}\lambda_{k+1}B_3^*k^{\rho+2c-1}.
\end{aligned}
\]
Since $\eta\leq \sqrt{2L_z}$, so $\gamma \eta^2\leq 1$. Let $r_1=\frac{\gamma \eta^2}{2}, r_2=B_3^*, t_1=\rho+c \in (0,1), t_2=1-\rho -2c >0$, and $a_k =\|\*z^k-\Bar{\*z}^k\|$, by Lemma \ref{lem:recurrlem} we have: $\|\*z^k-\Bar{\*z}^k\|\rightarrow 0$ as $k\to \infty$ and $\|\*z^k-\Bar{\*z}^k\|\leq B_4^*=\order{\*z^1, \rho,c,B_3^*}.$

Next, we denote the only minimizer of convex function $\$R_z(\*z)$ on compact convex set $\operatorname{Fix}(\$T)$ as $\Bar{\*z}$, which is bounded by $B_1^*$.
Note that the convexity follows from Lemma \ref{lem:FixofnonE}. 
We now prove the final results by contradiction.
Since the sequence $\{\Bar{\*z}^k\}$ is bounded by $B_1^*$, if it does not converge to $\Bar{\*z}$, there exists $\delta>0$ and a sub-convergent  $\{\Bar{\*z}^{k_j}\}$ such that:
\[ \|\Bar{\*z}^{k_j}-\Bar{\*z}\|\geq \delta, \quad\forall j \in \^N^+, \]
and
\[\Bar{\*z}^{k_j}\rightarrow  \Bar{\*z}',\quad  \|\Bar{\*z}'-\Bar{\*z}\|\geq \delta.\]
By the fixed point equation and Lemma \ref{lem:FixofnonE}, we have
\[
\begin{aligned}
& \Bar{\*z}^{k_j}= \beta_{k_j}\qty(\Bar{\*z}^{k_j}-\gamma ( \nabla \$R_z(\Bar{\*z}^{k_j})+\lambda_{k_j} \Bar{\*z}^{k_j} ))+(1-\beta_{k_j})\$T(\Bar{\*z}^{k_j})  \\
\Rightarrow &   \gamma ( \nabla \$R_z(\Bar{\*z}^{k_j})+\lambda_{k_j} \Bar{\*z}^{k_j} )=\frac{1-\beta_{k_j}}{\beta_{k_j}}(\$T(\Bar{\*z}^{k_j})-\Bar{\*z}^{k_j})  \\
\Rightarrow &   \innerprod{  \gamma ( \nabla \$R_z(\Bar{\*z}^{k_j})+\lambda_{k_j} \Bar{\*z}^{k_j} ) ,\*z-\Bar{\*z}^{k_j}} = \frac{1-\beta_{k_j}}{\beta_{k_j}}\langle \$T(\Bar{\*z}^{k_j})-\Bar{\*z}^{k_j}, \*z-\Bar{\*z}^{k_j} \rangle\\
&\, =  \frac{1-\beta_{k_j}}{\beta_{k_j}}\innerprod{ \$T(\Bar{\*z}^{k_j})-\Bar{\*z}^{k_j}-\qty(\$T(\*z)-\*z), \*z-\Bar{\*z}^{k_j} }   \geq 0   ,\quad \forall \*z\in \operatorname{Fix}(\$T)\\
\Rightarrow &  \langle   \nabla \$R_z(\Bar{\*z}^{k_j})+\lambda_{k_j} \Bar{\*z}^{k_j}  ,\*z-\Bar{\*z}^{k_j}\rangle  \geq 0  ,\quad \forall \*z\in \operatorname{Fix}(\$T).    
\end{aligned}
\]
Let $j\rightarrow \infty$, we get:
\[
\innerprod{   \nabla \$R_z(\Bar{\*z}') ,\*z-\Bar{\*z}'} \geq 0,\quad  \forall \*z\in \operatorname{Fix}(\$T),
\]
which is equivalent to $ \$R_z(\Bar{\*z}')\leq \$R_z(\*z), \forall \*z\in \operatorname{Fix}(\$T)$. 
Namely $\Bar{\*z}'=\Bar{\*z}$, which is a contradiction.
\par
Thus, the sequence $\{\Bar{\*z}^k\}$ converge to $\Bar{\*z}$. Combining the result $\|\*z^k-\Bar{\*z}^k\|\rightarrow 0$, we have:
\[  \*z^k\rightarrow \Bar{\*z}. \]
We finish the proof.
\end{proof}
\section{Proofs for Linear Convergence Training}
In the proof, we will consider the whole data set, i.e., $\*Z^{K} \in \^R^{m\times N}$, where $N$ is the training data size. 
We denote by $\*z^k \coloneqq \operatorname{vec}(\*Z^k)$, where $\operatorname{vec}(\*Z) \in \^R^{m N}$ is the vectorization of the matrix $\*Z \in \^R^{m\times N}$.
$\*Z^k$ is the $k$-th iterates of the sequence generated by Eq.(\ref{eq:InclusinApprox}) applied on the data matrix.
Then, we denote:
\[
\*Z^{(k,l)}\coloneqq f_l\circ\cdots\circ f_1(\*Z^{k},\*Z,\bm{\theta}), \qtext{and} \*z^{(k,l)} \coloneqq \operatorname{vec}\qty(\*Z^{(k,l)}).
\]
$\*z^{(k,l)}$ and $\*z^{k}$ all depend on the parameter $\bm{\theta}$.
However, for the sake of brevity, we omit the mark $\bm{\theta}$ when the meaning of the symbol is clear.
We let $\*I_n$ be the $n \times n$ identity matrix.
For the output of the network, we let $\*y \coloneqq \operatorname{vec}(\*W_{L+1}\*Z^K)$.
\par
Note that we have
\[
\$S_{\lambda_k}(\*Z,\*W_{L+1}) =  \*Z-\gamma \qty( \nabla \$R_z(\*Z)+\lambda_k \*Z),\qtext{and} \gamma < \frac{1}{2L_z+1}.
\]
And we let $\*s^k \coloneqq \operatorname{vec}\qty(\$S(\*Z^k,\*W_{L+1}))$, then:
\[
\*z^{k+1} = \beta_{k+1} \operatorname{vec}\qty(\$S_{\lambda_k}(\*Z^k,\*W_{L+1})) + (1-\beta_{k+1}) \$T(\*z^k,\*x,\bm{\theta}) = \beta_{k+1} \*s^k + (1-\beta_{k+1})\*z^{(k,L)}.
\]
We let:
\[
\$V_\gamma(\cdot)\coloneqq \qty(1-\gamma\lambda_k)\$I (\cdot) - \gamma \nabla \$R_z(\cdot),\qtext{and denote by}
\*G^{(k,l)} \coloneqq \sigma\qty(\*W_l\*Z^{(k,l-1)} + \*U_l\*X+ \*b_l).
\]
Note that by our setting on $\gamma$ and $\lambda_k$, we can easily conclude that $\norm{\$V_\gamma}_2 < 1$. Moreover, due to Lemma \ref{lem:sc2contra}, we also get the lower bound: $\norm{\$V_\gamma}_2 \geq \frac{1}{4}$.
\par
We also denote:
\[
\*D^{(k,l)} \coloneqq \mqty[\dmat{\widetilde{\*D}_1,\ddots,\widetilde{\*D}_N}],  \quad
\widetilde{\*D}_j \coloneqq \mqty[\dmat{\*d_{1j},\ddots,\*d_{mj}}],
\]
where 
\[
d_{ij} = \sigma'\qty(\*W_l\*z^k_j + \*U_l\*x+ \*b_l)_i,
\]
is the $(i,j)$-th entry of the derivative matrix \[\sigma'\qty(\*W_l\*Z^{(k,l)} + \*U_l\*X+ \*b_l) \in \^R^{m\times N}.\]
We let:
\[
   \*A(k,l_2,l_1)\coloneqq \prod_{l=l_1}^{l_2} \qty(\alpha \qty(\*I_N \otimes \*W_l^\top) \*D^{(k,l)} \qty(\*I_N \otimes \*W_l) + \qty(1-\alpha)\*I_{mN}),
\]
where $\otimes$ is the Kronecker product.
\par
For Theorem \ref{thm:global}, we consider the square loss, i.e.,
\[\ell(\widetilde{\bm{\theta}}) = \ell(\*y, \*y^0) = \frac{1}{2}\norm{\*y-\*y_0}^2.\]
\subsection{Auxiliary Lemmas}
We first offer several auxiliary lemmas.
\begin{lemma}\label{lem:VecGradient}
The following results hold:
\begin{equation}\label{eq:VecGrad}
\resizebox{\hsize}{!}{$
\left\{
\begin{aligned}
\pdv{\*z^{k+1}}{\operatorname{vec}(\*W_l)} = & \sum_{\widetilde{k}=1}^k \qty(
\qty(\prod_{q=\widetilde{k}+1}^k \qty(\beta_{q+1}\qty(\*I_{N} \otimes \$V_\gamma)  +\qty(1-\beta_{q+1}) \*A(q,1))) \*A(\widetilde{k},l+1)
\pdv{\operatorname{vec}\qty(f_l(\*Z^{(\widetilde{k},l-1)},\*X))}{\operatorname{vec}\qty(\*W_l)}),\\
\pdv{\*z^{k+1}}{\operatorname{vec}(\*U_l)} = &
\sum_{\widetilde{k}=1}^k \qty(
\qty(\prod_{q=\widetilde{k}+1}^k \qty(\beta_{q+1}\qty(\*I_{N} \otimes \$V_\gamma)  +\qty(1-\beta_{q+1}) \*A(q,1))) \*A(\widetilde{k},l+1)
\pdv{\operatorname{vec}\qty(f_l(\*Z^{(\widetilde{k},l-1)},\*X))}{\operatorname{vec}\qty(\*U_l)}),\\
\pdv{\*z^{k+1}}{\operatorname{vec}(\*b_l)} = &
\sum_{\widetilde{k}=1}^k \qty(
\qty(\prod_{q=\widetilde{k}+1}^k \qty(\beta_{q+1}\qty(\*I_{N} \otimes \$V_\gamma)  +\qty(1-\beta_{q+1}) \*A(q,1))) \*A(\widetilde{k},l+1)
\pdv{\operatorname{vec}\qty(f_l(\*Z^{(\widetilde{k},l-1)},\*X))}{\operatorname{vec}\qty(\*b_l)}),
\end{aligned}
\right.$} 
\end{equation}
where
\begin{equation}\label{eq:singleVecGrad}
\resizebox{\hsize}{!}{$
 \left\{
\begin{aligned}
\pdv{\operatorname{vec}\qty(f_l(\*Z^{(k,l)},\*X))}{\operatorname{vec}\qty(\*W_l)} = &\alpha\qty(\qty(\*G^{(k,l)})^\top\otimes\*I_m )\*K^{(n_l,m)} + \alpha \qty(\*I_N \otimes \*W_l^\top)\*D^{(k,l)}\qty(\qty(\*Z^{(k,l)})^\top \otimes \*I_{n_l}),\\
\pdv{\operatorname{vec}\qty(f_l(\*Z^{(k,l)},\*X))}{\operatorname{vec}\qty(\*U_l)} =  & \alpha\qty(\*I_N \otimes \*W_l^\top)\*D^{(k,l)}\qty(\*X^\top \otimes \*I_{n_l}),\quad
\pdv{\operatorname{vec}\qty(f_l(\*Z,\*X))}{\operatorname{vec}\qty(\*b_l)} =   \alpha\qty(\*I_N \otimes \*W_l^\top)\*D^{(k,l)}\*1_N,
\end{aligned}
\right.$}
\end{equation}
here $\*1_N$ is a $N$-dimensional all-one vector.
\end{lemma}
\begin{proof}
Note that $ f_l(\*Z^{(k,l)},\*X) = \alpha \*W_l^\top \sigma\qty(\*W_l\*Z^{(k,l)} + \*U_l\*X+ \*b_l) + (1-\alpha)\*Z^{(k,l)}$. Hence, we have:
\[
\pdv{\operatorname{vec}\qty(f_l(\*Z^{(k,l)},\*X))}{\*z^{(k,l)}} = \alpha \qty(\*I_N \otimes \*W_l^\top) \*D^{(k,l)} \qty(\*I_N \otimes \*W_l) + \qty(1-\alpha)\*I_{mN},
\]
where $\otimes$ is the Kronecker product.
Then, we can get:
\[
\pdv{\$T(\*z^k,\*x,\bm{\theta})}{\*z^k} = \prod_{l=1}^L \qty(\alpha \qty(\*I_N \otimes \*W_l^\top) \*D^{(k,l)} \qty(\*I_N \otimes \*W_l) + \qty(1-\alpha)\*I_{mN}),
\]
and
\begin{equation}\label{eq:pdvZ}
\begin{aligned}
\pdv{\*z^{k+1}}{\*z^{k}} = &\beta_{k+1} \pdv{\$S(\*z^k,\*W_{L+1})}{\*z^k}+\qty(1-\beta_{k+1}) \pdv{\$T(\*z^k,\*x,\bm{\theta})}{\*z^k} \\
= & \beta_{k+1}\qty(\*I_{N} \otimes \qty(\*I_{d_y} - \gamma\*W_{L+1}^\top\*W_{L+1})) +\qty(1-\beta_{k+1}) \pdv{\$T(\*z^k,\*x,\bm{\theta})}{\*z^k}\\
= & \beta_{k+1}\qty(\*I_{N} \otimes \$V_\gamma) +\qty(1-\beta_{k+1})
\prod_{l=1}^L \qty(\alpha \qty(\*I_N \otimes \*W_l^\top) \*D^k_l \qty(\*I_N \otimes \*W_l) + \qty(1-\alpha)\*I_{mN})\\
= &  \beta_{k+1}\qty(\*I_{N} \otimes \$V_\gamma) +\qty(1-\beta_{k+1}) \*A(k,L,1).
\end{aligned}
\end{equation}
On the other hand, we have:
\[
\begin{aligned}
&\pdv{\operatorname{vec}\qty(f_l(\*Z^{(k,l)},\*X))}{\operatorname{vec}\qty(\*W_l)}\\  
=& \alpha
\qty(\sigma\qty(\*W_l\*Z^{(k,l)} + \*U_l\*X+  \*b_l)^\top\otimes\*I_m )\*K^{(n_l,m)} + \alpha\qty(\*I_N \otimes \*W_l^\top)\*D^{(k,l)}\qty(\qty(\*Z^{(k,l)})^\top \otimes \*I_{n_l})\\
= & 
\alpha\qty(\qty(\*G^{(k,l)})^\top\otimes\*I_m )\*K^{(n_l,m)} + \alpha \qty(\*I_N \otimes \*W_l^\top)\*D^{(k,l)}\qty(\qty(\*Z^{(k,l)})^\top \otimes \*I_{n_l}),
\end{aligned}
\]
where $\*K^{(n_l,m)}\in \^R^{n_lm\times n_lm}$ is the commutation matrix such that $\*K^{(n_l,m)} \operatorname{vec}(\*W) = \operatorname{vec}(\*W^\top)$.
And
\[
\left\{
\begin{aligned}
&\pdv{\operatorname{vec}\qty(f_l(\*Z^{(k,l)},\*X))}{\operatorname{vec}\qty(\*U_l)} =   \alpha\qty(\*I_N \otimes \*W_l^\top)\*D^{(k,l)}\qty(\*X^\top \otimes \*I_{n_l}),\\
&\pdv{\operatorname{vec}\qty(f_l(\*Z,\*X))}{\operatorname{vec}\qty(\*b_l)} =   \alpha\qty(\*I_N \otimes \*W_l^\top)\*D^{(k,l)} \*1_N.
\end{aligned}
\right.
\]
Thus, we get:
\[
\begin{aligned}
&\pdv{\*z^{k+1}}{\operatorname{vec}(\*W_l)} = 
\pdv{\*z^{k+1}}{\*z^k} \pdv{\*z^k}{\operatorname{vec}(\*W_l)} + \pdv{\*z^{k+1}}{\operatorname{vec}(\*W_l)}\\
=&\qty(\beta_{k+1}\qty(\*I_{N} \otimes \$V_\gamma)  +\qty(1-\beta_{k+1}) \*A(k,1))\pdv{\*z^k}{\operatorname{vec}(\*W_l)}+
\*A(k,l+1)\pdv{\operatorname{vec}\qty(f_l(\*Z^{(k,l-1)},\*X))}{\operatorname{vec}\qty(\*W_l)} \\
=& \sum_{\widetilde{k}=1}^k \qty(
\qty(\prod_{q=\widetilde{k}+1}^k \qty(\beta_{q+1}\qty(\*I_{N} \otimes \$V_\gamma)  +\qty(1-\beta_{q+1}) \*A(q,1))) \*A(\widetilde{k},l+1)
\pdv{\operatorname{vec}\qty(f_l(\*Z^{(\widetilde{k},l-1)},\*X))}{\operatorname{vec}\qty(\*W_l)}).
\end{aligned}
\]
Similarly, we can also obtain:
\[
\pdv{\*z^{k+1}}{\operatorname{vec}(\*U_l)} =
\sum_{\widetilde{k}=1}^k \qty(
\*R_{\widetilde{k}} \*A(\widetilde{k},l+1)
\pdv{\operatorname{vec}\qty(f_l(\*Z^{(\widetilde{k},l-1)},\*X))}{\operatorname{vec}\qty(\*U_l)}),
\]
and
\[
\pdv{\*z^{k+1}}{\operatorname{vec}(\*b_l)} =
\sum_{\widetilde{k}=1}^k \qty( \*R_{\widetilde{k}}
 \*A(\widetilde{k},l+1)
\pdv{\operatorname{vec}\qty(f_l(\*Z^{(\widetilde{k},l-1)},\*X))}{\operatorname{vec}\qty(\*b_l)}),
\]
where
\[
\*R_{\widetilde{k}} = \qty(\prod_{q=\widetilde{k}+1}^k \qty(\beta_{q+1}\qty(\*I_{N} \otimes \$V_\gamma)  +\qty(1-\beta_{q+1}) \*A(q,1))).
\]
\end{proof}
Similar to the proof in the previous lemma, we can easily obtain
\begin{equation}\label{eq:VecGradkl}
\left\{
\begin{aligned}
&\pdv{\*z^{(k,l')}}{\operatorname{vec}(\*W_l)} = \*1_{\qty{l'>l}}\*A(k,l',l+1)\pdv{\operatorname{vec}\qty(f_l(\*Z^{({k},l-1)},\*X))}{\operatorname{vec}\qty(\*W_l)} + \*A(k,l',1)\pdv{\*z^{k-1}}{\operatorname{vec}(\*W_l)},\\
&\pdv{\*z^{(k,l')}}{\operatorname{vec}(\*U_l)} = \*1_{\qty{l'>l}}\*A(k,l',l+1)\pdv{\operatorname{vec}\qty(f_l(\*Z^{({k},l-1)},\*X))}{\operatorname{vec}\qty(\*U_l)} + \*A(k,l',1)\pdv{\*z^{k-1}}{\operatorname{vec}(\*U_l)},\\
&\pdv{\*z^{(k,l')}}{\operatorname{vec}(\*b_l)} = \*1_{\qty{l'>l}}\*A(k,l',l+1)\pdv{\operatorname{vec}\qty(f_l(\*Z^{({k},l-1)},\*X))}{\operatorname{vec}\qty(\*b_l)} + \*A(k,l',1)\pdv{\*z^{k-1}}{\operatorname{vec}(\*b_l)},
\end{aligned}
\right.
\end{equation}
where $\pdv{\*z^{(k-1)}}{\operatorname{vec}(\cdot)}$ and  $\pdv{\operatorname{vec}\qty(f_l(\*Z^{({k},l-1)},\*X))}{\operatorname{vec}\qty(\cdot)}$ are give in Eq.(\ref{eq:VecGrad}) and Eq.(\ref{eq:singleVecGrad}) in the previous lemma, and $\*1_{\qty{l'>l}}$ is the indicator function.
Before providing the lemmas, we present two assumptions commonly used in the following lemmas.
\begin{assumption}[Compact Set of Parameters]\label{asm:compactSet}
Given the learnable parameters
\[
\bm{\theta}=\qty{\qty(\*W_l,\*U_l,\*b_l)}_{l=1}^L, \quad \forall l \in [1,L],
\]
we assume $\norm{\*W_l}_2\leq 1$ and $\max\qty{\norm{\*U_l}_2, \norm{\*b_l}_2 }\leq B_{Ub}$. 
Moreover, we let $\norm{\*X}_F\leq B_x$ and $\norm{\*W_{L+1}}_2\leq B_L$.
\end{assumption}
\begin{assumption}[Existence and boundedness]\label{asm:exist&bound}
Let $\$T_{\beta,\lambda}(\*z,\*X,\bm{\theta})=\beta\qty(\*z-\gamma \qty( \nabla \$R_z(\*z)+\lambda \*z ))+(1-\beta)\$T(\*z,\*X,\bm{\theta})$. For any learnable parameters $\bm{\theta}$ and $\*X$ satisfy Assumption \ref{asm:compactSet}, and $\beta \in [0,\frac{1}{2}], \lambda \in [0,\frac{L_z}{2}]$, we assume the fixed point set $\operatorname{Fix}(\$T(\cdot,\*X,\bm{\theta})) \subset \^R^{mN}$ is non-empty and uniformly bounded, i.e., for any  $\bm{\theta}$ and $\*X$ satisfy Assumption \ref{asm:compactSet} and $\beta \in [0,\frac{1}{2}], \lambda \in [0,\frac{L_z}{2}]$, $\forall \*z \in \operatorname{Fix}(\$T_{\beta,\lambda}(\cdot,\*X,\bm{\theta})),~\norm{\*z}_2 \leq B_1^*$. 
Without loss of generality, we also let $\norm{\*z^1} \leq B_1^*$. 
\end{assumption}
We now show the uniform boundedness of $\*z^k$ and $\*z^{(k,l)}$. 
\begin{lemma}
If Assumption \ref{asm:compactSet} and Assumption \ref{asm:exist&bound} hold, and $\alpha\leq \frac{\ln 2}{2L}$, then:
\[
\norm{\*z^k} \leq B^* \quad\text{and}\quad \norm{\*z^{(k,l)}} \leq 3B^*,
\]
here $B^*=\order{B_x,B_{Ub},B_1^*}$, and $B_1^*$ is the uniform bound for any $\*z \in \operatorname{Fix}(\$T_{\beta,\lambda}(\cdot,\*X,\bm{\theta}))$, as shown in Assumption \ref{asm:exist&bound}.
\end{lemma}
\begin{proof}
We first give the boundedness for $\*z^k$. Following the proof of Theorem \ref{thm:bilevelConvergence}, we have $\|\*z^k-\Bar{\*z}^k\|\leq B_4^*=\order{\*z^1, \rho,c,B_3^*}$. Note that $B_3^*=\order{B_2^*,B_1^*,\rho,c,\gamma,\eta}, B_2^*=\order{B_x,B_{Ub},B_1^*}$, $\rho,c,\gamma,\eta$ are constants, and $\|\*z^1\|\leq B_1^*$ by Assumption \ref{asm:exist&bound}, so $B_4^*=\order{B_x,B_{Ub},B_1^*}$.

And by the definition of $\Bar{\*z}^k$, it is the fixed point of equation $\*z = \beta_k\qty(\*z-\gamma ( \nabla \$R_z(\*z)+\lambda_k \*z ))+(1-\beta_k)\$T(\*z,\*X,\bm{\theta})$, therefore $\|\Bar{\*z}^k\|\leq B_1^*$.
\par
Thus, we have
\[
\norm{\*z^k} \leq \norm{\*z^k - \Bar{\*z}^k} + \norm{\Bar{\*z}^k} \leq B_4^* + B_1^* = \order{B_x,B_{Ub},B_1^*}.
\]
On the other hand, we can suppose that $\|f_l\circ\cdots\circ f_1 (\*0)\|\leq B_5^* =\order{B_x,B_{Ub}}, \forall l \in [1,L]$, since they are continuously depended on $\*W,\*U,\*b\, \text{and}\,\alpha$. Let $B^*=\max\{ B_4^* + B_1^*,B_5^*\}$.
\par
Note that $\*z^{(k,l)}=f_l\circ\cdots\circ f_1(\*z^k)$, and each $f_k(\*z)=\*z+\alpha\left( \*W_k^\top \sigma (\*W_k \*z+\*U_k\*x+\*b)-\*z \right)$ is $(1+2\alpha)$-Lipschitz. Therefore 
\[
\begin{aligned}
 \norm{\*z^{(k,l)}} &\leq \|f_l\circ\cdots\circ f_1(\*z^k)-f_l\circ\cdots\circ f_1(\*0)  \|+ \| f_l\circ\cdots\circ f_1(\*0) \|\\
 &\leq (1+2\alpha)^l\|\*z^k-\*0\|+B^* \leq \exp{2\alpha L} B^*+B^* \leq 2 B^*+B^*= 3B^*.
\end{aligned}
\]
\end{proof}
For $\zeta \in (0,1)$, we say an operator $f$ is \emph{$\zeta$-averaged} if $f = (1-\zeta)\$I + \zeta \$S$ for some nonexpansive operator $\$S$.
\begin{lemma}
If Assumption \ref{asm:sigma} holds, given any learnable parameters satisfy Assumption \ref{asm:compactSet}, the function $f \coloneqq f_l\circ \cdots \circ f_1(\cdot,\*x)$ is averaged for all $l \in [1,L]$.
\end{lemma}
\begin{proof}
Note that:
\[
\pdv{\operatorname{vec}\qty(f_l(\*Z,\*X))}{\*z} = \alpha \qty(\*I_N \otimes \*W_l^\top) \*D \qty(\*I_N \otimes \*W_l) + \qty(1-\alpha)\*I_{mN}.
\]
Hence, we have:
\[
\norm{\pdv{\operatorname{vec}\qty(f_l(\*Z,\*X))}{\*z}} \leq \alpha \norm{\*W_l}^2_2 \norm{\*D}_2 + (1-\alpha) \leq 1,
\]
where the last inequality comes from Assumption \ref{asm:sigma} and Assumption \ref{asm:compactSet}. 
Hence the function $f_l(\cdot,\*x)$ is averaged for any parameters satisfy Assumption \ref{asm:compactSet}.
Moreover, by Proposition 4.46 in \cite{bauschke2011convex}, the composition of $l$ operators which are averaged with respect to the same norm is also averaged,
i.e., $f_l\circ \cdots \circ f_1(\cdot,\*x)$ is averaged when each $f_l(\cdot)$ is averaged.
\end{proof}
We now provides the bounds for $\*A_{k,l}$ and $\pdv{\operatorname{vec}\qty(f_l(\*Z^{(k,l)},\*X))}{\operatorname{vec}\qty(\cdot)}$.
\begin{lemma}
If Assumption \ref{asm:compactSet} and Assumption \ref{asm:sigma} hold, then:
\begin{equation}\label{eq:singlepdvBound}
 \left\{
\begin{aligned}
& \norm{\*A(k,l)}_2 \leq 1,\quad \norm{\prod_{q=\widetilde{k}+1}^k \qty(\beta_{q+1}\qty(\*I_{N} \otimes \$V_\gamma)  +\qty(1-\beta_{q+1}) \*A(q,1))} \leq 1,\\
&\norm{\pdv{\operatorname{vec}\qty(f_l(\*Z^{(k,l)},\*X))}{\operatorname{vec}\qty(\*W_l)}} \leq \alpha\qty(\sqrt{mn_l}\sigma(0) + 6B^* + 2B_{Ub}B_x),\\
&\norm{\pdv{\operatorname{vec}\qty(f_l(\*Z^{(k,l)},\*X))}{\operatorname{vec}\qty(\*U_l)}}\leq \alpha B_x, \quad
\norm{\pdv{\operatorname{vec}\qty(f_l(\*Z^{(k,l)},\*X))}{\operatorname{vec}\qty(\*b_l)}} \leq \alpha\sqrt{N} .
\end{aligned}
\right.   
\end{equation}
\end{lemma}
\begin{proof}
First of all, we get:
\[
\begin{aligned}
\norm{\*A(k,\widetilde{l})}_2 \leq & \prod_{l=\widetilde{l}}^L \norm{\qty(\alpha \qty(\*I_N \otimes \*W_l^\top) \*D^{(k,l)} \qty(\*I_N \otimes \*W_l) + \qty(1-\alpha)\*I_{mN})}_2\\
\leq & \prod_{l=\widetilde{l}}^L \qty(\alpha \norm{\*D^{(k,l)}}_2 \norm{\*W_l}_2^2 + (1-\alpha)) \leq 1, 
\end{aligned}
\]
where the last inenquality comes form Assumption \ref{asm:compactSet} and Assumption \ref{asm:sigma}.
Note that $\gamma <\frac{1}{2L_z+1}$, hence we have $\norm{\$V_{\gamma }}_2\leq 1$. Thus, we have:
\[
\norm{\prod_{q=\widetilde{k}+1}^k \qty(\beta_{q+1}\qty(\*I_{N} \otimes \$V_\gamma)  +\qty(1-\beta_{q+1}) \*A(q,1))}_2 \leq 
\prod_{q=\widetilde{k}+1}^k \qty(\beta_{q+1} \norm{\$V_\gamma}_2+\qty(1-\beta_{q+1})) \leq 1,
\]
where the first inequality we utilize $\norm{\*A(k,\widetilde{l})}_2\leq 1$.
By Assumption \ref{asm:sigma}, we can easily obtian that $\sigma(a) \leq a+\sigma(0),~\forall a\in \^R$.
Hence, for $\*G^{(k,l)} \coloneqq \sigma\qty(\*W_l\*Z^{(k,l-1)} + \*U_l\*X+ \*b_l)$, we have:
\[
\begin{aligned}
\norm{\*G^{(k,l)}}_F \leq & \sqrt{mn_l}\sigma(0)+ \norm{\*W_l\*Z^{(k,l-1)} + \*U_l\*X+ \*b_l}_F\\
\leq & \sqrt{mn_l}\sigma(0) + \norm{\*z^{(k,l)}} + B_{Ub}\norm{\*X}_F + \sqrt{N} B_{Ub}\\
\leq & \sqrt{mn_l}\sigma(0) + 3B^* + 2B_{Ub}B_x,
\end{aligned}
\]
where, w.l.o.g, we use $\norm{\*X}_F = \Theta(\sqrt{N})$ in the last inequality. Then, by Eq.(\ref{eq:singleVecGrad}), we get:
\[
\begin{aligned}
&\norm{\pdv{\operatorname{vec}\qty(f_l(\*Z^{(k,l)},\*X))}{\operatorname{vec}\qty(\*W_l)}}\\
&\leq  \alpha\qty(\norm{\*G^{(k,l)}}_2 + \norm{\*W_l}_2\norm{\*D^{(k,l)}}_2 \norm{\*Z^{(k,l)}}_2)
\leq   \alpha\qty(\sqrt{mn_l}\sigma(0) + 6B^* + 2B_{Ub}B_x),
\end{aligned}
\]
where we utilize $\norm{\*K^{(n_l,m)}}_2 =1$. Similarly, we have:
\[
\norm{\pdv{\operatorname{vec}\qty(f_l(\*Z^{(k,l)},\*X))}{\operatorname{vec}\qty(\*U_l)}} \leq \alpha B_x, \qtext{and} \norm{\pdv{\operatorname{vec}\qty(f_l(\*Z^{(k,l)},\*X))}{\operatorname{vec}\qty(\*b_l)}} \leq  \alpha \sqrt{N}.
\]
\end{proof}
\begin{lemma}
If Assumption \ref{asm:compactSet} and Assumption \ref{asm:sigma} hold, then:
\begin{equation}\label{eq:pdvBound}
    \left\{
    \begin{aligned}
   & \norm{\pdv{\*z^k}{\operatorname{vec}(\*W_l)}}_2 \leq \qty(k-1)\alpha\qty(\sqrt{mn_l}\sigma(0) + 6B^* + 2B_{Ub}B_x),\\
   & \norm{\pdv{\*z^k}{\operatorname{vec}(\*U_l)}}_2 \leq  \qty(k-1)\alpha B_x,
   \quad
   \norm{\pdv{\*z^k}{\operatorname{vec}(\*b_l)}}_2  \leq \qty(k-1)\alpha \sqrt{N} ,\\
    & \norm{\pdv{\*z^{k,l'}}{\operatorname{vec}(\*W_l)}}_2 \leq \qty(\*1_{\qty{l'>l}}+k-2)\alpha\qty(\sqrt{mn_l}\sigma(0) + 6B^* + 2B_{Ub}B_x),\\
   & \norm{\pdv{\*z^k}{\operatorname{vec}(\*U_l)}}_2 \leq  \qty(\*1_{\qty{l'>l}}+k-2)\alpha B_x,
   \quad
   \norm{\pdv{\*z^k}{\operatorname{vec}(\*b_l)}}_2  \leq \qty(\*1_{\qty{l'>l}}+k-2) \alpha\sqrt{N}  .
    \end{aligned}
    \right.
\end{equation}
Moreover, we get:
\begin{equation}\label{eq:lossBound}
    \left\{
    \begin{aligned}
   & \norm{\operatorname{vec}\qty(\nabla_{\*W_l}\ell(\*y,\*y_0))}\leq \qty(K-1)\alpha\qty(\sqrt{mn_l}\sigma(0) + 6B^* + 2B_{Ub}B_x) B_L \norm{\*y-\*y_0},\\
   & \norm{\operatorname{vec}\qty(\nabla_{\*U_l}\ell(\*y,\*y_0))} \leq  \qty(K-1)\alpha B_xB_L\norm{\*y-\*y_0},\\
   &\norm{\operatorname{vec}\qty(\nabla_{\*b_l}\ell(\*y,\*y_0))} \leq 
   \qty(K-1)\alpha \sqrt{N}B_l\norm{\*y-\*y_0}.
    \end{aligned}
    \right.
\end{equation}
\end{lemma}
\begin{proof}
By Eq.(\ref{eq:VecGrad}), we have:
\[
\norm{\pdv{\*z^k}{\operatorname{vec}(\*W_l)}}_2 \leq \sum_{\widetilde{k}=1}^{k-1} \qty(\norm{\pdv{\operatorname{vec}\qty(f_l(\*Z^{(k,l)},\*X))}{\operatorname{vec}\qty(\*W_l)}})\leq \qty(k-1)\alpha\qty(\sqrt{mn_l}\sigma(0) + 6B^* + 2B_{Ub}B_x).
\]
where we use the results in Eq.(\ref{eq:singlepdvBound}).
Similarly, by Eq.(\ref{eq:VecGrad}), we can also have:
\[
\norm{\pdv{\*z^k}{\operatorname{vec}(\*U_l)}}_2 \leq
\qty(k-1)\alpha B_x, \qtext{and}
   \norm{\pdv{\*z^k}{\operatorname{vec}(\*b_l)}}_2  \leq \sqrt{N}\qty(k-1)\alpha .
\]
By Eq.(\ref{eq:VecGradkl}), an immediate consequence of these bounds is:
\[
\norm{\pdv{\*z^{(k,l')}}{\operatorname{vec}(\*W_l)}}_2 \leq \qty(\*1_{\qty{l'>l}}+k-2)\alpha\qty(\sqrt{mn_l}\sigma(0) + 6B^* + 2B_{Ub}B_x).
\]
Similarly, we also get:
\[
\norm{\pdv{\*z^{(k,l')}}{\operatorname{vec}(\*U_l)}}_2 \leq \qty(\*1_{\qty{l'>l}}+k-2)\alpha B_x,\qtext{and} 
\norm{\pdv{\*z^{(k,l')}}{\operatorname{vec}(\*b_l)}}_2 \leq \qty(\*1_{\qty{l'>l}}+k-2)\sqrt{N}\alpha.
\]
Note that, we already have:
\[
\operatorname{vec}\qty(\nabla_{\*W_l}\ell(\*y,\*y_0)) = \qty(\pdv{\ell(\*y,\*y_0)}{\*z^K} \pdv{\*z^K}{\operatorname{vec}(\*W_l)})^\top
 = \qty(\pdv{\*z^K}{\operatorname{vec}(\*W_l)})^\top \qty(\*I_N \otimes \*W_{L+1}^\top)\qty(\*y-\*y_0).
\]
 Hence, by Eq.(\ref{eq:singlepdvBound}), we can easily get:
\[
\norm{\operatorname{vec}\qty(\nabla_{\*W_l}\ell(\*y,\*y_0))}\leq \qty(K-1)\alpha\qty(\sqrt{mn_l}\sigma(0) + 6B^* + 2B_{Ub}B_x) B_L \norm{\*y-\*y_0}.
\]
Similarly, we have:
\[
\left\{
\begin{aligned}
&\norm{\operatorname{vec}\qty(\nabla_{\*U_l}\ell(\*y,\*y_0))}\leq \qty(K-1)\alpha B_xB_L\norm{\*y-\*y_0},\\
&\norm{\operatorname{vec}\qty(\nabla_{\*b_l}\ell(\*y,\*y_0))}\leq \qty(K-1)\alpha \sqrt{N} B_L\norm{\*y-\*y_0}. 
\end{aligned}
\right.
\]
\end{proof}
\begin{lemma}
We let for all $l \in [1,L]$, $\sigma_{\min}(\*W_l)\geq \sigma_m$. If $\sigma'(\cdot)\geq \kappa >0$ and Assumption \ref{asm:compactSet} hold, then:
\begin{equation}\label{eq:lowerBound}
\norm{\operatorname{vec}\qty(\nabla \ell(\*y,\*y_0))}^2 \geq 
\sum_{l=1}^L \norm{\operatorname{vec}\qty(\nabla_{\*b_l}\ell(\*y,\*y_0))}^2
\geq 
c K^2 L \alpha^2 \kappa^2 \sigma_m^2 N \sigma^2_{\min}(\*W_{L+1}) \norm{\*y-\*y_0}^2,
\end{equation}
for some $0<c<1$.
\end{lemma}
\begin{proof}
Recall that, we already have:
\[
\pdv{\*z^{k+1}}{\operatorname{vec}(\*b_l)} = 
\sum_{\widetilde{k}=1}^k \qty(
\*R_{\widetilde{k}} \*A(\widetilde{k},L,l+1)
\pdv{\operatorname{vec}\qty(f_l(\*Z^{(\widetilde{k}, L,l-1)},\*X))}{\operatorname{vec}\qty(\*b_l)}),
\]
where
\[
\*R_{\widetilde{k}} = \qty(\prod_{q=\widetilde{k}+1}^k \qty(\beta_{q+1}\qty(\*I_{N} \otimes \$V_\gamma)  +\qty(1-\beta_{q+1}) \*A(q,1))).
\]
First of all, we note that  
\[
\norm{\$V_\gamma}_2 \geq \frac{1}{4}.
\]
On the other hand, for any $k \in [K]$, we have:
\[
\*A(k,L,l+1) \geq (1-\alpha)^{L-l}.
\]
 Hence, when $\alpha KL<1$, we have
\[
\begin{aligned}
 \norm{\prod_{q=\widetilde{k}+1}^k \qty(\beta_{q+1}\qty(\*I_{N} \otimes \$V_\gamma)  + \qty(1-\beta_{q+1}) \*A(q,L,1))}_2  \geq & \prod_{q=\widetilde{k}+1}^k  \qty(\frac{\beta_{q+1}}{4} +\qty(1-\beta_{q+1})\qty(1-\alpha L))\\
 \geq & \qty(1-\alpha L)^{(k-\widetilde{k})}.
\end{aligned}
\]
We observe that
\[
\sum_{\widetilde{k} = 1}^k  \qty(1-\alpha L)^{(k-\widetilde{k})} \geq \frac{1-\qty(1-\alpha L)^k}{\alpha L} \geq c_1 k,
\]
for some $0<c_1<1$.
And for any $k\in [K]$,
\[
\alpha\qty(\*I_N \otimes \*W_l^\top)\*D^{(k,l)} \*1_N \geq \alpha \sigma_m \kappa \sqrt{N} .
\]
Then by Eq.(\ref{eq:VecGrad}) and Eq.(\ref{eq:singleVecGrad}), we can obtain:
\[
\norm{\pdv{\*z^{K}}{\operatorname{vec}(\*b_l)}}^2_2 \geq K^2 \qty(1-\alpha)^{2(L-l)}\alpha^2 \kappa^2 \sigma_m^2 N.
\]
Observing that:
\[
\sum_{l=1}^L \qty(1-\alpha)^{2{L-l}} = \frac{1-(1-\alpha)^{2L}}{1-(1-\alpha)^2} \geq c_2 L,
\]
for some $0<c_2<1$. Hence, we have:
\[
\sum_{l=1}^L\norm{\pdv{\*z^{K}}{\operatorname{vec}(\*b_l)}}^2_2  \geq c K^2 L \alpha^2 \kappa^2 \sigma_m^2 N,
\]
for some $0<c<1$. Finally, we can conclude that:
\[
\sum_{l=1}^L \norm{\operatorname{vec}\qty(\nabla \ell(\*y,\*y_0))}^2
\geq 
 \norm{\operatorname{vec}\qty(\nabla_{\*b_l}\ell(\*y,\*y_0))}^2 
\geq 
c K^2 L \alpha^2 \kappa^2 \sigma_m^2 N \sigma^2_{\min}(\*W_{L+1}) \norm{\*y-\*y_0}^2.
\]
\end{proof}
Before proceeding, let us consider how to calculate $\pdv{\*z^k}{\*x}$.
\begin{lemma}
If Assumption \ref{asm:sigma} and Assumption \ref{asm:compactSet} hold, then:
\begin{equation}\label{eq:pdvBoundX}
    \norm{\pdv{\*z^{k+1}}{\*x}}_2 \leq \alpha KL B_{Ub}.
\end{equation}
Moreover, if $g(\*X_0,\*W)$ is smooth and $L_g$-Lipschitz continuous w.r.t $\*W$, we can obtain:
\begin{equation}\label{eq:lossboundW0}
   \norm{\operatorname{vec}\qty(\nabla_{\*W_0}\ell(\*y,\*y_0)) } \leq \alpha KLB_{Ub}B_L L_g \norm{\*y-\*y_0}. 
\end{equation}
\end{lemma}
\begin{proof}
\[
\resizebox{\hsize}{!}{$
\begin{aligned}
&\pdv{\*z^{k+1}}{\*x} = \pdv{\*z^{k+1}}{\*z^{k}} \pdv{\*z^{k}}{\*x} + \sum_{l=1}^L\qty(\*A(k,L,l) \pdv{\operatorname{vec}\qty(f_l(\*Z^{(k,l)},\*X))}{\*x}) \\
= & \beta_{k+1}\qty(\*I_{N} \otimes \$V_\gamma) +\qty(1-\beta_{k+1}) \*A(k,L,1) \pdv{\*z^{k-1}}{\*x} + \sum_{l=1}^L\qty(\*A(k,L,l) \pdv{\operatorname{vec}\qty(f_l(\*Z^{(k,l)},\*X))}{\*x}) \\
=& \sum_{\widetilde{k}=1}^k \qty(
\qty(\prod_{q=\widetilde{k}+1}^k \qty(\beta_{q+1}\qty(\*I_{N} \otimes \$V_\gamma)  +\qty(1-\beta_{q+1}) \*A(q,L,1))) \sum_{l=1}^L\qty(\*A(k,L,l) \pdv{\operatorname{vec}\qty(f_l(\*Z^{(k,l)},\*X))}{\*x})),
\end{aligned}
$}
\]
where the second inequality comes from Eq.(\ref{eq:pdvZ}) and 
\[
\pdv{\operatorname{vec}\qty(f_l(\*Z^{(k,l)},\*X))}{\*x} =  \alpha\qty(\*I_N \otimes \*W_l^\top)\*D^{(k,l)}\qty( \*I_{N} \otimes \*U_l ).
\]
By the bounds given in Eq.(\ref{eq:singlepdvBound}), we can get:
\[
\norm{\pdv{\operatorname{vec}\qty(f_l(\*Z^{(k,l)},\*X))}{\*x}}_2 \leq \alpha B_{Ub}.
\]
Hence, we can  immediately get:
\[
\norm{\pdv{\*z^{k+1}}{\*x}}_2 \leq \alpha KL B_{Ub}.
\]
Note that:
\[
\operatorname{vec}\qty(\nabla_{\*W_0}\ell(\*y,\*y_0)) 
 = \qty(\pdv{\*z^K}{\*x}\pdv{\*x}{\operatorname{vec}\qty(\*W_0)})^\top \qty(\*I_N \otimes \*W_{L+1}^\top)\qty(\*y-\*y_0).
\]
Thus, we can conclude that:
\[
\norm{\operatorname{vec}\qty(\nabla_{\*W_0}\ell(\*y,\*y_0)) } \leq \alpha KLB_{Ub}B_L L_g \norm{\*y-\*y_0}.
\]
\end{proof}

\begin{lemma}
If Assumption \ref{asm:sigma} and Assumption \ref{asm:compactSet} hold, then:
\begin{equation}\label{eq:pdvBoundWL}
    \norm{\pdv{\*z^{K}}{\operatorname{vec}\qty(\*W_{L+1})}}_2 \leq 2K \gamma B^*B_L.
\end{equation}
Moreover, we can obtain:
\begin{equation}\label{eq:lossboundWL}
\norm{\operatorname{vec}\qty(\nabla_{\*W_{L+1}}\ell(\*y,\*y_0)) }  \leq 3KB^*\norm{\*y-\*y_0}.
\end{equation}
\end{lemma}
\begin{proof}
We already have:
\[
\begin{aligned}
&\pdv{\*z^K}{\operatorname{vec}(\*W_{L+1})} = \pdv{\*z^K}{\*z^{K-1}}\pdv{\*z^{K-1}}{\operatorname{vec}(\*W_{L+1})} + \*H^k \\
=& \sum_{\widetilde{k}=1}^K \qty(
\qty(\prod_{q=\widetilde{k}+1}^K \qty(\beta_{q+1}\qty(\*I_{N} \otimes \$V_\gamma)  +\qty(1-\beta_{q+1}) \*A(q,1))) \*H^k )
\end{aligned}
\]
where we let 
\[
\*H^k \coloneqq  - \gamma \beta_K \qty(\qty(\qty(\*W_{L+1}\*Z^K)^\top\otimes\*I_m )\*K^{(d_y,m)} +  \qty(\*I_N \otimes \*W_{L+1}^\top)\qty(\qty(\*Z^K)^\top \otimes \*I_{d_y})).
\]
Hence, we can obtain:
\[
\norm{\pdv{\*z^K}{\operatorname{vec}(\*W_{L+1})} } \leq 2K \gamma B^*B_L,
\]
where we utilize the bounds in Assumption \ref{asm:compactSet}.
Note that:
\[
\operatorname{vec}\qty(\nabla_{\*W_{L+1}}\ell(\*y,\*y_0)) 
 = \qty(\pdv{\*z^K}{\operatorname{vec}(\*W_{L+1})} )^\top \qty(\*I_N \otimes \*W_{L+1}^\top)\qty(\*y-\*y_0) + \qty(\*z^K\otimes \*I_{d_y})\qty(\*y-\*y_0).
\]
Thus, we can conclude that:
\[
\norm{\operatorname{vec}\qty(\nabla_{\*W_{L+1}}\ell(\*y,\*y_0)) } \leq  3K \gamma B^*B_L^2 \norm{\*y-\*y_0} \leq 3KB^*\norm{\*y-\*y_0}.
\]
\end{proof}

In the following lemma, we consider the Lipschitz continuity, we let $\widetilde{\bm{\theta}} \coloneqq \qty{\*W_{L+1}, \bm{\theta},\*W_0}$. Moreover, we denote:
\begin{equation}\label{eq:JabDef}
    \$J_{\*z^K} \coloneqq \mqty[
    \pdv{\*z^K}{\operatorname{vec}(\*W_{L+1})}&
    \pdv{\*z^K}{\operatorname{vec}(\*W_L)}& \pdv{\*z^K}{\operatorname{vec}(\*U_L)}&
\pdv{\*z^K}{\operatorname{vec}(\*b_L)}&
\cdots &  \pdv{\*z^K}{\operatorname{vec}(\*W_0)}].
\end{equation}
\begin{lemma}
Let Assumption \ref{asm:sigma} hold. We assume that the activation function is $L_\sigma$-smooth, and $g(\*X_0,\*W)$ is smooth and $L_g$-Lipschitz continuous w.r.t $\*W$, given two parameters $\widetilde{\bm{\theta}}^a$ and $\widetilde{\bm{\theta}}^b$ satisfying Assumption \ref{asm:compactSet}, we have:
\begin{equation}\label{eq:LipBound}
\resizebox{\hsize}{!}{$
\left\{
\begin{aligned}
       &\norm{\*z^k(\bm{\theta}^a) - \*z^k(\bm{\theta}^b)}_2 \leq 
       \qty( 2 \alpha\sqrt{L} K C_z L_g)
       \norm{\widetilde{\bm{\theta}}^a-\widetilde{\bm{\theta}}^b},\\
       & \norm{\*z^{(k,l)}(\bm{\theta}^a) - \*z^{(k,l)}(\bm{\theta}^b)}  \leq 
       \qty( 2 \alpha\sqrt{L} K C_z L_g)
       \norm{\widetilde{\bm{\theta}}^a-\widetilde{\bm{\theta}}^b},\\
       & \norm{\*D^{(k,l)}(\bm{\theta}^a) - \*D^{(k,l)}(\bm{\theta}^b)}_2 \leq 
       L_{\sigma} \qty( C_D \norm{\Delta\bm{\theta}_l}   + B_{Ub}L_g  \norm{\*W_0^a- \*W_0^b}_2  + \norm{\*z^{(k,l)}(\bm{\theta}^a) - \*z^{(k,l)}(\bm{\theta}^b)}),\\
       & \norm{\*A(k,l_1,l_2, \bm{\theta}^a) - \*A(k,l_1,l_2, \bm{\theta}^b)}_2 \leq 
       \qty( 4\alpha \sqrt{L} L_\sigma C_z L_g )
       \norm{\widetilde{\bm{\theta}}^a-\widetilde{\bm{\theta}}^b},\\
       & \norm{\$J_{\*z^K}(\widetilde{\bm{\theta}}^a) - \$J_{\*z^K}(\widetilde{\bm{\theta}}^b)}_2 \leq 
       \qty(9 \alpha \sqrt{L} B^* L_{\sigma} C_z L_g + 2\sqrt{3L}\alpha K B^*L_\sigma C_D)
       \norm{\widetilde{\bm{\theta}}^a-\widetilde{\bm{\theta}}^b},
\end{aligned}
\right.
$}
\end{equation}
where
\[
\left\{
\begin{aligned}
&\norm{\Delta\bm{\theta}_l} \coloneqq \qty(\norm{\*W_l^a - \*W_l^b}_2 + \norm{\*U_l^a - \*U_l^b}_2 + \norm{\*b_l^a - \*b_l^b}),\\
&C_D\coloneqq \max\qty{3B^*,  B_x},\qtext{and}
C_z \coloneqq \qty(\sqrt{mn_l}\sigma(0) + 6B^* + 3B_{Ub}B_x)
\end{aligned}
\right.
\]
\end{lemma}
\begin{proof}
Due to the smoothness of the activation function, $\*z^k$ is a continuous function of $\bm{\theta}$ on the compact set given in Assumption \ref{asm:compactSet}. By the mean value theorem in high dimension, we can have:
\[
\norm{\*z^k(\bm{\theta}^a) - \*z^k(\bm{\theta}^b)}_2 \leq \norm{\mqty[\pdv{\*z^k(\hat{\bm{\theta}})}{\operatorname{vec}(\*W_{L})} & \pdv{\*z^k(\hat{\bm{\theta}})}{\operatorname{vec}(\*U_{L})} & \pdv{\*z^k(\hat{\bm{\theta}})}{\operatorname{vec}(\*b_{L})}&\cdots & \pdv{\*z^k(\hat{\bm{\theta}})}{\operatorname{vec}(\*W_{0})}]}_2  
\norm{\widetilde{\bm{\theta}}^a-\widetilde{\bm{\theta}}^b},
\]
where $\hat{\bm{\theta}} = \bm{\theta}^a + \xi (\bm{\theta}^b-\bm{\theta}^a)$ for some $\xi \in (0,1)$. 
Due to the convexity of the bounded ball in euclidean space, we can conclude that $\widetilde{\bm{\theta}}$ also satisfy Assumption \ref{asm:compactSet}.
Hence, by Eq.(\ref{eq:pdvBound}) , we can get:
\[
\begin{aligned}
  &\norm{\mqty[\pdv{\*z^k(\hat{\bm{\theta}})}{\operatorname{vec}(\*W_{L})} & \pdv{\*z^k(\hat{\bm{\theta}})}{\operatorname{vec}(\*U_{L})} &\pdv{\*z^k(\hat{\bm{\theta}})}{\operatorname{vec}(\*b_{L})}&\cdots & \pdv{\*z^k(\hat{\bm{\theta}})}{\operatorname{vec}(\*W_{0})}]}_2 \\  \leq
 & \qty(L\qty(k-1)^2\alpha^2\qty({C_z^2+B^2_x+N}) + \norm{\pdv{\*z^k(\hat{\bm{\theta}})}{\*x}\pdv{\*x(\hat{\bm{\theta}})}{\*W_0}}^2_2)^{\frac{1}{2}} \\
 \leq & \qty(3L\qty(k-1)^2\alpha^2 C^2_z + \qty(\alpha KL B_{Ub})^2 \norm{\pdv{\*x(\hat{\bm{\theta}})}{\*W_0}}^2_2)^{\frac{1}{2}} \\
 \leq & \qty(3L\qty(k-1)^2\alpha^2 C^2_z + \alpha^2 K^2L^2 B^2_{Ub}L_g^2)^{\frac{1}{2}}
 \\
 \leq & \qty(4LK^2\alpha^2 C^2_z L^2_g)^{\frac{1}{2}} = 2 \alpha\sqrt{L} K C_z L_g,
\end{aligned}
\]
where we utilize the bound in Eq.(\ref{eq:pdvBoundX}) in the penultimate inequality, and note that $C_z = \Theta(B_{Ub}\sqrt{N}) \gg L$.
Hence, we have:
\[
\norm{\*z^k(\bm{\theta}^a) - \*z^k(\bm{\theta}^b)}_2 \leq 
\qty( 2 \alpha\sqrt{L} K C_z L_g)\norm{\widetilde{\bm{\theta}}^a-\widetilde{\bm{\theta}}^b}.
\]
Similarly, by Eq.(\ref{eq:VecGradkl}) and Eq.(\ref{eq:pdvBound}), we can have:
\[
\norm{\*z^{(k,l)}(\bm{\theta}^a) - \*z^{(k,l)}(\bm{\theta}^b)} \leq \qty( 2 \alpha\sqrt{L} K C_z L_g)\norm{\widetilde{\bm{\theta}}^a-\widetilde{\bm{\theta}}^b}.
\]
Due to the activation function is $L_\sigma$-smooth, we have:
\[
\begin{aligned}
& \norm{\*D^{(k,l)}(\bm{\theta}^a) - \*D^{(k,l)}(\bm{\theta}^b)}_2\\
=  &\norm{\sigma'\qty(\*W^a_l\*Z^{(k,l)}(\bm{\theta}^a) + \*U^a_l\*X^a+ \*b^a_l) - \sigma'\qty(\*W^b_l\*Z^{(k,l)}(\bm{\theta}^b) + \*U^b_l\*X^b+ \*b^b_l)}_\infty\\
 \leq  L_{\sigma} & \norm{\qty(\*W^a_l\*Z^{(k,l)}(\bm{\theta}^a) + \*U^a_l\*X^a+ \*b^a_l) - \qty(\*W^b_l\*Z^{(k,l)}(\bm{\theta}^b) + \*U^b_l\*X^b+ \*b^b_l)}_\infty\\
\leq   L_{\sigma} &\norm{\qty(\*W^a_l\*Z^{(k,l)}(\bm{\theta}^a) + \*U^a_l\*X^a+ \*b^a_l) - \qty(\*W^b_l\*Z^{(k,l)}(\bm{\theta}^b) + \*U^b_l\*X^b+ \*b^b_l)}_2\\
\leq  L_{\sigma} &  \left(\norm{\*W_l^a - \*W_l^b}_2\norm{\*Z^{(k,l)}(\bm{\theta}^b)}_2 + \norm{\*W_l^a}_2\norm{\*Z^{(k,l)}(\bm{\theta}^a)-\*Z^{(k,l)}(\bm{\theta}^b)}_2 +  \right.\\
&\phantom{\;\;}\left.
\norm{\*U_l^a - \*U_l^b}_2\norm{\*X^b}_2 +
\norm{\*U_l^a }_2\norm{\*X^a-\*X^b}_2 + \sqrt{N} \norm{\*b_l^a - \*b_l^b}
\right)\\
\leq L_{\sigma} & \left(3B^* \norm{\*W_l^a - \*W_l^b}_2 + B_x\norm{\*U_l^a - \*U_l^b}_2 + \sqrt{N} \norm{\*b_l^a - \*b_l^b} + \right.\\
&\phantom{\;\;}\left.  B_{Ub}L_g  \norm{\*W_0^a- \*W_0^b}_2  + \norm{\*z^{(k,l)}(\bm{\theta}^a) - \*z^{(k,l)}(\bm{\theta}^b)}\right)\\
\leq  L_{\sigma} & \qty( C_D \norm{\Delta\bm{\theta}_l}   + B_{Ub}L_g  \norm{\*W_0^a- \*W_0^b}_2  + \norm{\*z^{(k,l)}(\bm{\theta}^a) - \*z^{(k,l)}(\bm{\theta}^b)}),
\end{aligned}
\]
where
\[
C_D\coloneqq \max\qty{3B^*,  B_x},\qtext{and}
\norm{\Delta\bm{\theta}_l} \coloneqq \qty(\norm{\*W_l^a - \*W_l^b}_2 + \norm{\*U_l^a - \*U_l^b}_2 + \norm{\*b_l^a - \*b_l^b}).
\]
Based on this upper bound, by the telescoping sum, we can easily have:
\[
\begin{aligned}
&\norm{\*A(k,l_1,l_2, \bm{\theta}^a) - \*A(k,l_1,l_2, \bm{\theta}^b)}_2 \\
\leq
\sum_{l=l_1}^{l_2} & \norm{\prod_{i=l+1}^{l_2} \*T^{(k,i)}(\bm{\theta}^a) \qty(
\*T^{(k,l)}(\bm{\theta}^a) - \*T^{(k,l)}(\bm{\theta}^b))
\prod_{j=l-1}^{l_1}\*T^{(k,i)}(\bm{\theta}^b)}_2\\
\leq \sum_{l=l_1}^{l_2} & \norm{\*T^{(k,l)}(\bm{\theta}^a) - \*T^{(k,l)}(\bm{\theta}^b)}_2\\
\leq 
\sum_{l=l_1}^{l_2} & \alpha \qty(2\norm{\*W_l^a - \*W_l^b}_2 \norm{\*D^{(k,l)}}_2 + \norm{\*W_l^a}_2\norm{\*W_l^b}_2\norm{\*D^{(k,l)}(\bm{\theta}^a) - \*D^{(k,l)}(\bm{\theta}^b)}_2)\\
\leq \sum_{l=l_1}^{l_2} &  
\alpha \qty(2\norm{\*W_l^a - \*W_l^b}_2 + \norm{\*D^{(k,l)}(\bm{\theta}^a) - \*D^{(k,l)}(\bm{\theta}^b)}_2).
\end{aligned}
\]
where we note that for $\prod_{j=l-1}^{l_1}$, $j$ runs in the reverse order, we also let
\[
\*T^{(k,i)} \coloneqq \qty(\alpha \qty(\*I_N \otimes \*W_i^\top) \*D^{(k,i)} \qty(\*I_N \otimes \*W_i) + \qty(1-\alpha)\*I_{mN}),
\]
and we use that $\norm{\*T^{(k,i)} }_2 \leq 1$ from Eq.(\ref{eq:singlepdvBound}) for the second inequality.
We note that:
\[
\begin{aligned}
&\qty(2\norm{\*W_l^a - \*W_l^b}_2 + \norm{\*D^{(k,l)}(\bm{\theta}^a) - \*D^{(k,l)}(\bm{\theta}^b)}_2)\\
\leq & C'_D L_\sigma \norm{\Delta\bm{\theta}_l} 
 + B_{Ub}L_\sigma L_g  \norm{\*W_0^a- \*W_0^b}_2  + L_\sigma\norm{\*z^{(k,l)}(\bm{\theta}^a) - \*z^{(k,l)}(\bm{\theta}^b)}
\end{aligned},
\]
where
\[
C'_D\coloneqq \max\qty{3B^*+ 2/L_\sigma,  B_x}.
\]
Hence, we can obtain:
\[
\begin{aligned}
&\norm{\*A(k,l_1,l_2, \bm{\theta}^a) - \*A(k,l_1,l_2, \bm{\theta}^b)}_2 \\
\leq & \alpha L_\sigma \sum_{l=l_1}^{l_2}  \qty(
  C'_D \norm{\Delta\bm{\theta}_l} 
 + B_{Ub} L_g  \norm{\*W_0^a- \*W_0^b}_2  + \norm{\*z^{(k,l)}(\bm{\theta}^a) - \*z^{(k,l)}(\bm{\theta}^b)})\\
 \leq & \sqrt{3\qty(l_1-l_2)} \alpha L_\sigma C'_D \norm{\widetilde{\bm{\theta}}^a-\widetilde{\bm{\theta}}^b} +\\
  &  \alpha L_\sigma \sum_{l=l_1}^{l_2} \qty(B_{Ub} L_g  \norm{\*W_0^a- \*W_0^b}_2  + \norm{\*z^{(k,l)}(\bm{\theta}^a) - \*z^{(k,l)}(\bm{\theta}^b)})\\
 \leq & \qty(\alpha L_\sigma L \qty( B_{Ub}L_g + 2 \alpha\sqrt{L} K C_z L_g)+ \alpha \sqrt{3L} \ L_\sigma C'_D)\norm{\widetilde{\bm{\theta}}^a-\widetilde{\bm{\theta}}^b} \\
 \leq & \qty( 3\alpha^2L^{3/2}K L_\sigma C_z L_g+ \alpha\sqrt{3L}  C'_D)\norm{\widetilde{\bm{\theta}}^a-\widetilde{\bm{\theta}}^b}\\
 \leq & \qty( 3\alpha \sqrt{L} L_\sigma C_z L_g+ \alpha\sqrt{3L}  C'_D)\norm{\widetilde{\bm{\theta}}^a-\widetilde{\bm{\theta}}^b}\\
 \leq & \qty( 4\alpha \sqrt{L} L_\sigma C_z L_g )\norm{\widetilde{\bm{\theta}}^a-\widetilde{\bm{\theta}}^b}.
\end{aligned}
\]
where we use $\alpha KL \ll 1$, w.l.o.g, and recal that:
\[
C_z \coloneqq \qty(\sqrt{mn_l}\sigma(0) + 6B^* + 3B_{Ub}B_x).
\]
Now, we consider the Lipschitz continuity for $\*G^{(k,l)}$, similar to $\*D^{(k,l)}$:
\[
\begin{aligned}
&\norm{\*G^{(k,l)}(\bm{\theta}^a) - \*G^{(k,l)}(\bm{\theta}^b)}_2 \\
\leq &
\norm{\qty(\*W^a_l\*Z^{(k,l)}(\bm{\theta}^a) + \*U^a_l\*X^a+ \*b^a_l) - \qty(\*W^b_l\*Z^{(k,l)}(\bm{\theta}^b) + \*U^b_l\*X^b+ \*b^b_l)}_F\\
\leq & L_{\sigma} \qty( C_D \norm{\Delta\bm{\theta}_l}   + B_{Ub}L_g  \norm{\*W_0^a- \*W_0^b}_2  + \norm{\*z^{(k,l)}(\bm{\theta}^a) - \*z^{(k,l)}(\bm{\theta}^b)}),.
\end{aligned}
\]
where we use the fact that the activation is $1$-Lipschitz continuous in the first inequality. 
Therefore, by Eq.(\ref{eq:singleVecGrad}), we get:
\[
\begin{aligned}
&\frac{1}{\alpha}\norm{\pdv{\operatorname{vec}\qty(f_l(\*Z^{(k,l)}(\bm{\theta}^a),\*X^a))}{\operatorname{vec}\qty(\*W_l)} - \pdv{\operatorname{vec}\qty(f_l(\*Z^{(k,l)}(\bm{\theta}^b),\*X^b))}{\operatorname{vec}\qty(\*W_l)}}\\
\leq  & \norm{\*G^{(k,l)}(\bm{\theta}^a) - \*G^{(k,l)}(\bm{\theta}^b)}_2 \\
& +  \qty(B^*\norm{\*W_l^a - \*W_l^b}_2+B^*\norm{\*D^{(k,l)}(\bm{\theta}^a) - \*D^{(k,l)}(\bm{\theta}^b)}_2 + \norm{\*z^{(k,l)}(\bm{\theta}^a) - \*z^{(k,l)}(\bm{\theta}^b)}) \\
\leq & L_{\sigma} \qty(B^*+1) \qty( C_D \norm{\Delta\bm{\theta}_l}   + B_{Ub}L_g  \norm{\*W_0^a- \*W_0^b}_2  + \norm{\*z^{(k,l)}(\bm{\theta}^a) - \*z^{(k,l)}(\bm{\theta}^b)}) + \\
& B^* \norm{\*W_l^a - \*W_l^b}_2 +  \norm{\*z^{(k,l)}(\bm{\theta}^a) - \*z^{(k,l)}(\bm{\theta}^b)}\\
\leq &  2 L_{\sigma} B^* \qty( C'_D \norm{\Delta\bm{\theta}_l}   + B_{Ub}L_g  \norm{\*W_0^a- \*W_0^b}_2  + \norm{\*z^{(k,l)}(\bm{\theta}^a) - \*z^{(k,l)}(\bm{\theta}^b)}).
\end{aligned}
\]
On the other hand, also by the telescoping sum , we have:
\[
\begin{aligned}
&\norm{\prod_{q=\widetilde{k}+1}^k \!\!\!\! \qty(\beta_{q+1}\qty(\*I_{N} \otimes \$V^a_\gamma)  +\qty(1-\beta_{q+1}) \*A(\bm{\theta}^a)) \!- \!\!\!
\prod_{q=\widetilde{k}+1}^k \!\!\!\! \qty(\beta_{q+1}\qty(\*I_{N} \otimes \$V^b_\gamma)  +\qty(1-\beta_{q+1}) \*A(\bm{\theta}^b))}_2\\
\leq & \sum_{q=\widetilde{k}+1}^k \qty(2\beta_{q+1}\gamma B_L \norm{\*W^a_{L+1} - \*W^b_{L+1}}_2 + \qty(1-\beta_{q+1}) \norm{\*A(q,L,1,\bm{\theta}^a)-\*A(q,L,1,\bm{\theta}^b)}_2)\\
\leq & k \norm{\*W^a_{L+1} - \*W^b_{L+1}}_2 + \sum_{q=\widetilde{k}+1}^k\norm{\*A(q,L,1,\bm{\theta}^a)-\*A(q,L,1,\bm{\theta}^b)}_2,
\end{aligned}
\]
where we utilize $\gamma B^2_L <1$ in the last inequality.
Since for any $k,l_1$ and $l_2$, $\norm{\*A( \bm{\theta}^a) - \*A(\bm{\theta}^b)}_2 $ share the same upper bound, in the following proof, we omit the index $(k,l_2,l_2)$.
Combing all things together, by Eq.(\ref{eq:VecGrad}), we have:
\[
\begin{aligned}
&\norm{\pdv{\*z^{K}(\bm{\theta}^a)}{\operatorname{vec}(\*W_l)} - \pdv{\*z^{K}(\bm{\theta}^a)}{\operatorname{vec}(\*W_l)}}_2 \\
\leq & K  \left(\qty(2K^2\alpha C_z)\qty(\norm{\*W^a_{L+1} - \*W^b_{L+1}}_2 + \norm{\*A(\bm{\theta}^a)-\*A(\bm{\theta}^b)}_2) 
+\right. \\
&\phantom{\;\;}\left.
\norm{\pdv{\operatorname{vec}\qty(f_l(\*Z^{(K,l)}(\bm{\theta}^a)))}{\operatorname{vec}\qty(\*W_l)} - \pdv{\operatorname{vec}\qty(f_l(\*Z^{(K,l)}(\bm{\theta}^b)))}{\operatorname{vec}\qty(\*W_l)}}
\right)\\
\leq & 2\alpha K^3C_z\norm{\*W^a_{L+1} - \*W^b_{L+1}}_2 + 2\alpha K B^*L_\sigma B_{Ub} L_g\norm{\*W^a_{0} - \*W^b_{0}}_2 + 2\alpha K B^*L_\sigma C_D\norm{\Delta\bm{\theta}_l}+\\
&\qty(2 \alpha L_{\sigma} B^*  2 \alpha\sqrt{L} K C_z L_g +  4\alpha K \sqrt{L} L_\sigma C_z L_g )
\norm{\widetilde{\bm{\theta}}^a-\widetilde{\bm{\theta}}^b}\\
\leq & 2\alpha K^3C_z\norm{\*W^a_{L+1} - \*W^b_{L+1}}_2 + 2\alpha K B^*L_\sigma B_{Ub} L_g\norm{\*W^a_{0} - \*W^b_{0}}_2 + 2\alpha K B^*L_\sigma C_D\norm{\Delta\bm{\theta}_l}+\\
& \qty(8 \alpha^2 K \sqrt{L} B^* L_{\sigma} C_z L_g)
\norm{\widetilde{\bm{\theta}}^a-\widetilde{\bm{\theta}}^b}\\
\leq & 2\alpha K B^*L_\sigma C_D\norm{\Delta\bm{\theta}_l} + \qty(9 \alpha^2 K \sqrt{L} B^* L_{\sigma} C_z L_g)\norm{\widetilde{\bm{\theta}}^a-\widetilde{\bm{\theta}}^b},
\end{aligned}
\]
where, w.l.o.g, we utilize the fact $K^3\sqrt{L}\leq B^*$.
An Immediate consequence we can get is:
\[
\resizebox{\hsize}{!}{$
\begin{aligned}
& \norm{\$J_{\*z^K}(\widetilde{\bm{\theta}}^a) - \$J_{\*z^K}(\widetilde{\bm{\theta}}^b)}_2  \\
\leq & \sum_{l=1}^L \qty(\norm{\pdv{\*z^{K}(\bm{\theta}^a))}{\operatorname{vec}(\*W_l)} - \pdv{\*z^{K}(\bm{\theta}^a))}{\operatorname{vec}(\*W_l)}}_2 + \norm{\pdv{\*z^{K}(\bm{\theta}^a))}{\operatorname{vec}(\*U_l)} - \pdv{\*z^{K}(\bm{\theta}^a))}{\operatorname{vec}(\*U_l)}}_2 
+
\norm{\pdv{\*z^{K}(\bm{\theta}^a))}{\operatorname{vec}(\*b_l)} - \pdv{\*z^{K}(\bm{\theta}^a))}{\operatorname{vec}(\*b_l)}}_2 
)\\
\leq & \qty(L \qty(9 \alpha^2 K \sqrt{L} B^* L_{\sigma} C_z L_g)+ 2\sqrt{3L}\alpha K B^*L_\sigma C_D)
\norm{\widetilde{\bm{\theta}}^a-\widetilde{\bm{\theta}}^b}\\
\leq & \qty(9 \alpha \sqrt{L} B^* L_{\sigma} C_z L_g + 2\sqrt{3L}\alpha K B^*L_\sigma C_D)\norm{\widetilde{\bm{\theta}}^a-\widetilde{\bm{\theta}}^b}.
\end{aligned}
$}
\]
We now finish the proof.
\end{proof}
\subsection{Proof for Theorem \ref{thm:global}}
Before presentation the main results, we assume a mild condition for initialization.
\begin{assumption}[Initial conditions]\label{asm:init}
We assume the initialized parameters of OptDeq in Eq.(\ref{eq:ComDeq}) satisfy:
\[
\left\{
    \begin{aligned}
    &\norm{\*W_l^0}_2 \leq \frac{3}{4}, \quad
    \sigma_{\min}(\*W_l^0) \geq \qty(\frac{1}{4}+\sigma_m),
    \\
    &\max\qty{\norm{\*U_{l}^0}_2, \norm{\*b_{l}^0} } \leq \frac{B_{Ub}}{2},\quad\norm{\*W_0^0}_2 \leq \frac{B_x}{2L_g},\\
    &\norm{\*W_{L+1}^0}_2 \leq \frac{B_L}{2},\quad
    \sigma_{\min}(\*W_{L+1}^0) > 0, 
    \end{aligned}
    \right.
\]
where $L_g$ is the Lipschitz constant of the function $g(\cdot)$ and $\sigma_{\min}(\cdot)$ is the smallest singular value of a matrix.
\end{assumption}
Due to the extractor $g(\*X,\cdot)$ is $L_g$-
Lipschitz continuous, the initial conditions is equivalent to:
\[
\left\{
    \begin{aligned}
    &\norm{\*W_l^0}_2 \leq \frac{3}{4}, \quad
    \sigma_{\min}(\*W_l^0) \geq \qty(\frac{1}{4}+\sigma_m),
    \\
    &\max\qty{\norm{\*U_{l}^0}_2, \norm{\*b_{l}^0} } \leq \frac{B_{Ub}}{2},\\
    &\norm{\*W_{L+1}^0}_2 \leq \frac{B_L}{2},\quad
    \sigma_{\min}(\*W_{L+1}^0) > 0,
    \\ &\norm{\*X^0}_F \leq \frac{B_x}{2}.
    \end{aligned}
    \right.
\]
Now, we are ready to prove the main theorem. We denote
\begin{equation}\label{eq:allconstants}
\left\{
\begin{aligned}
&C_D\coloneqq \max\qty{3B^*,  B_x},\\
&C_z \coloneqq \qty(\sqrt{mn_l}\sigma(0) + 6B^* + 3B_{Ub}B_x)\\
&B_J \coloneqq 2\alpha\sqrt{L}KC_z,\\
& C_J \coloneqq 
\qty(9 \alpha \sqrt{L} B^* L_{\sigma} C_z L_g + 2\sqrt{3L}\alpha K B^*L_\sigma C_D),\\
&Q_0 \coloneqq \frac{1}{4} K^2 L \alpha^2 \kappa^2 \sigma_m^2 N \sigma^2_{\min}(\*W^0_{L+1}) ,\\
& Q_1 \coloneqq \qty(C_J \sqrt{\ell(\widetilde{\bm{\theta}}^0)} + 
2 \alpha\sqrt{L} K C_z L_g B_J B_L+
B_J \sqrt{\ell(\widetilde{\bm{\theta}}^0)}
),\\
& Q_2 \coloneqq \frac{10C_z B_L L_g}{Q_0},
\end{aligned}
\right.    
\end{equation}
where $\ell(\widetilde{\bm{\theta}}^0)$ is the training loss at initialization.
\begin{theorem*}
[Global Convergence]
Suppose Assumption \ref{asm:exist&bound} and Assumption \ref{asm:init} hold.
Assume that the activation function is $L_\sigma$-Lipschitz smooth, strongly monotone and $1$-Lipschitz continuous. 
Let the learning rate be $\eta < \min\qty{\frac{1}{Q_0},\frac{1}{Q_1}}$.
If the training data size $N$ is in the order:
\[
\sqrt{N} = \Omega\qty(\frac{B_L \sqrt{\ell(\widetilde{\bm{\theta}}^0)} }{ \kappa^2\sigma_m^2\sigma^2_{\min}(\*W^0_{L+1})}),
\]
then the training loss vanishes at a linear rate as:
\[
\ell(\widetilde{\bm{\theta}}^t) \leq \ell(\widetilde{\bm{\theta}}^0)\qty(1-\eta Q_0)^t,\] 
where $t$ is the number of iteration. 
Furthermore, the network parameters also converge to a global minimizer $\widetilde{\bm{\theta}}^*$ at a linear speed:
\[
\|\widetilde{\bm{\theta}}^t - \widetilde{\bm{\theta}}^*\|\leq Q_2 \qty(1-\eta Q_0)^{t/2}.
\]
\end{theorem*}
\begin{proof}
We already have $B^*,B_x < C_z = \order{B_{Ub}\sqrt{N}}$. Hence, when the data size $N$ is large enough, i.e., when
\[
\sqrt{N} = \Omega\qty(\frac{B_L \sqrt{\ell(\widetilde{\bm{\theta}}^0)} }{ \kappa^2\sigma_m^2\sigma^2_{\min}(\*W^0_{L+1})}),
\] 
we have:
\begin{equation}\label{eq:boundQ0}
\resizebox{\hsize}{!}{$
  \frac{2}{Q_0} \alpha K C_z B_L  \norm{\*y^0-\*y_0} = 
\frac{8C_z B_L  \norm{\*y^0-\*y_0}}{K L \alpha \kappa^2 \sigma_m^2 N \sigma^2_{\min}(\*W^0_{L+1})} = \Theta\qty(\frac{B_L \sqrt{\ell(\widetilde{\bm{\theta}}^0)} }{ \kappa^2\sigma_m^2\sigma^2_{\min}(\*W^0_{L+1}) \sqrt{N}}) = \Theta(1).  
$}
\end{equation}
Hence, w.l.o.g we let:
\[
\frac{2}{Q_0} \alpha K C_z B_L  \norm{\*y^0-\*y_0} \leq \frac{1}{4}.
\]
We denote the index $t$ to represent the iteration number during training, i.e., $\{\*W^t_{L+1\}, \bm{\theta}^t,\*W_0^t}$ is the learnable parameters at the $t$-th iteration.
We show by induction that, for every $\widetilde{t}>0, t \in [1,\widetilde{t}]$ and $l\in [1,L]$, the following holds:
\begin{equation}\label{eq:induction}
\left\{
    \begin{aligned}
    &\norm{\*W_l^t}_2 \leq 1,\quad \max\qty{\norm{\*U_{l}^t}_2, \norm{\*b_{l}^t} } \leq B_{Ub},\quad \norm{\*W_{L+1}^t}_2 \leq B_L,\quad \norm{\*X^t}_F \leq B_x  \\
    &\sigma_{\min}\qty(\*W^t_l)\geq \sigma_m,
    \quad \sigma_{\min}(\*W^t_{L+1})\geq \frac{1}{2}\sigma_{\min}(\*W^0_{L+1})  \\
    & \ell(\widetilde{\bm{\theta}}^t) \leq \qty(1-\eta Q_0)^t \ell(\widetilde{\bm{\theta}}^0).
    \end{aligned}
    \right.
\end{equation}
By the initial condition, Eq.(\ref{eq:induction}) holds for $t=0$ clearly.
We now suppose that Eq.(\ref{eq:induction}) holds for all iterations from $0$ to $\widetilde{t}$, and show the claim for iteration $\widetilde{t}+1$. 
Note that for every $l\in [1,L]$ and $t \in [1,\widetilde{t}]$, we have:
\[
\begin{aligned}
&\norm{\*W_l^{t+1} - \*W_l^{0}}_2 \leq \sum_{i=1}^t \norm{\*W_l^{i+1} - \*W_l^{i}}_2  \leq \eta  \sum_{i=1}^t  \norm{\operatorname{vec}\qty(\nabla_{\*W_l}\ell(\bm{\theta}^i))}\\
\leq & \eta K \alpha C_z B_L \sum_{i=1}^t \qty(1-\eta Q_0)^{i/2} \norm{\*y^0-\*y_0}
\leq \frac{1}{Q_0} \alpha K C_z B_L (1-s^2)\frac{1}{1-s}  \norm{\*y^0-\*y_0}\\
\leq& \frac{2}{Q_0} \alpha K C_z B_L  \norm{\*y^0-\*y_0} \leq  \frac{1}{4},
\end{aligned}
\]
where we use the bound in Eq.(\ref{eq:lossBound}) and let $s\coloneqq \qty(1-\eta Q_0)^{1/2}$, the last inequality comes from the initial condition on $Q_0$, see Eq.(\ref{eq:boundQ0}).
Similarly, when the data size is large enough, we can also have:
\[
\begin{aligned}
&\norm{\*U_l^{t+1} - \*U_l^{0}}_2 \leq 
\frac{2}{Q_0} \alpha K B_x B_L  \norm{\*y^0-\*y_0} < 1/3\leq \frac{B_{Ub}}{2},\\
&\norm{\*b_l^{t+1} - \*b_l^{0}}_2 \leq 
\frac{2}{Q_0} \alpha K \sqrt{N}  B_L \norm{\*y^0-\*y_0} < 1/3 \leq \frac{B_{Ub}}{2}.
\end{aligned}
\]
And by Eq.(\ref{eq:pdvBoundX}) and Eq.(\ref{eq:pdvBoundWL})
\[
\begin{aligned}
&\norm{\*x^{t+1} - \*x^{0}} \leq L_g \norm{\*W_0^{t+1} - \*W_0^{0}}_F \leq 
\frac{2}{Q_0} \alpha KLB_{Ub}B_L L_g^2  \norm{\*y^0-\*y_0} < \frac{B_x}{2},\\
&\norm{\*W_{L+1}^{t+1} - \*W_{L+1}^{0}}_2 \leq 
\frac{2}{Q_0} 3\gamma KB^* \norm{\*y^0-\*y_0} < \frac{1}{2}\sigma_{\min}(\*W^0_{L+1}).
\end{aligned}
\]
Thus by Weyl’s inequality, we obtain:
\[
\left\{
    \begin{aligned}
    &\norm{\*W_l^{t+1}}_2 \leq 1,\quad \max\qty{\norm{\*U_{l}^{t+1}}_2, \norm{\*b_{l}^{t+1}} } \leq B_{Ub},\quad \norm{\*W_{L+1}^{t+1}}_2 \leq B_L,\quad \norm{\*X^{t+1}}_F \leq B_x  \\
    &\sigma_{\min}\qty(\*W^t_l)\geq \sigma_m,
    \quad \sigma_{\min}(\*W^{t+1}_{L+1})\geq \frac{1}{2}\sigma_{\min}(\*W^0_{L+1}).
    \end{aligned}
    \right.
\]
We now provide a Lipschitz constant for the gradient of loss function. Given two parameters $\widetilde{\bm{\theta}}^a$ and $\widetilde{\bm{\theta}}^b$ such that satisfies Assumption \ref{asm:compactSet} and has the bounds $\ell(\widetilde{\bm{\theta}}^a) \leq \ell(\widetilde{\bm{\theta}}^0)$ and  $\ell(\widetilde{\bm{\theta}}^b) \leq \ell(\widetilde{\bm{\theta}}^0)$, we have
\[
\resizebox{\hsize}{!}{$
\begin{aligned}
&\norm{\nabla \ell(\widetilde{\bm{\theta}}^a) - \nabla \ell(\widetilde{\bm{\theta}}^a)}\\
= & \norm{\$J^\top_{\*z^K}(\widetilde{\bm{\theta}}^a)\qty(\*I_N \otimes \qty(\*W^a_{L+1})^\top)\qty(\*y^a-\*y_0) - \$J^\top_{\*z^K}(\widetilde{\bm{\theta}}^b)\qty(\*I_N \otimes \qty(\*W^b_{L+1}))\qty(\*y^b-\*y_0)} \\
\leq & \norm{\$J_{\*z^K}(\widetilde{\bm{\theta}}^a) - \$J_{\*z^K}(\widetilde{\bm{\theta}}^b)}_2 \sqrt{\ell(\widetilde{\bm{\theta}}^0)} + B_J \norm{\*W_{L+1}^a - \*W_{L+1}^b}\sqrt{\ell(\widetilde{\bm{\theta}}^0)} + B_JB^2_L \norm{\*z^K(\widetilde{\bm{\theta}}^a)-\*z^K(\widetilde{\bm{\theta}}^b)}\\
\leq & 
\qty(C_J \sqrt{\ell(\widetilde{\bm{\theta}}^0)} + 
2 \alpha\sqrt{L} K C_z L_g B_J B_L+
B_J \sqrt{\ell(\widetilde{\bm{\theta}}^0)}
)
\norm{\widetilde{\bm{\theta}}^a - \widetilde{\bm{\theta}}^b}
= Q_1 \norm{\widetilde{\bm{\theta}}^a - \widetilde{\bm{\theta}}^b}.
\end{aligned}
$}
\]
where we use the bound in Eq.(\ref{eq:LipBound}) in the second and third inequality, and the bound from Eq.(\ref{eq:pdvBound}), Eq.(\ref{eq:pdvBound}) and Eq.(\ref{eq:pdvBoundWL}) for
\[
\norm{\$J_{\*z^K}} \leq 
\qty(3LK^2\alpha^2 C_z^2 + \alpha^2K^2L^2B^2_{Ub} + 4K^2\gamma^2 \qty(B^*)^2B_L^2)^{\frac{1}{2}} \leq 
2\alpha\sqrt{L}KC_z
\coloneqq B_J.
\]
When $\eta \leq 1/Q_1$, the Lipschitz bound $\norm{\nabla \ell(\widetilde{\bm{\theta}}^a) - \nabla \ell(\widetilde{\bm{\theta}}^a)} \leq Q_1 \norm{\widetilde{\bm{\theta}}^a - \widetilde{\bm{\theta}}^b}$ implies that:
\[
\begin{aligned}
&\ell(\widetilde{\bm{\theta}}^{t+1}) \leq \ell(\widetilde{\bm{\theta}}^{t}) + \innerprod{\nabla \ell(\widetilde{\bm{\theta}}^t), \widetilde{\bm{\theta}}^{t+1} - \widetilde{\bm{\theta}}^t} + \frac{Q_1}{2} \norm{\widetilde{\bm{\theta}}^{t+1} - \widetilde{\bm{\theta}}^t}^2\\
\leq & \ell(\widetilde{\bm{\theta}}^{t}) - \eta \norm{\nabla \ell(\widetilde{\bm{\theta}}^t)}^2 + \frac{Q_1}{2} \eta^2 \norm{\nabla \ell(\widetilde{\bm{\theta}}^t)}^2\\
\leq & \ell(\widetilde{\bm{\theta}}^{t}) - \frac{\eta}{2} \norm{\nabla \ell(\widetilde{\bm{\theta}}^t)}^2\\
\leq & \ell(\widetilde{\bm{\theta}}^{t}) - \frac{\eta}{2} \qty(K^2 L \alpha^2 \kappa^2 \sigma_m^2 N \sigma^2_{\min}(\*W^t_{L+1}) \norm{\*y^t-\*y_0}^2)\\
\leq & \qty(1- \frac{\eta}{4} K^2 L \alpha^2 \kappa^2 \sigma_m^2 N \sigma^2_{\min}(\*W^0_{L+1}) ) 
\ell(\widetilde{\bm{\theta}}^{t})
 =  \qty(1-\eta Q_0) \ell(\widetilde{\bm{\theta}}^{t}),
\end{aligned}
\]
where the third inequality comes from the fact Eq.(\ref{eq:lowerBound}) and recall that
\[
Q_0 \coloneqq \frac{1}{4} K^2 L \alpha^2 \kappa^2 \sigma_m^2 N \sigma^2_{\min}(\*W^0_{L+1}).
\]
So far, we have proven the hypothesis in Eq.(\ref{eq:induction}). 
\par
We start to show that the sequence $\{\widetilde{\bm{\theta}}^t\}_{t=1}^\infty$ is a Cauchy sequence. Given any $\epsilon>0$ and the index $r>0$, we chose two indices $j>i\geq r$.
Then, we have:
\[
\begin{aligned}
 & \norm{\widetilde{\bm{\theta}}^i - \widetilde{\bm{\theta}}^j} \leq  \norm{\*W^i_{L+1} - \*W^j_{L+1}}_F +   \norm{\bm{\theta}^i - \bm{\theta}^j} + \norm{\*W_0^i - \*W_0^j}_F\\
 \leq &  \eta\sum_{s = i}^{j-1} \eta \qty(\norm{\*W^{s+1}_{L+1} - \*W^s_{L+1}}_F+ \norm{\bm{\theta}^{s+1} - \bm{\theta}^s} + \norm{\*W_0^{s+1} - \*W_0^s}_F) \\
  \leq &   \sum_{s = i}^{j-1}\eta\left(\norm{\*W^{s+1}_{L+1} - \*W^s_{L+1}}_F+\right. \\
&\phantom{\;\;}\left.
   \sum_{l=1}^L\qty(\norm{\*W_l^{s+1} -\*W_l^{s} } + \norm{\*U_l^{s+1} -\*U_l^{s} } + \norm{\*b_l^{s+1} -\*b_l^{s} }) + \norm{\*W_0^{s+1} - \*W_0^s}_F\right)\\
  \overset{(a)}{\leq}& \sum_{s = i}^{j-1} \eta \norm{\*y^s - \*y^0} \qty(3 C_z B_L + B_{Ub}L_g B_L + 3KB^*)\\
  \leq & \qty(1-\eta Q_0)^{i/2} \qty(\sum_{s=0}^{j-i-1}\qty(1-\eta Q_0)^{s/2}\norm{\*y^0 - \*y_0})\eta \qty(3 C_z B_L + B_{Ub}L_g B_L + 3KB^*)\\
  \overset{(b)}{=}& \qty(1-\eta Q_0)^{i/2} \qty(3 C_z B_L + B_{Ub}L_g B_L + 3KB^*) \qty(\frac{1}{Q_0}  (1-s^2)\frac{1-s^{j-i}}{1-s} )\norm{\*y^0 - \*y_0}\\
  \leq & \qty(1-\eta Q_0)^{i/2} \frac{2}{Q_0}\qty(3 C_z B_L + B_{Ub}L_g B_L + 3KB^*)\norm{\*y^0 - \*y_0},
\end{aligned}
\]
where $(a)$ comes from  Eq.(\ref{eq:lossBound}), Eq.(\ref{eq:lossboundW0}) and Eq.(\ref{eq:lossboundWL}) and the assumption $\alpha K L <1$, and in $(b)$ we set $s = \sqrt{1-\eta Q_0}$.
Note that $\qty(1-\eta Q_0)^{i/2} \leq \qty(1-\eta Q_0)^{r/2}$ and thus we can select a sufficiently large $r$ such that $\norm{\widetilde{\bm{\theta}}^i - \widetilde{\bm{\theta}}^j} \leq \epsilon$.
Hence, we can conclude that  $\{\widetilde{\bm{\theta}}^t\}_{t=1}^\infty$ is a Cauchy sequence, and thus has a convergent point $\widetilde{\bm{\theta}}^*$.
Due to the continuity, we have:
\[
\ell(\widetilde{\bm{\theta}}^*) = \ell(\lim_{t\to \infty}\widetilde{\bm{\theta}}^t) = \lim_{t\to \infty}\ell(\widetilde{\bm{\theta}}^t) = 0,
\]
where the last equality comes from Eq.(\ref{eq:induction}).
Hence, $\widetilde{\bm{\theta}}^*$ is a global minimizer, and the rate of convergence is:
\[
\norm{\widetilde{\bm{\theta}}^i - \widetilde{\bm{\theta}}^*} = \lim_{j\to \infty} \norm{\widetilde{\bm{\theta}}^i - \widetilde{\bm{\theta}}^j} \leq \qty(1-\eta Q_0)^{i/2} Q_2,
\]
note that 
\[
\frac{2}{Q_0}\qty(3 C_z B_L + B_{Ub}L_g B_L + 3KB^*)\norm{\*y^0 - \*y_0} \leq \frac{10C_z B_L L_g}{Q_0} \coloneqq Q_2 .
\]
We now finish the whole proof.
\end{proof}

\end{document}